\documentclass{article} 
\usepackage{iclr2027_conference,times}

\usepackage{amsmath,amsfonts,bm}

\def\eqref#1{equation~\ref{#1}}

\def\1{\bm{1}}

\DeclareMathAlphabet{\mathsfit}{\encodingdefault}{\sfdefault}{m}{sl}
\SetMathAlphabet{\mathsfit}{bold}{\encodingdefault}{\sfdefault}{bx}{n}

\usepackage{hyperref}
\usepackage{url}
\usepackage{amssymb}
\usepackage{amsmath}
\usepackage{mathtools}
\usepackage{amsthm}
\usepackage{enumitem}
\usepackage{tikz}
\usepackage{subcaption}
\usepackage{wrapfig}
\usepackage{cleveref}
\usepackage{booktabs}
\usepackage{multirow}
\usepackage{graphicx}
\usepackage{makecell}
\usepackage{comment}
\usepackage{centernot}

\usetikzlibrary{shapes.geometric,positioning,arrows.meta,decorations.pathreplacing,calc,fit,backgrounds}

\title{Conformal Prediction Sets Quantify Information Gain: A Theoretical Perspective}

\author{Kevin Zhang\thanks{EECS, MIT, Cambridge MA, USA.} \qquad\qquad\qquad Stephen Bates\footnotemark[1]}

\newtheorem{definition}{Definition}
\newtheorem{proposition}{Proposition}
\newtheorem{theorem}{Theorem}
\newtheorem{corollary}{Corollary}

\newtheorem{lemma}{Lemma}

\arxiv
\begin{document}

\maketitle

\begin{abstract}
    Conformal prediction is a popular tool for uncertainty quantification that outputs prediction sets with finite-sample coverage guarantees. While prediction set size is commonly used as a heuristic measure of uncertainty, the information-theoretic basis for this interpretation remains poorly understood. In this work, we provide such a foundation using a decision-theoretic generalization of entropy tailored to set-valued prediction. In particular, we introduce a family of generalized information measures based on the size and coverage of conformal prediction sets. Notably, Shannon mutual information admits an exact integral representation in terms of these measures. We then show that, in standard classification settings, the reduction in conformal set size from additional information (i) is sandwiched between calibration-dependent members of this family and (ii) obeys a data processing inequality, both up to finite-sample calibration and model error terms. Together, our results formally relate conformal prediction to classical information-theoretic quantities and justify using set-size reduction as an information gain metric. Empirically, we validate our theory across 11 classification settings and show that set-size reduction and Shannon mutual information can rank features differently in a greedy feature selection experiment.
\end{abstract}

\section{Introduction}

Conformal prediction is a general procedure for constructing prediction sets with finite-sample and distribution-free coverage guarantees under minimal assumptions~\citep{angelopoulos2021gentle}. Recently, the \emph{split conformal prediction} variant has gained popularity as a lightweight and scalable calibration procedure, with applications in healthcare~\citep{lu2022fair}, dynamic control~\citep{lindemann2023safe}, and language modeling~\citep{kumar2023conformal,quach2024conformal}, among other settings.

Beyond coverage, conformal predictors are commonly evaluated by their inefficiency, or expected set size~\citep{vovk2016criteria,sadinle2019least}. At the instance level, set size also provides an easy-to-compute and operationally meaningful heuristic measure of uncertainty~\citep{cresswell2024conformal}, since more ambiguous inputs typically produce larger sets~\citep{romano2019conformalized}. Motivated by this interpretation, we examine conformal set-size reduction as a natural notion of information gain.

The relationship between conformal set size and information-theoretic quantities has been studied both empirically~\citep[e.g.][]{hagos2026performance} and theoretically by using list-decoding arguments to upper bound the conditional entropy in terms of set size and coverage~\citep{correia2024information}. However, these one-sided bounds do not establish whether set-size reduction itself behaves as a principled information measure. In this work, we instead analyze conformal prediction using a decision-theoretic notion of entropy~\citep{degroot1962uncertainty}. This perspective provides new insights into expected conformal set size and its connection to Shannon entropy and mutual information.

Our primary contributions are as follows: (i) we introduce a family of generalized information measures for set-valued prediction that exactly recovers Shannon mutual information via an integral representation; (ii) we prove that, in standard classification settings, conformal set-size reduction approximately lies between calibration-dependent members of this family and obeys a data processing inequality; and (iii) we empirically validate our theory across 11 classification settings and verify that set-size reduction and Shannon mutual information can induce different feature selection orderings in practice.
Our results further the theoretical understanding of uncertainty quantification with conformal prediction and motivate set-size reduction as a principled information gain metric.

\section{Preliminaries} \label{sec:background}

We introduce the notation and relevant background on information theory and conformal prediction. In this work, we consider a classification setting in which $X \in \mathcal{X}$ is a generic feature vector and \mbox{$Y \in \mathcal{Y} \coloneqq \{1, \ldots, K\}$} is a categorical prediction target. 
 
\textbf{Generalized entropy and information.~}~Our analysis relies on a decision-theoretic generalization of entropy and information, in which uncertainty is represented using Bayes risk~\citep{degroot1962uncertainty, grunwald2004game}. Suppose we fix an action space $\mathcal{A}$ and a loss function $\ell : \mathcal{Y} \times \mathcal{A} \to \mathbb{R}$. The generalized entropy induced by $\ell$ is defined as
\begin{equation}
    H_\ell (Y \mid X) = \mathbb{E} \left[ \inf_{a \in \mathcal{A}} \, \mathbb{E} [\ell(Y, a) \mid X ] \right].
\end{equation}
Given side information $Z$, the generalized mutual information between $X$ and $Y$ quantifies the predictive value of $X$ by the reduction in Bayes risk under the chosen loss.
\begin{equation}
    I_\ell(X; Y \mid Z) = H_\ell(Y \mid Z) - H_\ell(Y \mid X, Z) \geq 0.
\end{equation}
Generalized mutual information retains several natural information-theoretic properties, including non-negativity and a data processing inequality~\citep{duchi2018multiclass}. The marginal versions $H_\ell(Y)$ and $I_\ell(X; Y)$ are defined by taking the respective conditioning variables to generate the trivial \mbox{$\sigma$-algebra} (e.g.,~if $X$ and $Z$ are almost surely constant, respectively). 
 
\textit{Example.} Shannon entropy $H(Y) = -\mathbb{E}[\log P_Y(Y)]$ is a special case of the generalized entropy by taking the action space $\mathcal{A} = \Delta^{K-1}$ to be the probability simplex over $\mathcal{Y}$ and using the log loss function $\ell(y, a) = -\log a(y)$. The Bayes-optimal action is then $a = P_Y$, giving $H_\ell(Y) = H(Y)$.

\textbf{Conformal prediction.~}~
Conformal prediction provides a general framework for converting the output of a predictive model into a prediction set with finite-sample coverage guarantees~\citep{vovk2005algorithmic, shafer2008tutorial, angelopoulos2025theoretical}. In this work, we focus on \emph{split conformal prediction}~\citep{papadopoulos2002inductive, lei2018distribution}, which we briefly review next.
 
Let $f$ be a pretrained model and $\mathcal{D}_\mathrm{cal} = \{(X_i, Y_i)\}_{i=1}^n$ be a calibration dataset. Given a score function $s: \mathcal{X} \times \mathcal{Y} \to \mathbb{R} \cup \{\infty\}$, define the calibration scores $S_i = s(X_i, Y_i)$ for $i = 1, \ldots, n$. Fix a nominal target coverage level $1 - \alpha \in (0,1)$, and define the calibrated conformal threshold $\hat{\tau}_{\alpha} = \mathrm{Quantile}\big((1 + 1/n)(1 - \alpha);\, \{S_i\}_{i=1}^n\big)$. At inference time, the conformal prediction set for a new test point is $\mathcal{C}_\alpha (X_\mathrm{test}) = \{y \in \mathcal{Y} : s(X_\mathrm{test}, y) \leq \hat{\tau}_{\alpha}\}$. 

A key appeal of split conformal prediction is that the calibration procedure is lightweight and yields a finite-sample marginal coverage guarantee.
\begin{proposition}[Conformal coverage guarantee] \label{prop:conformal_guarantee}
    Suppose that the calibration observations $\mathcal{D}_\mathrm{cal}$ and test point $(X_\mathrm{test}, Y_\mathrm{test})$ are exchangeable and that the scores $S_i$ are almost surely distinct. Then,
    \begin{equation} \label{eqn:conformal_guarantee}
        1 - \alpha \leq \mathbb{P}\big(Y_\mathrm{test} \in \mathcal{C}_\alpha (X_\mathrm{test})\big) \leq 1 - \alpha + \tfrac{1}{n+1},
    \end{equation}
    where the probability is taken jointly over the calibration data and the test point.
\end{proposition}
 
Proposition~\ref{prop:conformal_guarantee} holds for \emph{any} model $f$ and score function $s$. Nonetheless, the score function determines the size and structure of the conformal sets~\citep{angelopoulos2021uncertainty}. We review two standard choices for classification. Let $q(\cdot \mid x) \in \Delta^{K-1}$ denote the predicted distribution for an input $x$.
 
Under the \emph{probability score}~\citep{sadinle2019least}, $s(x,y) = -q(y \mid x)$ is the negative predicted probability of the label. The resulting conformal prediction set $\mathcal{C}_\alpha (x) = \{ y : q(y \mid x) \geq -\hat\tau_\alpha \}$ is therefore a level set of the predicted class probabilities determined by the calibrated cutoff. 
 
The \emph{Adaptive Prediction Sets (APS) score}~\citep{romano2020classification} instead computes the score by accumulating the predicted probability mass of labels that are \mbox{at least as likely as the true label}. 
\mbox{Formally, let $\smash{q(y^{\scriptscriptstyle (1)}\mid x) > q(y^{\scriptscriptstyle (2)}\mid x) > \cdots > q(y^{\scriptscriptstyle (K)}\mid x)}$} denote the ordered predicted class probabilities (assume for simplicity that there are no ties). Let $r_x(y)$ be the rank of the label $y$, i.e.,~$\smash{y = y^{\scriptscriptstyle (r_x(y))}}$. The randomized APS score function is
\begin{equation}
    s(x, y) = \sum_{\ell=1}^{r_x(y) - 1} q(y^{(\ell)} \mid x) + U\cdot q(y \mid x),
\end{equation}
where $U \sim \mathrm{Unif}[0,1]$ is drawn independently to randomize inclusion of the boundary label. If there are ties among the predicted probabilities, we break them at random following \citet{romano2020classification}.

\section{Conformal Entropy and Mutual Information} \label{sec:methods}
 
We introduce a family of generalized entropy and mutual information measures motivated by set-valued prediction. Our main result then establishes a two-sided bound showing that information gain from conformal prediction is bracketed by calibration-dependent members of this family, up to calibration and model error terms. Lastly, we extend the construction to the APS score. All proofs are deferred to Appendix~\ref{app:sec:proofs}. 
 
\subsection{Generalized information measures for the probability score} \label{subsec:prob_score_construction}
 
We begin with a family of loss functions indexed by $\lambda \in (0,1]$. For a set-valued prediction $\Gamma \subseteq \mathcal{Y}$, define the loss function
\begin{equation} \label{eqn:conformal_loss}
    \ell_\lambda(Y, \Gamma) = |\Gamma| + \tfrac{1}{\lambda} \mathbf{1}(Y \notin \Gamma),
\end{equation}
which penalizes both the size of $\Gamma$ and its miscoverage of the true label $Y$. We use $\ell_\lambda$ to define a generalized entropy with action space $\mathcal{A} = 2^\mathcal{Y}$, the collection of all subsets of $\mathcal{Y}$.
\begin{definition}[Conformal entropy and mutual information]
    For a fixed $\lambda \in (0,1]$, the conformal entropy is defined as
    \begin{equation}
        H_\lambda(Y \mid X) = \mathbb{E}\Bigl[\, \inf_{\Gamma \subseteq \mathcal{Y}} \, \mathbb{E}\bigl[\ell_\lambda(Y, \Gamma) \mid X\bigr]\Bigr].
    \end{equation}
    The corresponding conformal mutual information is $I_\lambda(X; Y \mid Z) = H_\lambda(Y \mid Z) - H_\lambda(Y \mid X, Z)$.
\end{definition}
Since the action space is finite, the conformal entropy can equivalently be expressed as an infimum over (measurable) decision rules $\Gamma: \mathcal{X} \to \mathcal{A}$, so that $H_\lambda(Y \mid X) = \inf_{\Gamma: \mathcal{X} \to \mathcal{A}} \mathbb{E}\bigl[\ell_\lambda(Y, \Gamma(X))\bigr]$. We next characterize the Bayes-optimal actions under $\ell_\lambda$.
\begin{proposition}[Optimal set-valued predictors] \label{prop:optimal_predictor_prob_score}
    Fix $\lambda \in (0,1]$ and let $p(\cdot \mid X)$ be the conditional distribution of $Y$ given $X$. The decision rule $\Gamma_\lambda^*(X; p) = \{y \in \mathcal{Y} : p(y \mid X) \geq \lambda\}$ minimizes $\mathbb{E} \big[ \ell_\lambda(Y, \Gamma(X)) \big]$ over all $\Gamma: \mathcal{X} \to 2^\mathcal{Y}$.
\end{proposition}
Proposition~\ref{prop:optimal_predictor_prob_score} explains the conformal entropy terminology: the Bayes-optimal action $\Gamma_\lambda^*(X; p)$ is exactly the conformal prediction set from an oracle model $f(X) = p(\cdot \mid X)$ using the probability score at the calibrated threshold $\smash{\hat{\tau}_\alpha = -\lambda}$. This connection should be unsurprising since $\smash{\mathbb{E}[\ell_\lambda(Y, \Gamma)]}$ is the Lagrangian relaxation of minimizing expected set size subject to a coverage constraint, which is precisely the tradeoff underlying the probability score~\citep{sadinle2019least}.

In practice, the model predictions may not coincide with the true conditional distribution of $Y$. The resulting model imperfection can be quantified using a generalized information divergence\footnote{More precisely, $D_\lambda$ is a pseudo-divergence, since $D_\lambda(p \parallel q) = 0 \centernot\implies p = q$.}.
\begin{definition}[Conformal information divergence]
    For a fixed $\lambda \in (0,1]$, the conformal information divergence between conditional distributions $p(\cdot \mid X)$ and $q(\cdot \mid X)$ is defined as
    \begin{equation}
        D_\lambda(p(\cdot \mid X) \parallel q(\cdot \mid X)) = \mathbb{E}_X \! \big[ \mathbb{E}_{Y \sim p(\cdot \mid X)} [\ell_\lambda(Y, \Gamma^*_{\lambda}(X; q)) - \ell_\lambda(Y, \Gamma^*_{\lambda}(X; p))]\big].
    \end{equation}
\end{definition}
Intuitively, $D_\lambda(p \parallel q)$ is the excess risk incurred by acting under $q$ when the true distribution is $p$. The divergence is non-negative by the optimality of $\Gamma_\lambda^*(X; p)$ and vanishes when $p = q$.

Using our construction, we can decompose Shannon mutual information into a uniform average of conformal mutual information terms over $\lambda$, analogous to the classical decomposition of the log loss into cost-weighted binary losses~\citep{schervish1989general}.
\begin{theorem}[Integral representation of Shannon information] \label{thm:integral}
    Let $I$ denote the Shannon mutual information (in nats). For any random variables $X$ and $Z$,
    \begin{equation}
        I(X; Y \mid Z) = \int_0^1 I_\lambda(X; Y \mid Z) \, d\lambda.
    \end{equation}
\end{theorem}
This identity connects information gain in conformal prediction to classical information-theoretic quantities. At a high level, Proposition~\ref{prop:optimal_predictor_prob_score} expresses conformal mutual information in terms of differences in set size and miscoverage. Standard coverage guarantees then allow us to bound set-size reduction in terms of $I_\lambda$ at calibrated values of $\lambda$, each of which corresponds to a slice of Shannon mutual information by Theorem~\ref{thm:integral}. The following section makes this intuition precise.

\begin{figure}[t]
    \centering
    \input{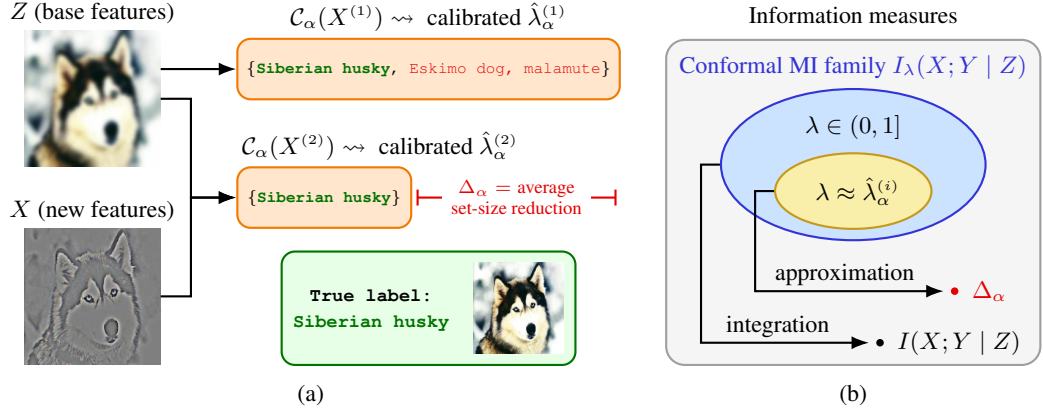}
    \hfill
\begin{subfigure}[t]{0.36\textwidth}
    \centering
    \begin{tikzpicture}[
        font=\small, >=Latex,
        outer/.style={ellipse, draw=blue!80!black!80, fill=blue!55!cyan!20, minimum width=3.5cm, minimum height=2cm, thick, align=center},
        inner/.style={ellipse, draw=gold!95!yellow!90!black, fill=gold!50!yellow!35, minimum width=2.0cm, minimum height=1cm, thick, align=center},
        arrow/.style={->, thick, shorten >=2.5pt},
        pt/.style={circle, fill=black, inner sep=1pt},
        container/.style={rounded corners=7pt, draw=gray!75, fill=gray!8, thick, minimum width=4.95cm, minimum height=4.3cm},
    ]
    \definecolor{gold}{HTML}{DAA520}
    \node[outer] (cmi) {};
    \node[above=0cm of cmi.north, align=center, text=blue!80!black!80] {Conformal MI family $I_\lambda(X;Y \mid Z)$};
    \node at ([yshift=0.5cm]cmi.center) {$\lambda \in (0,1]$};
    \node[inner] (local) at ([yshift=-0.35cm]cmi.center) {\hspace{1pt} $\lambda \approx \hat{\lambda}_\alpha^{\scriptscriptstyle{(i)}}$};
    \node[text=black,        below=1.06cm of cmi, xshift=1.4cm] (shannon) {$I(X;Y \mid Z)$};
    \node[text=red!90!black,  above=0.15cm of shannon, xshift=0.4cm] (cssr) {$\Delta_\alpha$};
    \node[pt]                    (sdot) at ([xshift=-3pt]shannon.west) {};
    \node[pt, fill=red!90!black] (cdot) at ([xshift=-3pt]cssr.west) {};
    \draw[arrow] (cmi.west)   -- ++(-0.25cm,0) |- node[pos=0.72, above, yshift=-0.06cm] {integration} (sdot.west);
    \draw[arrow] (local.west) -- ++(-0.25cm,0) |- node[pos=0.72, above, yshift=-0.05cm, xshift=0.05cm] {approximation} (cdot.west);
    \begin{scope}[on background layer]
        \node[container] (box) at ([yshift=-0.5cm]cmi.center) {};
    \end{scope}
    \node[anchor=north] at ([yshift=0.55cm]box.north) {Information measures};
    \end{tikzpicture}
    \caption{}
    \label{fig:cmi-hierarchy}
\end{subfigure}
    \caption{We illustrate an example of our setting in (\ref{fig:cssr-overview}), where $Z$ is a low-resolution image and $X$ is its residual. The information gain of $X$ beyond $Z$ can be quantified by $\Delta_\alpha$. We develop the information hierarchy in (\ref{fig:cmi-hierarchy}), showing that conformal mutual information near the calibrated $\lambda$ approximates $\Delta_\alpha$, while Shannon mutual information aggregates the entire family (Theorem~\ref{thm:integral}).}
    \label{fig:teaser}
\end{figure}

\subsection{Approximate information gain using conformal set-size reduction} \label{subsec:probability_score_info}

We first formalize the information gained from additional features $X$ beyond base features $Z$ with conformal prediction. For notational convenience, let $X^{\scriptscriptstyle (1)} = Z$ and $X^{\scriptscriptstyle (2)} = (X,Z)$.

For $i \in \{1, 2\}$, let $p(\cdot \mid X^{\scriptscriptstyle(i)})$ and \mbox{$q(\cdot \mid X^{\scriptscriptstyle(i)})$} denote the true and predicted conditional distributions, respectively. Let $\hat{\tau}_\alpha^{\scriptscriptstyle (i)}$ be the calibrated probability-score threshold using $X^{\scriptscriptstyle (i)}$ at nominal coverage level \mbox{$1 - \alpha$}, obtained from a shared calibration dataset $\mathcal{D}_{\mathrm{cal}}$ of size $n$. Let $\mathcal{C}_\alpha(X^{\scriptscriptstyle(i)})$ denote the corresponding conformal prediction set. The (training-conditional) expected set-size reduction is
\begin{equation} \label{eqn:setsize_delta}
    \Delta_\alpha = \mathbb{E}\bigl[|\mathcal{C}_\alpha(X^{\scriptscriptstyle(1)})|\bigr] - \mathbb{E}\bigl[|\mathcal{C}_\alpha(X^{\scriptscriptstyle(2)})|\bigr],
\end{equation}
where the expectations are over a new test input. This setting is illustrated in Figure~\ref{fig:cssr-overview}.

We emphasize that $\Delta_\alpha$ depends on $\mathcal{D}_\mathrm{cal}$. One might expect set-size reduction to be non-negative on average over draws of the calibration data (i.e.,~\mbox{$\mathbb{E}[\Delta_\alpha] \geq 0$}). In the following analysis, we show that an approximate version of this intuition holds.
\begin{theorem}[Probability score sandwich bound] \label{thm:sandwich_prob_score}
    Fix a calibration dataset $\mathcal{D}_{\mathrm{cal}}$ and nominal coverage level $1 - \alpha$. For $i \in \{1, 2\}$, let $\smash{\hat{\lambda}_\alpha^{\scriptscriptstyle (i)} = -\hat{\tau}_\alpha^{\scriptscriptstyle (i)}}$ and assume that $\smash{\hat{\lambda}_\alpha^{\scriptscriptstyle (i)} > 0}$. Moreover, define
    \begin{equation}
        I^{\scriptscriptstyle (i)} = I_{\hat{\lambda}_\alpha^{\scriptscriptstyle (i)}} (X; Y \mid Z), \qquad D^{\scriptscriptstyle (i)} = D_{\hat{\lambda}_\alpha^{\scriptscriptstyle (i)}} (p(\cdot \mid X^{\scriptscriptstyle (i)}) \parallel q(\cdot \mid X^{\scriptscriptstyle (i)})).
    \end{equation}
    Let the training-conditional miscoverage be $\varepsilon^{\scriptscriptstyle (i)} = \mathbb{P}(Y \notin \mathcal{C}_\alpha(X^{\scriptscriptstyle (i)}); \mathcal{D}_\mathrm{cal})$, where the probability is over a fresh test point $(X^{\scriptscriptstyle (i)}, Y)$, and define $\varepsilon^{\scriptscriptstyle \Delta} = \varepsilon^{\scriptscriptstyle (2)} - \varepsilon^{\scriptscriptstyle (1)}$. Then,
    \begin{equation} \label{eqn:sandwich}
        I^{\scriptscriptstyle (2)} - D^{\scriptscriptstyle (2)} + \varepsilon^{\scriptscriptstyle \Delta}  / \hat{\lambda}_\alpha^{\scriptscriptstyle (2)} \;\leq\; \Delta_\alpha \;\leq\; I^{\scriptscriptstyle (1)} + D^{\scriptscriptstyle (1)} + \varepsilon^{\scriptscriptstyle \Delta}  / \hat{\lambda}_\alpha^{\scriptscriptstyle (1)}.
    \end{equation}
\end{theorem}
Theorem~\ref{thm:sandwich_prob_score} states that conformal set-size reduction is bracketed by conformal mutual information terms at two different values of $\lambda$, up to calibration terms and model error. By standard conformal arguments, the training-conditional miscoverage concentrates around $\alpha$ as $n$ grows, which allows us to control the calibration term in expectation under suitable regularity conditions.
\begin{corollary}[Probability score sandwich bound in expectation] \label{cor:expectation_sandwich_prob_score}
    Fix a nominal coverage level $1 - \alpha$. Assume that the calibration scores are almost surely distinct and that the calibration data and future test samples are i.i.d. For $i \in \{1, 2\}$, define $\smash{\hat{\lambda}^{\scriptscriptstyle (i)}_\alpha}$, $I^{\scriptscriptstyle (i)}$, and $D^{\scriptscriptstyle (i)}$ as in Theorem~\ref{thm:sandwich_prob_score}, and suppose that $\smash{s^{\scriptscriptstyle (i)} = \mathrm{Var}(\smash{1 / \hat{\lambda}^{\scriptscriptstyle (i)}_\alpha})^{1/2} < \infty}$. Then,
    \begin{equation}
        \mathbb{E}\bigl[I^{\scriptscriptstyle (2)}\bigr] - \mathbb{E}\bigl[D^{\scriptscriptstyle (2)}\bigr] - s^{\scriptscriptstyle (2)} / \sqrt{n} \;\leq\; \mathbb{E}[\Delta_\alpha] \;\leq\; \mathbb{E}\bigl[I^{\scriptscriptstyle (1)}\bigr] + \mathbb{E}\bigl[D^{\scriptscriptstyle (1)}\bigr] + s^{\scriptscriptstyle (1)} / \sqrt{n},
    \end{equation}
    where all expectations are over the calibration dataset $\mathcal{D}_{\mathrm{cal}}$.
\end{corollary}
Since $I_\lambda \geq 0$, Corollary~\ref{cor:expectation_sandwich_prob_score} implies that on average over draws of the calibration data, the expected set-size reduction from adding $X$ is asymptotically non-negative when $\mathbb{E}[D^{\scriptscriptstyle (2)}] = 0$ and $s^{\scriptscriptstyle (2)}$ remains bounded as $n$ increases. 

The variance condition in Corollary~\ref{cor:expectation_sandwich_prob_score} is mild, requiring only that small predicted probabilities of the true label \mbox{$q(Y \mid X^{\scriptscriptstyle (i)})$} are sufficiently rare. In Appendix~\ref{app:sec:proofs}, we further relax this to almost-sure positivity using an alternative bound that holds with high probability over the calibration data.
\begin{figure}[t]
    \centering
    \includegraphics[width=\textwidth]{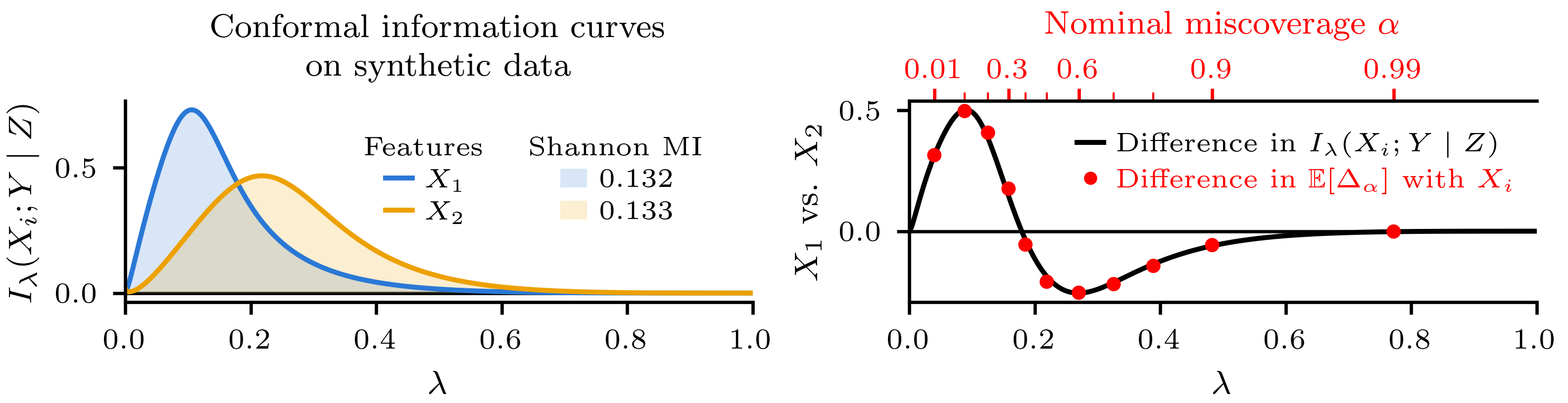}
    \caption{Conformal mutual information accurately reflects feature utility for reducing set size, whereas Shannon mutual information may not. We show information profiles \mbox{$\lambda \mapsto I_{\lambda}(X_i; Y \mid Z)$} for two candidate features $X_1$ and $X_2$ with similar Shannon mutual information (left), along with their differences in $I_\lambda$ and $\mathbb{E}[\Delta_\alpha]$ using the probability score (right). Nominal miscoverage levels $\alpha \in \{0.01, 0.1, \ldots, 0.9, 0.99\}$ are indicated along the top axis. Despite having nearly identical Shannon mutual information, $X_1$ yields a larger conformal set-size reduction at high coverage \mbox{($1 - \alpha > 0.6$)}, whereas $X_2$ reduces set size more at low coverage levels. Each difference in $\mathbb{E}[\Delta_\alpha]$ is positioned at the geometric mean of the two calibrated thresholds.}
    \label{fig:crossing_profile}
\end{figure}

We now revisit the integral representation in Theorem~\ref{thm:integral}. Shannon mutual information uniformly averages $I_\lambda$ over all $\lambda \in (0,1]$, whereas $\Delta_\alpha$ is approximated by members near the calibrated $\smash{\hat{\lambda}_\alpha^{\scriptscriptstyle (i)}}$. We summarize this relationship in Figure~\ref{fig:cmi-hierarchy}. Since these two quantities operate across different regions of $\lambda$, Shannon mutual information can misrepresent the utility of new features for conformal prediction. We demonstrate this behavior in Figure~\ref{fig:crossing_profile} using a long-tailed synthetic setting with two candidate features $X_1$ and $X_2$. Details on the data-generating process are described in Appendix~\ref{app:sec:datasets}.
 
\subsection{Data processing inequality for conformal set-size reduction} \label{subsec:dpi_prob_score}
 
Conformal set-size reduction enjoys additional properties that make it attractive as an approximate notion of mutual information. Here, we highlight one such property by deriving a data processing inequality (DPI) for conformal prediction sets.
 
Suppose that for each data point we also observe a \emph{single realization} of a degraded feature $\smash{\tilde{X}}$ obtained from a (possibly noisy) channel, so that $\smash{Y \to (X, Z) \to (\tilde{X}, Z)}$ forms a Markov chain. \mbox{Let $\smash{\tilde{X}^{\scriptscriptstyle (2)} = (\tilde{X}, Z)}$} and, similarly to $\Delta_\alpha$, define the set-size reduction with the degraded features as
\begin{equation}
    \tilde{\Delta}_\alpha = \mathbb{E}\bigl[|\mathcal{C}_\alpha(X^{\scriptscriptstyle (1)})|\bigr] - \mathbb{E}\bigl[|\mathcal{C}_\alpha(\tilde{X}^{\scriptscriptstyle (2)})|\bigr].    
\end{equation}
The following analysis shows that conformal set-size reduction satisfies an approximate version of the DPI by comparing $\Delta_\alpha$ and $\smash{\tilde{\Delta}_\alpha}$.
\begin{theorem}[Probability score DPI] \label{thm:dpi_prob_score}
    Fix a calibration dataset $\mathcal{D}_{\mathrm{cal}}$ and nominal coverage level $1 - \alpha$. Let $\smash{\hat{\lambda}_\alpha^{\scriptscriptstyle (2)}}$, $D^{\scriptscriptstyle(2)}$, and $\varepsilon^{\scriptscriptstyle (2)}$ be defined as in Theorem~\ref{thm:sandwich_prob_score} and assume $\smash{\hat{\lambda}_\alpha^{\scriptscriptstyle (2)} > 0}$. Define the training-conditional miscoverage for the degraded features as $\smash{\tilde{\varepsilon}^{\scriptscriptstyle (2)} = \mathbb{P}\bigl(Y \notin \mathcal{C}_\alpha(\tilde{X}^{\scriptscriptstyle (2)}) ; \mathcal{D}_{\mathrm{cal}}\bigr)}$. Then,
    \begin{equation}
        \tilde{\Delta}_\alpha \leq \Delta_\alpha + D^{\scriptscriptstyle(2)} + (\tilde{\varepsilon}^{\scriptscriptstyle (2)} - \varepsilon^{\scriptscriptstyle (2)}) / {\hat{\lambda}_\alpha^{\scriptscriptstyle (2)}}.
    \end{equation}
\end{theorem}
As before, a suitable regularity condition controls the finite-sample calibration term in expectation.
\begin{corollary}[Probability score DPI in expectation] \label{cor:expectation_dpi_prob_score}
    Fix a nominal coverage level $1 - \alpha$. Assume that the calibration scores are almost surely distinct and that the calibration data and future test samples are i.i.d. Let $\smash{\hat{\lambda}_\alpha^{\scriptscriptstyle (2)}}$ and $D^{\scriptscriptstyle(2)}$ be defined as in Theorem~\ref{thm:sandwich_prob_score}, and suppose that $\smash{s^{\scriptscriptstyle (2)} = \mathrm{Var}(\smash{1 / \hat{\lambda}^{\scriptscriptstyle (2)}_\alpha})^{1/2} < \infty}$. Then,
    \begin{equation}
        \mathbb{E}[\tilde{\Delta}_\alpha] \leq \mathbb{E}[\Delta_\alpha] + \mathbb{E}[D^{\scriptscriptstyle (2)}] + s^{\scriptscriptstyle (2)} / \sqrt{n}.
    \end{equation}
\end{corollary}
Corollary~\ref{cor:expectation_dpi_prob_score} establishes the approximate DPI for the conformal set-size reduction: on average over draws of the calibration data, the degraded feature $\smash{\tilde{X}}$ reduces the set size by no more than $X$ does, up to model imperfection and an $\smash{O(n^{-1/2})}$ calibration term. However, in practice, these slack terms could make the degraded feature spuriously appear more informative.

\subsection{Extension of conformal mutual information for aps score} \label{subsec:aps_construction}

We extend our results to conformal prediction sets with the APS score. 
Unlike the probability score, whose calibrated threshold acts as a single global cutoff on $q(\cdot \mid x)$, the APS score effectively applies an input-dependent threshold with randomization at the boundary. Hence, we make two key extensions in our analysis. First, we index the conformal mutual information by a random $\Lambda \in (0,1]$ in place of the fixed $\lambda$. Second, we introduce $U \sim \mathrm{Unif}[0,1]$ independent of $(\Lambda, X, Y)$ to randomize the set-valued action. Let $\Omega$ denote the space of measurable functions $[0,1] \to 2^\mathcal{Y}$.
\begin{definition}[Conformal entropy and mutual information for random $\Lambda$]
    Let $\Lambda \in (0,1]$ be a random index variable and let $\ell_\Lambda(Y, \Gamma)$ denote the loss in Equation~\ref{eqn:conformal_loss} evaluated using the value of $\Lambda$. Define the conformal entropy under $\Lambda$ as
    \begin{equation}
        H_\Lambda(Y \mid X) = \mathbb{E}\Bigl[ \inf_{\Gamma \in \Omega} \, \mathbb{E}\bigl[\ell_\Lambda(Y, \Gamma(U)) \mid X \bigr] \Bigr],
    \end{equation}
    where the inner expectation is taken over $Y$, $U$, and $\Lambda$. The corresponding conformal mutual information is $I_\Lambda(X; Y \mid Z) = H_\Lambda(Y \mid Z) - H_\Lambda(Y \mid X, Z)$.
\end{definition}
The random $\Lambda$ allows the loss index to depend on the input. Since the action space is finite, we again express the conformal entropy as an infimum over measurable decision rules $\Gamma: \mathcal{X} \times [0,1] \to \mathcal{A}$ and show that conformal prediction sets using the APS score achieve $H_\Lambda$ for suitable choices of $\Lambda$.
\begin{proposition}[Optimal set-valued predictors for random $\Lambda$] \label{prop:optimal_predictor_aps}
    Suppose that $\Lambda = \xi(X)$ for some fixed measurable function $\xi: \mathcal{X} \to (0,1]$, and let $p(\cdot \mid X)$ be the conditional distribution of $Y$ given $X$. Let $\Gamma_\Lambda^*(X,U;p)$ be any decision rule such that, for all $u \in [0,1]$,
    \begin{equation}
        \bigl\{y : p(y \mid X) > \xi(X)\bigr\} \subseteq \Gamma^*_{\Lambda}(X, u; p) \subseteq \bigl\{y : p(y \mid X) \geq \xi(X)\bigr\}.
    \end{equation}
    Then, $\Gamma_\Lambda^*(X,U;p)$ minimizes $\mathbb{E} \left[ \ell_\Lambda(Y, \Gamma(X, U)) \right]$ over all measurable $\Gamma: \mathcal{X} \times [0,1] \to 2^\mathcal{Y}$.
\end{proposition}
The conformal prediction set using an oracle model and the APS score is a Bayes-optimal action $\Gamma^*_\Lambda(X,U; p)$, where $\Lambda$ is the \emph{induced} input-dependent probability threshold equal to the first predicted probability (in decreasing order) for which the cumulative mass exceeds $\hat{\tau}_\alpha$. When $\hat{\tau}_\alpha < 1$, this quantity is well-defined and satisfies $\Lambda = \xi(X) \in (0,1]$ for some measurable function $\xi$.

We define $\Delta_\alpha$ as in Equation~\ref{eqn:setsize_delta}, with the expectation now also taken over $U$. Furthermore, we overload $\hat{\tau}_\alpha^{\scriptscriptstyle (i)}$ to denote the calibrated APS-score threshold using $X^{\scriptscriptstyle (i)}$ and $\varepsilon^{\scriptscriptstyle (i)}$ to be the $X^{\scriptscriptstyle (2)}$-conditional miscoverage. We next formalize the conformal information divergence under a random $\Lambda$.
\begin{definition}[Conformal information divergence for random $\Lambda$] \label{def:aps_divergence}
    Suppose that $\Lambda = \xi(X)$ for some fixed measurable function $\xi: \mathcal{X} \to (0,1]$. Define the conformal information divergence under $\Lambda$ between conditional distributions $p(\cdot \mid X)$ and $q(\cdot \mid X)$ as
    \begin{equation}
        D_\Lambda(p(\cdot \mid X) \parallel q(\cdot \mid X)) = \mathbb{E}_X \! \big[ \mathbb{E}_{Y \sim p(\cdot \mid X), \, U} [\ell_\Lambda(Y, \Gamma^*_{\Lambda}(X, U; q)) - \ell_\Lambda(Y, \Gamma^*_{\Lambda}(X, U; p))]\big].
    \end{equation}
\end{definition}
In Definition~\ref{def:aps_divergence}, the value of $D_\Lambda$ depends on the boundary randomization choice for $\Gamma^*_\Lambda(X, U; q)$. When $\Lambda$ is the probability threshold induced by the APS score of $q(\cdot \mid X)$ at some $\tau < 1$, we use the convention that $\Gamma^*_\Lambda(X, U; q)$ is the corresponding conformal prediction set. We are now ready to extend Theorem~\ref{thm:sandwich_prob_score} for the APS score.  
\begin{theorem}[APS score sandwich bound] \label{thm:sandwich_aps}
    Fix a calibration dataset $\mathcal{D}_{\mathrm{cal}}$ and nominal coverage level $1 - \alpha$. For $i \in \{1, 2\}$, assume that $\hat{\tau}_\alpha^{\scriptscriptstyle (i)} < 1$, and let $\smash{\hat{\Lambda}_\alpha^{\scriptscriptstyle (i)}}$ be the input-dependent random probability threshold induced by $\hat{\tau}_\alpha^{\scriptscriptstyle (i)}$. Moreover, define
    \begin{equation}
        I^{\scriptscriptstyle (i)} = I_{\hat{\Lambda}_\alpha^{\scriptscriptstyle (i)}} (X; Y \mid Z), \qquad D^{\scriptscriptstyle (i)} = D_{\hat{\Lambda}_\alpha^{\scriptscriptstyle (i)}} (p(\cdot \mid X^{\scriptscriptstyle (i)}) \parallel q(\cdot \mid X^{\scriptscriptstyle (i)})).
    \end{equation}
    Let $\varepsilon^{\scriptscriptstyle (i)} = \mathbb{P}(Y \not\in \mathcal{C}_\alpha(X^{\scriptscriptstyle(i)}) \mid X^{\scriptscriptstyle (2)}; \mathcal{D}_\mathrm{cal})$, where the probability is taken jointly over a fresh test label $Y$ given $X^{\scriptscriptstyle (2)}$ and the APS randomization $U$. Define $\smash{\varepsilon^{\scriptscriptstyle \Delta} = \varepsilon^{\scriptscriptstyle (2)} - \varepsilon^{\scriptscriptstyle (1)}}$. Then,
    \begin{equation} \label{eqn:sandwich_aps}
        I^{\scriptscriptstyle (2)} - D^{\scriptscriptstyle (2)} + \mathbb{E}\big[\varepsilon^{\scriptscriptstyle \Delta} / \hat{\Lambda}^{\scriptscriptstyle (2)}_\alpha\big] \;\leq\; \Delta_\alpha \;\leq\; I^{\scriptscriptstyle (1)} + D^{\scriptscriptstyle (1)} + \mathbb{E}\big[\varepsilon^{\scriptscriptstyle \Delta} / \hat{\Lambda}^{\scriptscriptstyle (1)}_\alpha\big] ,
    \end{equation}
    where the expectations are over $X^{\scriptscriptstyle (2)}$, and both bounds are finite.
\end{theorem}
Theorem~\ref{thm:sandwich_aps} reduces to Theorem~\ref{thm:sandwich_prob_score} when $\smash{\hat{\Lambda}}$ is constant. For a random $\smash{\hat{\Lambda}}$, the $\varepsilon^{\scriptscriptstyle \Delta}$ terms depend on the $X^{\scriptscriptstyle (2)}$-conditional miscoverage, rather than the training-conditional miscoverage. In Appendix~\ref{app:sec:proofs}, we derive corresponding bounds in expectation and with high probability over the calibration data, as well as a corresponding data processing inequality. In general, the APS results incur additional slack on both sides of the bound, because the conditional miscoverage typically does not concentrate as $n$ grows. For oracle models, this slack vanishes in the upper bound.

\section{Related Work}

We build on the decision-theoretic notion of information introduced by \citet{degroot1962uncertainty}. A long line of work has studied generalized entropies in relation to proper scoring rules~\citep{gneiting2007strictly} and information-theoretic divergences~\citep{reid2011information,garcia2012divergences}. We instead focus on a particular family of generalized information measures tailored to set-valued prediction. The loss in Equation~\ref{eqn:conformal_loss} is related to class-selective rejection~\citep{ha1996optimum} and classification with rejection~\citep{herbei2006classification}, though these ideas have largely been developed independently of conformal prediction~\citep{chzhen2021setvalued}.

A growing body of empirical work studies conformal set size as an operational measure of predictive uncertainty~\citep{fayyad2024empirical,portela2025conformal,hagos2026performance}. On the theoretical side, several works characterize expected set size as a measure of predictive inefficiency using finite-sample estimates~\citep{dhillon2024expected} and model-dependent bounds~\citep{zecchin2024generalization}. However, these works do not formally motivate set size as a measure of uncertainty or information.

The work closest to ours is \citet{correia2024information}, who use list-decoding arguments to upper bound the conditional entropy in terms of set size and coverage. In contrast, our decision-theoretic approach uses generalized entropy to exactly characterize set size and miscoverage. This formulation yields two-sided bounds on set-size reduction, which we further connect to Shannon mutual information via an integral representation. These results provide a direct information-theoretic interpretation of set-size reduction that is not captured by existing work.

Information gain metrics are widely used in practice for active learning~\citep{houlsby2011bayesian,kirsch2019batchbald,bickfordsmith2023prediction} and feature acquisition~\citep{covert2023learning}. While these approaches typically rely on classical information-theoretic quantities, \citet{neiswanger2022generalizing} study decision-theoretic information to optimize downstream utility. Separately, conformal prediction has been used to design acquisition policies based on prediction set size~\citep{kharazian2024copal}. Our main contribution is identifying theoretical connections between these notions of information gain.

\section{Experiments} \label{sec:experiments}

We empirically examine our theoretical results and then probe set-size reduction and Shannon mutual information in a feature selection experiment. Motivated by the behavior in Figure~\ref{fig:crossing_profile}, we show that these two criteria can rank features differently in practice. For all experiments, we fix the nominal coverage to be $1 - \alpha = 0.9$, unless noted otherwise.

\textbf{Datasets.~}~We consider 11 datasets spanning synthetic, image, and tabular domains. We use three synthetic settings. The first (S1) defines the true conditional distribution using a two-layer ReLU network, while the second (S2) uses a Gaussian mixture model with each mixture component representing a class. For the feature selection experiments, we additionally introduce a third \mbox{setting (S3)}, where the conditional distribution has a long low-probability tail. All synthetic datasets provide access to the true conditional probabilities, allowing us to evaluate our results using an oracle model.

For tabular domains, we use the Wine Quality (Wine)~\citep{cortez2009wine}, Human Activity Recognition (HAR)~\citep{reyesortiz2013human}, Letter Recognition (Letter)~\citep{slate1991letter}, and Covertype~\citep{blackard1998covertype} datasets. For image classification, we use FMNIST~\citep{xiao2017fashion}, CIFAR-10, CIFAR-100~\citep{krizhevsky2009learning}, and ImageNet-1k~\citep{deng2009imagenet}. Dataset sizes, dimensions, numbers of classes, and preprocessing details are provided in Appendix~\ref{app:sec:datasets}.

\textbf{Feature construction.~}~We partition coordinates of the synthetic data and features of the tabular datasets into disjoint subsets $Z$ and $X$. For images, we take $Z$ to be a Gaussian-blurred version of the original data and $X$ to be its residual (see Figure~\ref{fig:cssr-overview}). The DPI experiments construct $\smash{\tilde{X}}$ by either adding Gaussian noise or applying Gaussian blur to $X$. Candidate features in the selection experiment are partitions of the input vector for synthetic and tabular datasets or a pretrained dense representation for image datasets. Appendix~\ref{app:sec:datasets} provides further details on feature construction.

\textbf{Models.~}~We train a model for each feature set on the real-world datasets and hold it fixed across subsequent calibration procedures. The tabular datasets use MLPs with two or three hidden layers, ReLU activations, and dropout. For FMNIST, CIFAR-10, and CIFAR-100, we train CNNs with max pooling, while for ImageNet-1k we fine-tune a pretrained ResNet-18~\citep{he2016deep}.

\textbf{Training.~}~We use the oracle model for the synthetic settings and generate a fixed evaluation pool of $200{,}000$ samples. For the tabular datasets, we allocate $60\%$ of the data to training, $15\%$ to validation, and the remaining $25\%$ to an evaluation pool. The image datasets have dedicated evaluation pools, and the remaining data is split into $80\%$ for training and $20\%$ for validation.

Across all training procedures, we use early stopping based on validation loss and select the learning rate, weight decay, and dropout rate by grid search. In Appendix~\ref{app:sec:training}, we provide additional details on model architectures, training procedures, hyperparameters, and predictive accuracy.

\begin{figure}[t]
    \centering
    \includegraphics[width=0.99\textwidth]{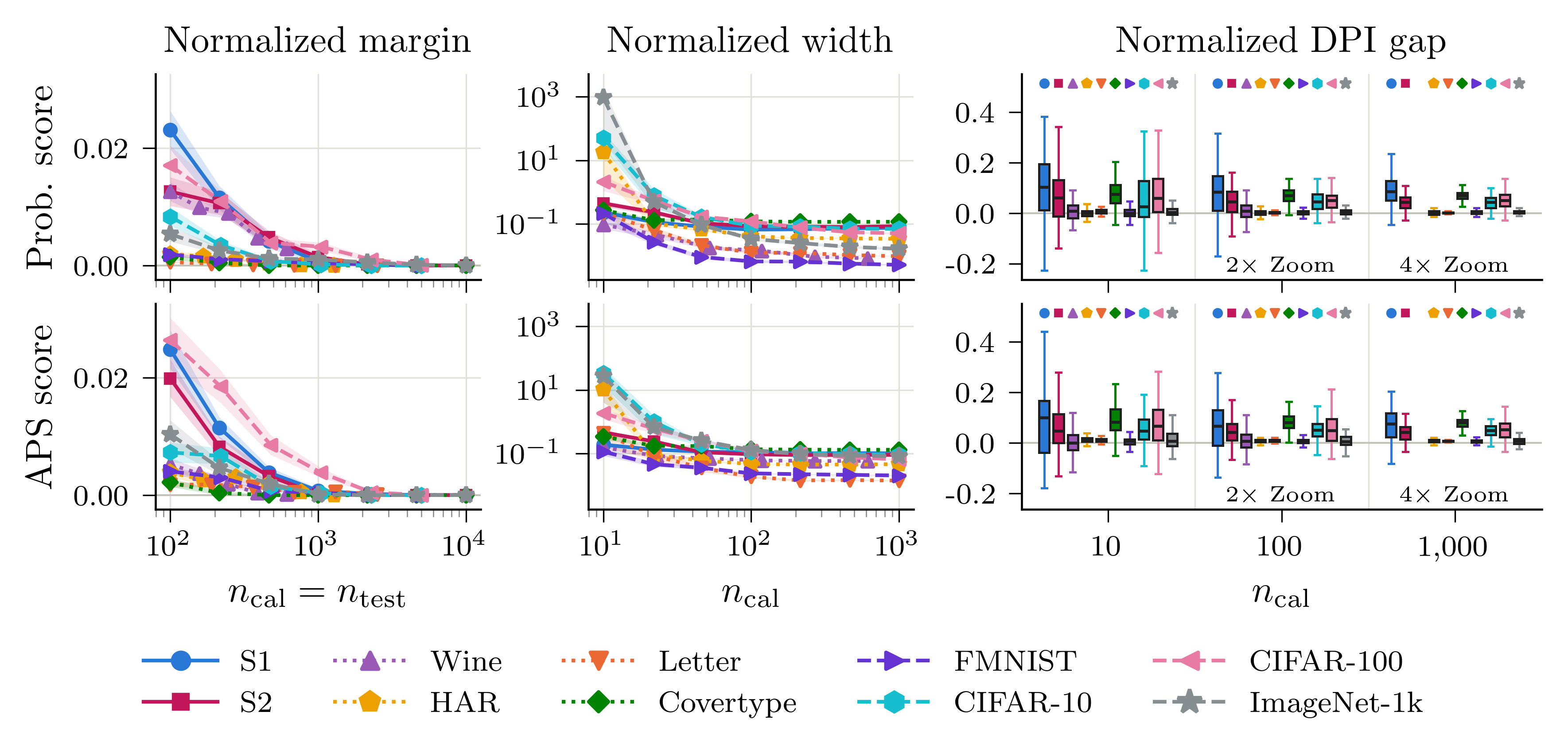}
    \caption{We show the normalized margin, width, and DPI gap based on $100$ random draws of the calibration and test sets from the fixed evaluation pool. The margin and width panels report the \mbox{mean $\pm$ standard error}, while the DPI gap panel summarizes the distribution using boxplots. Datasets whose evaluation pool is too small for a given calibration size are omitted.}
    \label{fig:validation_experiments}
    \vspace{-1pt}
\end{figure}

\subsection{Sandwich and dpi validation}

We first evaluate the sandwich bounds in Theorems~\ref{thm:sandwich_prob_score} and~\ref{thm:sandwich_aps}. From each evaluation pool, we sample separate calibration and test sets to estimate $\Delta_\alpha$ and $\varepsilon^{\scriptscriptstyle(i)}$. In the synthetic settings, the true conditional distribution is known, so $I^{\scriptscriptstyle(i)}$ is computable and \mbox{$D^{\scriptscriptstyle(i)} = 0$}. For the real-world datasets, $I^{\scriptscriptstyle(i)}$ and $D^{\scriptscriptstyle(i)}$ are inaccessible, so we use the plug-in model estimates, which set $D^{\scriptscriptstyle (i)} = 0$. Moreover, we use a $k$-NN approximation for $I^{\scriptscriptstyle(2)}$ in the APS lower bound since $\smash{\hat\Lambda_\alpha^{\scriptscriptstyle(2)}}$ depends on $X$ (see Appendix~\ref{app:sec:analysis}). The real-data results are thus an empirical corroboration rather than strict validation of the theory.

Letting $L$ and $U$ denote the lower and upper bounds, we measure the magnitude of violations using the normalized margin $d(\Delta_\alpha,[L,U])/K$, where $d(x,[L,U])$ is the distance from $x$ to $[L,U]$. We separately quantify the informativeness of the bound using the normalized width $(U-L)/K$.

The left column of Figure~\ref{fig:validation_experiments} reports the normalized margin as we vary $n_{\mathrm{cal}} = n_{\mathrm{test}}$. We limit the synthetic sample sizes to match the real datasets for a controlled comparison. Even with oracle models, the average margin is positive since finite samples introduce Monte Carlo error, but it approaches zero as sample sizes increase. The margins on the real data exhibit similar behavior. The middle column reports the normalized width for a fixed $n_\mathrm{test}$. The intervals narrow with more calibration data, but the width remains nonzero due to the gap between $I^{\scriptscriptstyle(1)}$ and $I^{\scriptscriptstyle(2)}$. At very small $n_\mathrm{cal}$ and large $K$, the $\smash{1/\hat\lambda_\alpha^{\scriptscriptstyle(i)}}$ factor makes the bounds vacuous since the calibrated threshold is typically small.

We next empirically analyze Theorem~\ref{thm:dpi_prob_score} and its APS analogue (Theorem~\ref{thm:dpi_aps} in Appendix~\ref{app:sec:proofs}). The right column of Figure~\ref{fig:validation_experiments} reports the normalized DPI gap $\smash{(\Delta_\alpha - \tilde{\Delta}_\alpha)/K}$ as we vary $n_\mathrm{cal}$ with $n_\mathrm{test}$ fixed. At small calibration sizes, finite-sample effects can produce negative DPI gaps, meaning that the degraded features $\smash{\tilde X}$ appear to reduce set size more than the original features $X$. As $n_\mathrm{cal}$ increases, however, these violations shrink in magnitude and the gap concentrates around non-negative values.

We include additional experimental analyses in Appendix~\ref{app:sec:experiments}, including violation rates for Theorems~\ref{thm:sandwich_prob_score} and~\ref{thm:sandwich_aps}, results across nominal coverage levels, and DPI behavior under varying noise levels.

\begin{table}[t]
    \centering
    \caption{We evaluate feature selection performance using conformal set-size reduction. Each row represents one selection criterion and values are reported as \mbox{mean $\pm$ standard error} over 10 random splits of the evaluation pool. Arrows indicate the direction of improvement, and the best-performing method for each metric is bolded. Empirical coverage is reported for reference without bolding.}
    \label{tab:feature_selection}
    \resizebox{0.99\columnwidth}{!}{%
    \begin{tabular}{c c cccccc}
        \toprule
        \multirow{2}{*}{Dataset} & \multirow{2}{*}{Criterion}
        & \multicolumn{2}{c}{Prob.}
        & \multicolumn{2}{c}{APS}
        & \multicolumn{2}{c}{Acc.} \\
        \cmidrule(lr){3-4} \cmidrule(lr){5-6} \cmidrule(lr){7-8}
        &
        & $\mathbb{E}[|\mathcal{C}|]$ ($\downarrow$)
        & Cov.
        & $\mathbb{E}[|\mathcal{C}|]$ ($\downarrow$)
        & Cov.
        & Top-1 ($\uparrow$)
        & Top-5 ($\uparrow$) \\
        \midrule

        \multirow{5}{*}{\rotatebox[origin=c]{90}{\shortstack{S3\\($K = 40$)}}}
        & $\Delta_\alpha$ (Prob.)
        & $\mathbf{2.239 \pm 0.002}$
        & $0.900 \pm 0.000$
        & $2.629 \pm 0.002$
        & $0.900 \pm 0.000$
        & $0.750 \pm 0.001$
        & $\mathbf{0.952 \pm 0.000}$ \\
        &
        $\Delta_\alpha$ (APS)
        & $2.243 \pm 0.002$
        & $0.900 \pm 0.000$
        & $\mathbf{2.611 \pm 0.002}$
        & $0.900 \pm 0.000$
        & $0.752 \pm 0.000$
        & $\mathbf{0.952 \pm 0.000}$ \\
        &
        Shannon MI
        & $2.447 \pm 0.005$
        & $0.900 \pm 0.000$
        & $2.876 \pm 0.006$
        & $0.900 \pm 0.000$
        & $0.767 \pm 0.000$
        & $\mathbf{0.952 \pm 0.000}$ \\
        &
        Accuracy
        & $3.613 \pm 0.020$
        & $0.900 \pm 0.000$
        & $4.487 \pm 0.029$
        & $0.900 \pm 0.000$
        & $\mathbf{0.768 \pm 0.001}$
        & $0.914 \pm 0.001$ \\
        &
        Random
        & $3.046 \pm 0.100$
        & $0.900 \pm 0.000$
        & $3.591 \pm 0.116$
        & $0.900 \pm 0.000$
        & $0.672 \pm 0.009$
        & $0.931 \pm 0.004$ \\

        \midrule

        \multirow{5}{*}{\rotatebox[origin=c]{90}{\shortstack{CIFAR\\($K = 100$)}}}
        & $\Delta_\alpha$ (Prob.)
        & $12.683 \pm 0.100$
        & $0.904 \pm 0.002$
        & $15.661 \pm 0.127$
        & $0.902 \pm 0.002$
        & $0.466 \pm 0.002$
        & $\mathbf{0.753 \pm 0.002}$ \\
        &
        $\Delta_\alpha$ (APS)
        & $\mathbf{12.594 \pm 0.088}$
        & $0.902 \pm 0.002$
        & $\mathbf{15.472 \pm 0.129}$
        & $0.902 \pm 0.002$
        & $\mathbf{0.467 \pm 0.002}$
        & $\mathbf{0.753 \pm 0.002}$ \\
        &
        Shannon MI
        & $13.763 \pm 0.115$
        & $0.902 \pm 0.002$
        & $16.597 \pm 0.145$
        & $0.902 \pm 0.002$
        & $0.458 \pm 0.003$
        & $0.741 \pm 0.002$ \\
        &
        Accuracy
        & $12.872 \pm 0.130$
        & $0.902 \pm 0.002$
        & $15.877 \pm 0.158$
        & $0.901 \pm 0.002$
        & $\mathbf{0.467 \pm 0.003}$
        & $0.752 \pm 0.002$ \\
        &
        Random
        & $14.243 \pm 0.253$
        & $0.902 \pm 0.003$
        & $17.189 \pm 0.224$
        & $0.902 \pm 0.003$
        & $0.454 \pm 0.002$
        & $0.737 \pm 0.002$ \\

        \midrule

        \multirow{5}{*}{\rotatebox[origin=c]{90}{\shortstack{Letter\\($K = 26$)}}}
        & $\Delta_\alpha$ (Prob.)
        & $\mathbf{2.753 \pm 0.016}$
        & $0.898 \pm 0.003$
        & $\mathbf{3.358 \pm 0.019}$
        & $0.900 \pm 0.001$
        & $0.756 \pm 0.001$
        & $\mathbf{0.926 \pm 0.000}$ \\
        &
        $\Delta_\alpha$ (APS)
        & $2.835 \pm 0.038$
        & $0.898 \pm 0.002$
        & $3.398 \pm 0.032$
        & $0.900 \pm 0.001$
        & $0.758 \pm 0.002$
        & $0.924 \pm 0.001$ \\
        &
        Shannon MI
        & $2.801 \pm 0.019$
        & $0.899 \pm 0.003$
        & $3.399 \pm 0.026$
        & $0.901 \pm 0.001$
        & $0.758 \pm 0.001$
        & $0.924 \pm 0.001$ \\
        &
        Accuracy
        & $2.879 \pm 0.020$
        & $0.899 \pm 0.002$
        & $3.415 \pm 0.022$
        & $0.901 \pm 0.001$
        & $\mathbf{0.763 \pm 0.001}$
        & $0.922 \pm 0.001$ \\
        &
        Random
        & $4.459 \pm 0.184$
        & $0.897 \pm 0.002$
        & $5.054 \pm 0.174$
        & $0.900 \pm 0.001$
        & $0.642 \pm 0.011$
        & $0.871 \pm 0.007$ \\

        \bottomrule
    \end{tabular}%
}

\end{table}

\subsection{Feature selection}

We now investigate conformal set-size reduction as an information gain metric in a sequential feature selection setting with candidate features $\{X_1, \ldots, X_m\}$. At each step, we select the subsequent feature by maximizing the conformal set-size reduction or Shannon mutual information (MI), with top-1 accuracy gain and random selection as baselines. The purpose of this experiment is to empirically probe the theoretical differences between these two criteria, as highlighted in Figure~\ref{fig:crossing_profile}, on real-world datasets, rather than to propose a competitive active feature acquisition method.

Evaluating the selection criteria requires estimating the label distribution for many subsets of candidate features. We amortize this computation by training a single mask-conditioned classifier that can make predictions from arbitrary subsets of observed features~\citep{yoon2019invase,covert2021explaining}. The evaluation pool is split into two halves to determine the feature ordering and evaluate its performance on held-out data. Further details on the analysis procedure are provided in Appendix~\ref{app:sec:analysis}.

We measure the resulting conformal prediction set size and coverage using the probability and APS scores, as well as top-1 and top-5 accuracy. To summarize performance across the selection trajectory, we average each metric over all non-empty proper prefixes of the feature ordering. Table~\ref{tab:feature_selection} summarizes results for one dataset from each domain, with the remaining results in Appendix~\ref{app:sec:experiments}.

Selecting features according to conformal set-size reduction consistently produces the smallest sets, although the relative performance of the probability-score and APS criteria varies. This advantage is partly expected since $\Delta_\alpha$ directly targets set size, but the gap to Shannon MI shows that it can induce different feature orderings on real data. In an ablation on CIFAR-100, this gap persists when we retrain a separate model for each selected feature subset, indicating that it is not solely an artifact of the shared mask-conditioned model. Accuracy-based selection achieves the highest top-1 accuracy, while $\Delta_\alpha$ matches it on top-5 accuracy. We visualize metric trajectories in Appendix~\ref{app:sec:experiments}.

\section{Discussion}

In this work, we provide new theoretical insights relating conformal prediction and information theory. Our main results show that in standard classification settings, $\Delta_\alpha$ is bracketed by calibration-dependent members of the conformal mutual information family, which exactly recovers Shannon mutual information via an integral representation. Moreover, set-size reduction satisfies an approximate data processing inequality, further motivating it as a principled measure of information gain.

\textbf{Limitations.~}~Our theoretical analysis is limited to classification with the probability and APS scores, while our empirical evaluation considers sequential feature selection with a globally fixed ordering. Future work could extend these connections to other conformal scores and regression settings or develop active feature acquisition methods using $\Delta_\alpha$ as a proxy objective.



\bibliography{iclr2027_conference}

@inproceedings{ha1996optimum,
  title={An Optimum Class-Selective Rejection Rule for Pattern Recognition},
  author={Ha, Thien M.},
  booktitle={Proceedings of the 13th International Conference on Pattern Recognition},
  volume={2},
  year={1996},
  publisher={IEEE}
}

@article{herbei2006classification,
  title={Classification with Reject Option},
  author={Herbei, Radu and Wegkamp, Marten H.},
  journal={The Canadian Journal of Statistics},
  volume={34},
  year={2006}
}

@article{chzhen2021setvalued,
  title={Set-valued Classification -- Overview via a Unified Framework},
  author={Chzhen, Evgenii and Denis, Christophe and Hebiri, Mohamed and Lorieul, Titouan},
  journal={CoRR},
  volume={abs/2102.12318},
  year={2021}
}

@article{degroot1962uncertainty,
  title={Uncertainty, Information, and Sequential Experiments},
  author={DeGroot, Morris H.},
  journal={The Annals of Mathematical Statistics},
  volume={33},
  year={1962}
}

@article{schervish1989general,
  title={A General Method for Comparing Probability Assessors},
  author={Schervish, Mark J.},
  journal={The Annals of Statistics},
  volume={17},
  year={1989},
  publisher={Institute of Mathematical Statistics}
}

@article{grunwald2004game,
  title={Game Theory, Maximum Entropy, Minimum Discrepancy and Robust Bayesian Decision Theory},
  author={Gr{\"u}nwald, Peter D. and Dawid, A. Philip},
  journal={The Annals of Statistics},
  volume={32},
  year={2004}
}

@article{gneiting2007strictly,
  title={Strictly Proper Scoring Rules, Prediction, and Estimation},
  author={Gneiting, Tilmann and Raftery, Adrian E.},
  journal={Journal of the American Statistical Association},
  volume={102},
  year={2007}
}

@article{reid2011information,
  title={Information, Divergence and Risk for Binary Experiments},
  author={Reid, Mark D. and Williamson, Robert C.},
  journal={Journal of Machine Learning Research},
  volume={12},
  year={2011}
}

@inproceedings{garcia2012divergences,
  title={Divergences and Risks for Multiclass Experiments},
  author={Garc{\'i}a-Garc{\'i}a, Dario and Williamson, Robert C.},
  booktitle={Proceedings of the 25th Annual Conference on Learning Theory},
  volume={23},
  year={2012}
}

@article{duchi2018multiclass,
  title={Multiclass Classification, Information, Divergence, and Surrogate Risk},
  author={Duchi, John C. and Khosravi, Khashayar and Ruan, Feng},
  journal={The Annals of Statistics},
  volume={46},
  year={2018}
}

@inproceedings{neiswanger2022generalizing,
  title={Generalizing Bayesian Optimization with Decision-theoretic Entropies},
  author={Neiswanger, Willie and Yu, Lantao and Zhao, Shengjia and Meng, Chenlin and Ermon, Stefano},
  booktitle={Advances in Neural Information Processing Systems},
  volume={35},
  year={2022}
}

@article{houlsby2011bayesian,
  title={Bayesian Active Learning for Classification and Preference Learning},
  author={Houlsby, Neil and Husz{\'a}r, Ferenc and Ghahramani, Zoubin and Lengyel, M{\'a}t{\'e}},
  journal={CoRR},
  volume={abs/1112.5745},
  year={2011}
}

@inproceedings{kirsch2019batchbald,
  title={BatchBALD: Efficient and Diverse Batch Acquisition for Deep Bayesian Active Learning},
  author={Kirsch, Andreas and van Amersfoort, Joost and Gal, Yarin},
  booktitle={Advances in Neural Information Processing Systems},
  volume={32},
  year={2019}
}

@inproceedings{bickfordsmith2023prediction,
  title={Prediction-Oriented Bayesian Active Learning},
  author={Bickford Smith, Freddie and Kirsch, Andreas and Farquhar, Sebastian and Gal, Yarin and Foster, Adam and Rainforth, Tom},
  booktitle={Proceedings of the 26th International Conference on Artificial Intelligence and Statistics},
  volume={206},
  year={2023}
}

@inproceedings{papadopoulos2002inductive,
  title={Inductive Confidence Machines for Regression},
  author={Papadopoulos, Harris and Proedrou, Kostas and Vovk, Vladimir and Gammerman, Alex},
  booktitle={Machine Learning: ECML 2002},
  publisher={Springer},
  year={2002}
}

@book{vovk2005algorithmic,
  title={Algorithmic Learning in a Random World},
  author={Vovk, Vladimir and Gammerman, Alex and Shafer, Glenn},
  volume={29},
  publisher={Springer},
  year={2005}
}

@article{shafer2008tutorial,
  title={A Tutorial on Conformal Prediction},
  author={Shafer, Glenn and Vovk, Vladimir},
  journal={Journal of Machine Learning Research},
  volume={9},
  year={2008}
}

@inproceedings{vovk2016criteria,
  title={Criteria of Efficiency for Conformal Prediction},
  author={Vovk, Vladimir and Fedorova, Valentina and Nouretdinov, Ilia and Gammerman, Alex},
  booktitle={Conformal and Probabilistic Prediction with Applications},
  year={2016},
  publisher={Springer}
}

@article{lei2018distribution,
  title={Distribution-Free Predictive Inference for Regression},
  author={Lei, Jing and G'Sell, Max and Rinaldo, Alessandro and Tibshirani, Ryan J. and Wasserman, Larry},
  journal={Journal of the American Statistical Association},
  volume={113},
  year={2018}
}

@article{sadinle2019least,
  title={Least Ambiguous Set-Valued Classifiers With Bounded Error Levels},
  author={Sadinle, Mauricio and Lei, Jing and Wasserman, Larry},
  journal={Journal of the American Statistical Association},
  volume={114},
  year={2019}
}

@inproceedings{romano2019conformalized,
  title={Conformalized Quantile Regression},
  author={Romano, Yaniv and Patterson, Evan and Cand{\`e}s, Emmanuel J.},
  booktitle={Advances in Neural Information Processing Systems},
  volume={32},
  year={2019}
}

@inproceedings{romano2020classification,
  title={Classification with Valid and Adaptive Coverage},
  author={Romano, Yaniv and Sesia, Matteo and Cand{\`e}s, Emmanuel J.},
  booktitle={Advances in Neural Information Processing Systems},
  volume={33},
  year={2020}
}

@inproceedings{angelopoulos2021uncertainty,
  title={Uncertainty Sets for Image Classifiers Using Conformal Prediction},
  author={Angelopoulos, Anastasios N. and Bates, Stephen and Jordan, Michael I. and Malik, Jitendra},
  booktitle={International Conference on Learning Representations},
  year={2021}
}

@article{angelopoulos2021gentle,
  title={A Gentle Introduction to Conformal Prediction and Distribution-Free Uncertainty Quantification},
  author={Angelopoulos, Anastasios N. and Bates, Stephen},
  journal={CoRR},
  volume={abs/2107.07511},
  year={2021}
}

@book{angelopoulos2025theoretical,
  title={Theoretical Foundations of Conformal Prediction},
  author={Angelopoulos, Anastasios N. and Barber, Rina Foygel and Bates, Stephen},
  publisher={Cambridge University Press},
  year={2025},
}

@inproceedings{lu2022fair,
  title={Fair Conformal Predictors for Applications in Medical Imaging},
  author={Lu, Charles and Lemay, Andr{\'e}anne and Chang, Ken and H{\"o}bel, Katharina and Kalpathy-Cramer, Jayashree},
  booktitle={Proceedings of the AAAI Conference on Artificial Intelligence},
  volume={36},
  year={2022},
}

@article{lindemann2023safe,
  title={Safe Planning in Dynamic Environments Using Conformal Prediction},
  author={Lindemann, Lars and Cleaveland, Matthew and Shim, Gihyun and Pappas, George J.},
  journal={IEEE Robotics and Automation Letters},
  volume={8},
  year={2023},
}

@inproceedings{kumar2023conformal,
  title={Conformal Prediction with Large Language Models for Multi-Choice Question Answering},
  author={Lu, Charles and Kumar, Bhawesh and Gupta, Gauri and Palepu, Anil and Bellamy, David and Raskar, Ramesh and Beam, Andrew},
  booktitle={ICML 2023 Workshop on Neural Conversational AI: What's Left to TEACH (Trustworthy, Enhanced, Adaptable, Capable and Human-centric) Chatbots?},
  year={2023}
}

@inproceedings{cresswell2024conformal,
  title={Conformal Prediction Sets Improve Human Decision Making},
  author={Cresswell, Jesse C. and Sui, Yi and Kumar, Bhargava and Vouitsis, No{\"e}l},
  booktitle={Proceedings of the 41st International Conference on Machine Learning},
  volume={235},
  year={2024}
}

@inproceedings{correia2024information,
  title={An Information Theoretic Perspective on Conformal Prediction},
  author={Correia, Alvaro H. C. and Massoli, Fabio Valerio and Louizos, Christos and Behboodi, Arash},
  booktitle={Advances in Neural Information Processing Systems},
  volume={37},
  year={2024}
}

@article{fayyad2024empirical,
  title={Empirical Validation of Conformal Prediction for Trustworthy Skin Lesions Classification},
  author={Fayyad, Jamil and Alijani, Shadi and Najjaran, Homayoun},
  journal={Computer Methods and Programs in Biomedicine},
  volume={253},
  year={2024}
}

@inproceedings{dhillon2024expected,
  title={On the Expected Size of Conformal Prediction Sets},
  author={Dhillon, Guneet S. and Deligiannidis, George and Rainforth, Tom},
  booktitle={Proceedings of the 27th International Conference on Artificial Intelligence and Statistics},
  volume={238},
  year={2024}
}

@inproceedings{zecchin2024generalization,
  title={Generalization and Informativeness of Conformal Prediction},
  author={Zecchin, Matteo and Park, Sangwoo and Simeone, Osvaldo and Hellstr{\"o}m, Fredrik},
  booktitle={2024 IEEE International Symposium on Information Theory},
  year={2024},
  publisher={IEEE}
}

@inproceedings{kharazian2024copal,
  title={CoPAL: Conformal Prediction in Active Learning An Algorithm for Enhancing Remaining Useful Life Estimation in Predictive Maintenance},
  author={Kharazian, Zahra and Lindgren, Tony and Magnusson, Sindri and Boström, Henrik},
  booktitle={Proceedings of the Thirteenth Symposium on Conformal and Probabilistic Prediction with Applications},
  volume={230},
  year={2024},
  publisher={PMLR}
}

@inproceedings{quach2024conformal,
  title={Conformal Language Modeling},
  author={Quach, Victor and Fisch, Adam and Schuster, Tal and Yala, Adam and Sohn, Jae Ho and Jaakkola, Tommi and Barzilay, Regina},
  booktitle={International Conference on Learning Representations},
  year={2024}
}

@article{portela2025conformal,
  title={Conformal Prediction for Uncertainty Quantification in Dynamic Biological Systems},
  author={Portela, Alberto and Banga, Julio R. and Matabuena, Marcos},
  journal={PLOS Computational Biology},
  volume={21},
  year={2025}
}

@inproceedings{hagos2026performance,
  title={Performance of Conformal Prediction in Capturing Aleatoric Uncertainty},
  author={Hagos, Misgina Tsighe and Lundstr{\"o}m, Claes},
  booktitle={Proceedings of the IEEE/CVF Winter Conference on Applications of Computer Vision},
  year={2026}
}

@article{marchal2017subgaussianity,
  title={On the Sub-Gaussianity of the Beta and Dirichlet Distributions},
  author={Marchal, Olivier and Arbel, Julyan},
  journal={Electronic Communications in Probability},
  volume={22},
  year={2017},
}

@inproceedings{deng2009imagenet,
  title={ImageNet: A Large-Scale Hierarchical Image Database},
  author={Deng, Jia and Dong, Wei and Socher, Richard and Li, Li-Jia and Li, Kai and Fei-Fei, Li},
  booktitle={IEEE Conference on Computer Vision and Pattern Recognition},
  year={2009}
}

@techreport{krizhevsky2009learning,
  title={Learning Multiple Layers of Features from Tiny Images},
  author={Krizhevsky, Alex},
  institution={University of Toronto},
  year={2009}
}

@inproceedings{he2016deep,
  title={Deep Residual Learning for Image Recognition},
  author={He, Kaiming and Zhang, Xiangyu and Ren, Shaoqing and Sun, Jian},
  booktitle={Proceedings of the IEEE Conference on Computer Vision and Pattern Recognition},
  year={2016}
}

@article{xiao2017fashion,
  title={Fashion-MNIST: A Novel Image Dataset for Benchmarking Machine Learning Algorithms},
  author={Xiao, Han and Rasul, Kashif and Vollgraf, Roland},
  journal={CoRR},
  volume={abs/1708.07747},
  year={2017}
}

@article{chrabaszcz2017downsampled,
  title={A Downsampled Variant of ImageNet as an Alternative to the CIFAR Datasets},
  author={Chrabaszcz, Patryk and Loshchilov, Ilya and Hutter, Frank},
  journal={CoRR},
  volume={abs/1707.08819},
  year={2017}
}

@misc{cortez2009wine,
  title={Wine Quality},
  author={Cortez, Paulo and Cerdeira, Ant{\'o}nio and Almeida, Fernando and Matos, Telmo and Reis, Jos{\'e}},
  year={2009},
  howpublished={UCI Machine Learning Repository}
}

@misc{reyesortiz2013human,
  title={Human Activity Recognition Using Smartphones},
  author={Reyes-Ortiz, Jorge Luis and Anguita, Davide and Ghio, Alessandro and Oneto, Luca and Parra, Xavier},
  year={2013},
  howpublished={UCI Machine Learning Repository}
}

@misc{slate1991letter,
  title={Letter Recognition},
  author={Slate, David},
  year={1991},
  howpublished={UCI Machine Learning Repository}
}

@misc{blackard1998covertype,
  title={Covertype},
  author={Blackard, Jock},
  year={1998},
  howpublished={UCI Machine Learning Repository}
}

@inproceedings{shim2018joint,
  title={Joint Active Feature Acquisition and Classification with Variable-Size Set Encoding},
  author={Shim, Hajin and Hwang, Sung Ju and Yang, Eunho},
  booktitle={Advances in Neural Information Processing Systems},
  volume={31},
  year={2018}
}

@inproceedings{yoon2019invase,
  title={INVASE: Instance-wise Variable Selection Using Neural Networks},
  author={Yoon, Jinsung and Jordon, James and van der Schaar, Mihaela},
  booktitle={International Conference on Learning Representations},
  year={2019}
}

@article{covert2021explaining,
  title={Explaining by Removing: A Unified Framework for Model Explanation},
  author={Covert, Ian and Lundberg, Scott and Lee, Su-In},
  journal={Journal of Machine Learning Research},
  volume={22},
  year={2021}
}

@inproceedings{covert2023learning,
  title={Learning to Maximize Mutual Information for Dynamic Feature Selection},
  author={Covert, Ian Connick and Qiu, Wei and Lu, Mingyu and Kim, Na Yoon and White, Nathan J. and Lee, Su-In},
  booktitle={Proceedings of the 40th International Conference on Machine Learning},
  volume={202},
  pages={6424--6447},
  year={2023},
  publisher={PMLR}
}
\bibliographystyle{iclr2027_conference}

\clearpage

\appendix
\section*{Supplementary Material for Conformal Prediction Sets Quantify Information Gain: A Theoretical Perspective}

\section{Proofs} \label{app:sec:proofs}

We first establish several useful results about conformal entropy and mutual information, followed by concentration results for the finite-sample miscoverage terms. We then prove our main results on conformal prediction with the probability score in Appendix~\ref{app:subsec:prob_score_proofs}, and extend these results to the APS score in Appendix~\ref{app:subsec:aps_score_proofs}.

\subsection{Properties of conformal entropy and mutual information} \label{app:proofs:conformal_entropy_props}

In this section, we first prove our results on optimal set-valued predictors for conformal entropy and derive the integral representation of Shannon mutual information. We also establish a data processing inequality for generalized mutual information.

\subsubsection{Proposition~\ref{prop:optimal_predictor_prob_score}: Optimal set-valued predictors} \label{app:proof:optimal_predictor_prob_score}

Consider the risk for a fixed $\lambda \in (0,1]$. Using Equation~\ref{eqn:conformal_loss},
\begin{equation}
    \mathbb{E} \left[ \ell_\lambda(Y, \Gamma(X)) \right] = \mathbb{E} [|\Gamma(X)|] + \tfrac{1}{\lambda}\, \mathbb{P}(Y \notin \Gamma(X)).
\end{equation}
We can decompose the risk into a sum over the elements of $\Gamma(X)$ and its complement. Since $p(\cdot \mid X)$ is the true conditional probability of $Y$ given $X$,
\begin{equation} \label{eqn:risk_decomposition_prob_score_helper}
    \mathbb{E} [|\Gamma(X)|] + \tfrac{1}{\lambda}\,\mathbb{P}(Y \notin \Gamma(X)) = \mathbb{E} \left[ \sum_{y \in \Gamma(X)} 1 + \tfrac{1}{\lambda} \sum_{y \not\in \Gamma(X)} p(y \mid X) \right].
\end{equation}
Using $\sum_{y \in \mathcal{Y}} p(y \mid X) = 1$, we have
\begin{equation} \label{eqn:risk_decomposition_prob_score}
    \mathbb{E}_X \left[ \mathbb{E}_{p(\cdot \mid X)} [\ell_\lambda(Y, \Gamma(X))] \right] = \tfrac{1}{\lambda} + \mathbb{E} \left[ \sum_{y \in \Gamma(X)} \Bigl(1 - \tfrac{1}{\lambda}\, p(y \mid X)\Bigr) \right].
\end{equation}
Hence, $\Gamma(X)$ minimizes the risk pointwise in $x$ if and only if it includes every $y$ with $p(y \mid x) > \lambda$ and excludes every $y$ with $p(y \mid x) < \lambda$. The decision rule that outputs probability superlevel sets
\begin{equation}
    \Gamma^*_\lambda(X; p) = \{y \in \mathcal{Y} : p(y \mid X) \geq \lambda \}
\end{equation}
is one such minimizer for every $x \in \mathcal{X}$ and thus achieves the conformal entropy $H_\lambda(Y \mid X)$.

\subsubsection{Theorem~\ref{thm:integral}: Integral representation of shannon information} \label{app:proof:integral}

We first establish the corresponding integral representation for the Shannon entropy.
\begin{lemma}[Integral representation of Shannon entropy] \label{lemma:integral_entropy}
    Let $H$ denote the Shannon entropy (in nats). For any random variable $X$,
    \begin{equation} 
        1 + H(Y \mid X) = \int_0^1 H_\lambda(Y \mid X) \,d\lambda.
    \end{equation}
\end{lemma}
\begin{proof}
    Recall that the decision rule $\Gamma_\lambda^*(X; p) = \{y : p(y \mid X) \geq \lambda \}$ achieves $H_\lambda(Y \mid X)$. At this optimum value, each $y \in \mathcal{Y}$ either increases the set size by one when $p(y \mid x) \geq \lambda$ or increases the miscoverage penalty term by $p(y \mid x) / \lambda$ otherwise. Hence,
    \begin{equation}
        H_\lambda(Y \mid X) = \mathbb{E}_X \left[ \sum_{y \in \mathcal{Y}} \min \left\{ 1, \frac{p(y \mid X)}{\lambda} \right\} \right],
    \end{equation}
    and therefore by Tonelli's theorem,
    \begin{equation} \label{eqn:entropy_integral_expansion}
        \int_0^1 H_\lambda(Y \mid X) \,d\lambda = \mathbb{E}_X \left[ \sum_{y \in \mathcal{Y}} \; \int_0^1 \min \left\{ 1, \frac{p(y \mid X)}{\lambda} \right\} \,d\lambda \right].
    \end{equation}
    Evaluating the inner integral,
    \begin{equation} \label{eqn:inner_integral_evaluation}
        \int_0^1 \min \left\{ 1, \frac{p(y \mid X)}{\lambda} \right\} \,d\lambda = \underbrace{\int_0^{p(y \mid X)} d\lambda}_{p(y \,\mid\, X)} + \underbrace{\int_{p(y \mid X)}^1 \frac{p(y \mid X)}{\lambda} \,d\lambda}_{- p(y \,\mid\, X) \log p(y \,\mid\, X)}.
    \end{equation}
    Substituting Equation~\ref{eqn:inner_integral_evaluation} into Equation~\ref{eqn:entropy_integral_expansion},
    \begin{equation}
        \int_0^1 H_\lambda(Y \mid X) \,d\lambda = \mathbb{E}_X \left[ \sum_{y \in \mathcal{Y}} p(y \mid X) - p(y \mid X) \log p(y \mid X) \right] = 1 + H(Y \mid X).
    \end{equation}
\end{proof}
The integral representation of Shannon information now follows immediately from Lemma~\ref{lemma:integral_entropy}.
\begin{equation}
    H(Y \mid Z) - H(Y \mid X, Z) = \int_0^1 H_\lambda(Y \mid Z) \,d\lambda - \int_0^1 H_\lambda(Y \mid X, Z) \,d\lambda.
\end{equation}
Thus,
\begin{equation}
    I(X; Y \mid Z) = \int_0^1 I_\lambda(X; Y \mid Z) \,d\lambda.
\end{equation}

\subsubsection{Proposition~\ref{prop:optimal_predictor_aps}: Optimal set-valued predictors for random \texorpdfstring{$\Lambda$}{Lambda}}

Consider the risk when $\Lambda = \xi(X)$, where $\xi: \mathcal{X} \to (0,1]$ is a fixed measurable function. Using Equation~\ref{eqn:conformal_loss} with a fixed value of $u \in [0,1]$,
\begin{equation}
    \mathbb{E} \left[ \ell_\Lambda (Y, \Gamma(X, u)) \right] = \mathbb{E}_X [|\Gamma(X, u)|] + \mathbb{E}_X \left[ \tfrac{1}{\xi(X)} \mathbb{P}(Y \notin \Gamma(X, u) \mid X) \right].
\end{equation}
Since $p(\cdot \mid X)$ is the true conditional probability of $Y$ given $X$, the risk decomposition in Equations~\ref{eqn:risk_decomposition_prob_score_helper}~and~\ref{eqn:risk_decomposition_prob_score} gives that
\begin{equation}
    \mathbb{E} \left[ \ell_\Lambda (Y, \Gamma(X, u)) \right] = \mathbb{E}_X \left[ \tfrac{1}{\xi(X)} + \sum_{y \in \Gamma(X, u)} \left( 1 - \tfrac{1}{\xi(X)} p(y \mid X) \right) \right].
\end{equation}
Hence, $\Gamma(X, u)$ minimizes the risk pointwise in $x$ if and only if it includes every $y$ with \mbox{$p(y \mid x) > \xi(x)$} and excludes every $y$ with $p(y \mid x) < \xi(x)$. Thus, any decision rule $\Gamma^*_{\Lambda}(X, U; p)$ such that, for all $u \in [0,1]$,
\begin{equation}
    \bigl\{y : p(y \mid X) > \xi(X) \bigr\} \subseteq \Gamma^*_{\Lambda} (X, u; p) \subseteq \bigl\{y : p(y \mid X) \geq \xi(X)\bigr\},
\end{equation}
minimizes $\mathbb{E} \left[ \ell_\Lambda (Y, \Gamma(X, U)) \right]$ and thus achieves the conformal entropy $H_\Lambda(Y \mid X)$. The decision rule may choose to include or exclude labels at the boundary $p(y \mid X) = \xi(X)$ depending on $U$.

\subsubsection{Lemma~\ref{lemma:dpi_generalized_entropy}: Generalized information dpi}

The data processing inequality for conformal prediction sets follows from a more general result for generalized mutual information. For completeness, we include the result here. Let $H_\ell$ and $I_\ell$ denote the generalized entropy and mutual information under a loss $\ell : \mathcal{Y} \times \mathcal{A} \to \mathbb{R}$, as defined in Section~\ref{sec:background}.
\begin{lemma}[Generalized information DPI] \label{lemma:dpi_generalized_entropy}
    Let $\ell$ be any loss with finite action space $\mathcal{A}$, and suppose $Y \to X \to \tilde{X}$ forms a Markov chain. Then,
    \begin{equation}
        H_\ell(Y \mid X) \leq H_\ell(Y \mid \tilde{X}).
    \end{equation}
    Equivalently, $I_\ell(\tilde{X}; Y) \leq I_\ell(X; Y)$.
\end{lemma}
\begin{proof}
    Fix any decision rule $g$ for the $\tilde{X}$-problem. Since $Y \perp \tilde{X} \mid X$,
    \begin{equation}
        \mathbb{E} \bigl[ \ell(Y, g(\tilde{X})) \mid X, \tilde{X} = \tilde{x} \bigr] = \mathbb{E} \bigl[ \ell(Y, g(\tilde{x})) \mid X \bigr] \geq \inf_{a \in \mathcal{A}} \mathbb{E} \bigl[ \ell(Y,a) \mid X \bigr].
    \end{equation}
    Taking the joint expectation over $(X,\tilde{X})$,
    \begin{equation}
        \mathbb{E} \left[ \inf_{a \in \mathcal{A}} \mathbb{E} \bigl[ \ell(Y,a) \mid X \bigr] \right] \leq \mathbb{E} \bigl[ \ell(Y,g(\tilde{X})) \bigr].
    \end{equation}
    Therefore,
    \begin{equation}
        H_\ell(Y \mid X) \leq \mathbb{E}\bigl[ \ell(Y,g(\tilde{X})) \bigr].
    \end{equation}
    Minimizing the right-hand side over $g$ gives $H_\ell(Y \mid X) \leq H_\ell(Y \mid \tilde{X})$. The equivalent form expressed in terms of mutual information follows by subtracting both sides from $H_\ell(Y)$.
\end{proof}

\subsection{Probability score results} \label{app:subsec:prob_score_proofs}

We now prove our main results for the probability score in Section~\ref{subsec:probability_score_info}.

\subsubsection{Theorem~\ref{thm:sandwich_prob_score}: Probability score sandwich bound}

We wish to bound the expected conformal set-size reduction
\begin{equation}
    \Delta_{\alpha} = \mathbb{E}\bigl[|\mathcal{C}_\alpha(X^{\scriptscriptstyle (1)})|\bigr] - \mathbb{E}\bigl[|\mathcal{C}_\alpha(X^{\scriptscriptstyle (2)})|\bigr]
\end{equation}
in terms of the conformal mutual information $I_\lambda$, for a fixed realization of the calibration data $\mathcal{D}_{\mathrm{cal}}$. Since $\mathcal{D}_{\mathrm{cal}}$ is fixed, we treat the calibrated thresholds $\smash{\hat\lambda_\alpha^{\scriptscriptstyle (1)}, \hat\lambda_\alpha^{\scriptscriptstyle (2)}}$ and the training-conditional miscoverage rates $\varepsilon^{\scriptscriptstyle (1)}, \varepsilon^{\scriptscriptstyle (2)}$ as constants. A key intermediate quantity in our analysis is
\begin{equation}
    \rho \coloneqq H_{\hat\lambda_\alpha^{\scriptscriptstyle (1)}}(Y \mid X^{\scriptscriptstyle (1)}) - H_{\hat\lambda_\alpha^{\scriptscriptstyle (2)}}(Y \mid X^{\scriptscriptstyle (2)}),
\end{equation}
which compares the conformal entropies at the two different calibrated thresholds. We first establish a sandwich bound on $\rho$, then convert it into a bound on $\Delta_{\alpha}$ using the relationship between conformal prediction sets and Bayes-optimal actions.
\begin{lemma}[Probability score sandwich bound on $\rho$] \label{lemma:rho_sandwich_prob_score}
    Fix a calibration dataset $\mathcal{D}_\mathrm{cal}$ and nominal coverage level $1 - \alpha$. For $i \in \{1, 2\}$, define $\smash{\hat{\lambda}_\alpha^{\scriptscriptstyle (i)}}$, $I^{\scriptscriptstyle (i)}$, $D^{\scriptscriptstyle (i)}$, and $\varepsilon^{\scriptscriptstyle (i)}$ as in Theorem~\ref{thm:sandwich_prob_score}, and assume that $\smash{\hat{\lambda}_\alpha^{\scriptscriptstyle (i)}} > 0$. Then,
    \begin{equation}\label{eqn:rho_bound_prob_score}
        I^{\scriptscriptstyle (2)} - D^{\scriptscriptstyle (1)} + \varepsilon^{\scriptscriptstyle (1)} \!\left(\frac{1}{\hat\lambda_\alpha^{\scriptscriptstyle (1)}} - \frac{1}{\hat\lambda_\alpha^{\scriptscriptstyle (2)}}\right) \;\leq\; \rho \;\leq\; I^{\scriptscriptstyle (1)} + D^{\scriptscriptstyle (2)} + \varepsilon^{\scriptscriptstyle (2)} \!\left(\frac{1}{\hat\lambda_\alpha^{\scriptscriptstyle (1)}} - \frac{1}{\hat\lambda_\alpha^{\scriptscriptstyle (2)}}\right).
    \end{equation}
\end{lemma}
\begin{proof}
    In both bounds we insert an intermediate conformal entropy term that isolates $I_\lambda$ and leaves a cross-threshold conformal entropy difference that we control with an appropriately chosen predictor. Let $p(\cdot \mid X)$ be the true conditional distribution and $q(\cdot \mid X)$ be the predicted conditional distribution. We start with the upper bound.

    \textit{Upper bound:} Adding and subtracting $H_{\hat\lambda_\alpha^{\scriptscriptstyle (1)}}(Y \mid X^{\scriptscriptstyle (2)})$ decomposes $\rho$ as
    \begin{equation} \label{eqn:rho_upper_bound_prob_score_decomp}
        \rho = I^{\scriptscriptstyle (1)} + H_{\hat\lambda_\alpha^{\scriptscriptstyle (1)}}(Y \mid X^{\scriptscriptstyle (2)}) - H_{\hat\lambda_\alpha^{\scriptscriptstyle (2)}}(Y \mid X^{\scriptscriptstyle (2)}).
    \end{equation}
    It remains to upper bound the cross-threshold conformal entropy difference. By the definition of the conformal information divergence,
    \begin{equation} \label{eqn:rho_upper_bound_prob_score_lam2_identity}
        H_{\hat\lambda_\alpha^{\scriptscriptstyle (2)}}(Y \mid X^{\scriptscriptstyle (2)}) = \mathbb{E}\bigl[|\Gamma^*_{\hat\lambda_\alpha^{\scriptscriptstyle (2)}}(X^{\scriptscriptstyle (2)}; q)|\bigr] + \frac{1}{\hat\lambda_\alpha^{\scriptscriptstyle (2)}}\, \mathbb{P}\bigl(Y \notin \Gamma^*_{\hat\lambda_\alpha^{\scriptscriptstyle (2)}}(X^{\scriptscriptstyle (2)}; q)\bigr) - D^{\scriptscriptstyle (2)}.
    \end{equation}
    Under the $\ell_{\hat\lambda_\alpha^{\scriptscriptstyle (1)}}$ loss, the same predictor has risk no smaller than $\smash{H_{\hat\lambda_\alpha^{\scriptscriptstyle (1)}}(Y \mid X^{\scriptscriptstyle (2)})}$. Thus,
    \begin{equation} \label{eqn:rho_upper_bound_prob_score_lam1_bound}
        H_{\hat\lambda_\alpha^{\scriptscriptstyle (1)}}(Y \mid X^{\scriptscriptstyle (2)}) \leq \mathbb{E}\bigl[|\Gamma^*_{\hat\lambda_\alpha^{\scriptscriptstyle (2)}}(X^{\scriptscriptstyle (2)}; q)|\bigr] + \frac{1}{\hat\lambda_\alpha^{\scriptscriptstyle (1)}}\,\mathbb{P}\bigl(Y \notin \Gamma^*_{\hat\lambda_\alpha^{\scriptscriptstyle (2)}}(X^{\scriptscriptstyle (2)}; q)\bigr).
    \end{equation}
    By Proposition~\ref{prop:optimal_predictor_prob_score}, we have $\Gamma^*_{\hat\lambda_\alpha^{\scriptscriptstyle (2)}}(X^{\scriptscriptstyle (2)}; q) = \mathcal{C}_\alpha(X^{\scriptscriptstyle (2)})$, so the miscoverage probabilities in Equations~\ref{eqn:rho_upper_bound_prob_score_lam1_bound} and~\ref{eqn:rho_upper_bound_prob_score_lam2_identity} are both $\varepsilon^{\scriptscriptstyle (2)}$. Subtracting Equation~\ref{eqn:rho_upper_bound_prob_score_lam2_identity} from Equation~\ref{eqn:rho_upper_bound_prob_score_lam1_bound} gives
    \begin{equation} \label{eqn:rho_upper_bound_prob_score_crossH_bound}
        H_{\hat\lambda_\alpha^{\scriptscriptstyle (1)}}(Y \mid X^{\scriptscriptstyle (2)}) - H_{\hat\lambda_\alpha^{\scriptscriptstyle (2)}}(Y \mid X^{\scriptscriptstyle (2)}) \leq D^{\scriptscriptstyle (2)} + \varepsilon^{\scriptscriptstyle (2)} \!\left(\frac{1}{\hat\lambda_\alpha^{\scriptscriptstyle (1)}} - \frac{1}{\hat\lambda_\alpha^{\scriptscriptstyle (2)}}\right).
    \end{equation}
    Substituting Equation~\ref{eqn:rho_upper_bound_prob_score_crossH_bound} into~\ref{eqn:rho_upper_bound_prob_score_decomp} establishes the upper bound.

    \textit{Lower bound:} Adding and subtracting $H_{\hat\lambda_\alpha^{\scriptscriptstyle (2)}}(Y \mid X^{\scriptscriptstyle (1)})$ instead decomposes $\rho$ as
    \begin{equation} \label{eqn:rho_lower_bound_prob_score_decomp}
        \rho = I^{\scriptscriptstyle (2)} + H_{\hat\lambda_\alpha^{\scriptscriptstyle (1)}}(Y \mid X^{\scriptscriptstyle (1)}) - H_{\hat\lambda_\alpha^{\scriptscriptstyle (2)}}(Y \mid X^{\scriptscriptstyle (1)}).
    \end{equation}
    It remains to lower bound the cross-threshold conformal entropy difference. By the definition of the conformal information divergence,
    \begin{equation} \label{eqn:rho_lower_bound_prob_score_lam1_identity}
        H_{\hat\lambda_\alpha^{\scriptscriptstyle (1)}}(Y \mid X^{\scriptscriptstyle (1)}) = \mathbb{E}\bigl[|\Gamma^*_{\hat\lambda_\alpha^{\scriptscriptstyle (1)}}(X^{\scriptscriptstyle (1)}; q)|\bigr] + \frac{1}{\hat\lambda_\alpha^{\scriptscriptstyle (1)}}\, \mathbb{P}\bigl(Y \notin \Gamma^*_{\hat\lambda_\alpha^{\scriptscriptstyle (1)}}(X^{\scriptscriptstyle (1)}; q)\bigr) - D^{\scriptscriptstyle (1)}.
    \end{equation}
    Under the $\ell_{\hat\lambda_\alpha^{\scriptscriptstyle (2)}}$ loss, the same predictor has risk no smaller than $\smash{H_{\hat\lambda_\alpha^{\scriptscriptstyle (2)}}(Y \mid X^{\scriptscriptstyle (1)})}$. Thus,
    \begin{equation} \label{eqn:rho_lower_bound_prob_score_lam2_bound}
        H_{\hat\lambda_\alpha^{\scriptscriptstyle (2)}}(Y \mid X^{\scriptscriptstyle (1)}) \leq \mathbb{E}\bigl[|\Gamma^*_{\hat\lambda_\alpha^{\scriptscriptstyle (1)}}(X^{\scriptscriptstyle (1)}; q)|\bigr] + \frac{1}{\hat\lambda_\alpha^{\scriptscriptstyle (2)}}\,\mathbb{P}\bigl(Y \notin \Gamma^*_{\hat\lambda_\alpha^{\scriptscriptstyle (1)}}(X^{\scriptscriptstyle (1)}; q)\bigr).
    \end{equation}
    By Proposition~\ref{prop:optimal_predictor_prob_score}, we have $\Gamma^*_{\hat\lambda_\alpha^{\scriptscriptstyle (1)}}(X^{\scriptscriptstyle (1)}; q) = \mathcal{C}_\alpha(X^{\scriptscriptstyle (1)})$, so the miscoverage probabilities in Equations~\ref{eqn:rho_lower_bound_prob_score_lam2_bound} and~\ref{eqn:rho_lower_bound_prob_score_lam1_identity} are both $\varepsilon^{\scriptscriptstyle (1)}$. Subtracting Equation~\ref{eqn:rho_lower_bound_prob_score_lam2_bound} from Equation~\ref{eqn:rho_lower_bound_prob_score_lam1_identity} gives
    \begin{equation} \label{eqn:rho_lower_bound_prob_score_crossH_bound}
        H_{\hat\lambda_\alpha^{\scriptscriptstyle (1)}}(Y \mid X^{\scriptscriptstyle (1)}) - H_{\hat\lambda_\alpha^{\scriptscriptstyle (2)}}(Y \mid X^{\scriptscriptstyle (1)}) \geq \varepsilon^{\scriptscriptstyle (1)} \!\left(\frac{1}{\hat\lambda_\alpha^{\scriptscriptstyle (1)}} - \frac{1}{\hat\lambda_\alpha^{\scriptscriptstyle (2)}}\right) - D^{\scriptscriptstyle (1)}.
    \end{equation}
    Substituting Equation~\ref{eqn:rho_lower_bound_prob_score_crossH_bound} into Equation~\ref{eqn:rho_lower_bound_prob_score_decomp} establishes the lower bound.
\end{proof}

We now convert the two-sided bound on $\rho$ in Lemma~\ref{lemma:rho_sandwich_prob_score} into the corresponding bound on $\Delta_\alpha$ in Theorem~\ref{thm:sandwich_prob_score}. Using Proposition~\ref{prop:optimal_predictor_prob_score}, we can write the conformal entropy at $\hat\lambda_\alpha^{\scriptscriptstyle (1)}$ and $\hat\lambda_\alpha^{\scriptscriptstyle (2)}$ as
\begin{equation} \label{eqn:delta_identity_H1}
    H_{\hat\lambda_\alpha^{\scriptscriptstyle (1)}}(Y \mid X^{\scriptscriptstyle (1)}) = \mathbb{E}\bigl[|\mathcal{C}_\alpha(X^{\scriptscriptstyle (1)})|\bigr] + \frac{1}{\hat\lambda_\alpha^{\scriptscriptstyle (1)}}\, \mathbb{P}\bigl(Y \notin \mathcal{C}_\alpha(X^{\scriptscriptstyle (1)}) ; \mathcal{D}_{\mathrm{cal}}\bigr) - D^{\scriptscriptstyle (1)},
\end{equation}
and
\begin{equation} \label{eqn:delta_identity_H2}
    H_{\hat\lambda_\alpha^{\scriptscriptstyle (2)}}(Y \mid X^{\scriptscriptstyle (2)}) = \mathbb{E}\bigl[|\mathcal{C}_\alpha(X^{\scriptscriptstyle (2)})|\bigr] + \frac{1}{\hat\lambda_\alpha^{\scriptscriptstyle (2)}}\, \mathbb{P}\bigl(Y \notin \mathcal{C}_\alpha(X^{\scriptscriptstyle (2)}) ; \mathcal{D}_{\mathrm{cal}}\bigr) - D^{\scriptscriptstyle (2)}.
\end{equation}
Recall $\rho = H_{\hat\lambda_\alpha^{\scriptscriptstyle (1)}}(Y \mid X^{\scriptscriptstyle (1)}) - H_{\hat\lambda_\alpha^{\scriptscriptstyle (2)}}(Y \mid X^{\scriptscriptstyle (2)})$, and the miscoverage probabilities in Equations~\ref{eqn:delta_identity_H1} and~\ref{eqn:delta_identity_H2} are $\varepsilon^{\scriptscriptstyle (1)}$ and $\varepsilon^{\scriptscriptstyle (2)}$, respectively. Subtracting the two identities,
\begin{equation} \label{eqn:delta_identity}
    \Delta_\alpha = \rho - \frac{\varepsilon^{\scriptscriptstyle (1)}}{\hat\lambda_\alpha^{\scriptscriptstyle (1)}} + \frac{\varepsilon^{\scriptscriptstyle (2)}}{\hat\lambda_\alpha^{\scriptscriptstyle (2)}} + D^{\scriptscriptstyle (1)} - D^{\scriptscriptstyle (2)}.
\end{equation}
We combine Lemma~\ref{lemma:rho_sandwich_prob_score} with the identity in Equation~\ref{eqn:delta_identity} to obtain the corresponding bounds on $\Delta_\alpha$.

\textit{Upper bound:} Substituting the upper bound from Equation~\ref{eqn:rho_bound_prob_score},
\begin{equation}
    \Delta_{\alpha} \leq I^{\scriptscriptstyle (1)} + D^{\scriptscriptstyle (2)} + \varepsilon^{\scriptscriptstyle (2)} \!\left(\frac{1}{\hat\lambda_\alpha^{\scriptscriptstyle (1)}} - \frac{1}{\hat\lambda_\alpha^{\scriptscriptstyle (2)}}\right) - \frac{\varepsilon^{\scriptscriptstyle (1)}}{\hat\lambda_\alpha^{\scriptscriptstyle (1)}} + \frac{\varepsilon^{\scriptscriptstyle (2)}}{\hat\lambda_\alpha^{\scriptscriptstyle (2)}} + D^{\scriptscriptstyle (1)} - D^{\scriptscriptstyle (2)}.
\end{equation}
Therefore,
\begin{equation} \label{eqn:sandwich_upper_bound_prob_score}
    \Delta_{\alpha} \leq I^{\scriptscriptstyle (1)} + D^{\scriptscriptstyle (1)} + \frac{\varepsilon^{\scriptscriptstyle (2)} - \varepsilon^{\scriptscriptstyle (1)}}{\hat\lambda_\alpha^{\scriptscriptstyle (1)}}.
\end{equation}

\textit{Lower bound:} Substituting the lower bound from Equation~\ref{eqn:rho_bound_prob_score},
\begin{equation}
    \Delta_{\alpha} \geq I^{\scriptscriptstyle (2)} - D^{\scriptscriptstyle (1)} + \varepsilon^{\scriptscriptstyle (1)} \!\left(\frac{1}{\hat\lambda_\alpha^{\scriptscriptstyle (1)}} - \frac{1}{\hat\lambda_\alpha^{\scriptscriptstyle (2)}}\right) - \frac{\varepsilon^{\scriptscriptstyle (1)}}{\hat\lambda_\alpha^{\scriptscriptstyle (1)}} + \frac{\varepsilon^{\scriptscriptstyle (2)}}{\hat\lambda_\alpha^{\scriptscriptstyle (2)}} + D^{\scriptscriptstyle (1)} - D^{\scriptscriptstyle (2)}.
\end{equation}
Therefore,
\begin{equation} \label{eqn:sandwich_lower_bound_prob_score}
    \Delta_{\alpha} \geq I^{\scriptscriptstyle (2)} - D^{\scriptscriptstyle (2)} + \frac{\varepsilon^{\scriptscriptstyle (2)} - \varepsilon^{\scriptscriptstyle (1)}}{\hat\lambda_\alpha^{\scriptscriptstyle (2)}}.
\end{equation}
This completes the proof.

\subsubsection{Concentration of finite-sample miscoverage difference}

The bounds in Equations~\ref{eqn:sandwich_lower_bound_prob_score} and~\ref{eqn:sandwich_upper_bound_prob_score} contain the finite-sample calibration terms $\varepsilon^{\scriptscriptstyle \Delta} / \hat{\lambda}^{\scriptscriptstyle (i)}_\alpha$ for $i \in \{1, 2\}$ involving the difference in miscoverage rates. We now turn our attention to these terms under a random calibration dataset. Using standard arguments from conformal prediction, the marginal distribution of the coverage $1 - \varepsilon^{\scriptscriptstyle (i)}$ is that of an order statistic of a uniform sample. We make this result precise below and include the proof for completeness.\footnote{We follow the proof given in \citet{angelopoulos2025theoretical}; see also~\citet{vovk2005algorithmic}.}
\begin{lemma}[Training-conditional coverage] \label{lemma:coverage_beta}
    Fix a nominal coverage level $1 - \alpha \in (0,1)$ such that $k = \lceil (1-\alpha)(n+1) \rceil \leq n$. Assume that the calibration scores are almost surely distinct and that the calibration data and future test samples are i.i.d. Then,
    \begin{equation}
        1 - \varepsilon^{\scriptscriptstyle (i)} \sim \mathrm{Beta}(k,\, n + 1 - k), \qquad \text{for } i \in \{1, 2\}.
    \end{equation}
\end{lemma}
\begin{proof}
    The argument is identical for $i \in \{1, 2\}$, so we suppress the superscript. Let $S_1, \dots, S_n$ be the non-conformity scores on the calibration data and $S_{n+1}$ be the score of an independent test point. Let $G$ denote the CDF of the score distribution. The test point is covered if and only if $S_{n+1} \leq \hat{\tau}_\alpha = S_{(k)}$. Hence, the training-conditional coverage is
    \begin{equation}
        1 - \varepsilon = \mathbb{P}\bigl(S_{n+1} \leq \hat{\tau}_\alpha \;;\; \mathcal{D}_{\mathrm{cal}}\bigr) = G\bigl(S_{(k)}\bigr).
    \end{equation}
    If the scores are almost surely distinct, then $G$ is continuous. Thus, $G(S_j) \stackrel{\text{i.i.d.}}{\sim} \mathrm{Unif}(0,1)$, so $G(S_{(k)})$ is the $k$-th order statistic of $n$ uniform samples, which follows $\mathrm{Beta}(k, n+1-k)$.
\end{proof}
Because Lemma~\ref{lemma:coverage_beta} applies for $i \in \{1,2\}$, the two miscoverages share the same mean. Thus,
\begin{equation} \label{eqn:equal_miscoverage_means}
    \mathbb{E}\bigl[\varepsilon^{\scriptscriptstyle (1)}\bigr] = \mathbb{E}\bigl[\varepsilon^{\scriptscriptstyle (2)}\bigr] = 1 - \frac{k}{n+1} = \frac{n + 1 - k}{n + 1} \leq \alpha.
\end{equation}
Since $\mathbb{E}[\varepsilon^{\scriptscriptstyle \Delta}] = 0$, one might hope that the calibration term $\varepsilon^{\scriptscriptstyle \Delta} / \hat{\lambda}^{\scriptscriptstyle (i)}_\alpha$ also vanishes in expectation. Unfortunately, this is not the case. Taking expectations gives
\begin{equation} \label{eqn:covariance_identity}
    \mathbb{E}\!\left[\frac{\varepsilon^{\scriptscriptstyle \Delta}}{\hat\lambda_\alpha^{\scriptscriptstyle (i)}}\right] = \mathrm{Cov}\!\left(\frac{1}{\hat\lambda_\alpha^{\scriptscriptstyle (i)}},\; \varepsilon^{\scriptscriptstyle \Delta}\right),
\end{equation}
which is generally nonzero because the calibrated thresholds are correlated with the miscoverages. We provide two approaches to control $\smash{\varepsilon^{\scriptscriptstyle \Delta} / \hat\lambda_\alpha^{\scriptscriptstyle (i)}}$ without placing assumptions on the \emph{joint} distribution of $(\varepsilon^{\scriptscriptstyle (1)}, \varepsilon^{\scriptscriptstyle (2)})$. First, we analyze its expectation when the reciprocal of the calibrated threshold has finite variance. We then derive a high-probability guarantee that only requires the calibrated threshold to be almost surely positive.
\begin{lemma}[Expectation bound on the calibration term] \label{lemma:expectation_bound}
    Fix a nominal coverage level $1 - \alpha$. Assume that the calibration scores are almost surely distinct and that the calibration data and future test samples are i.i.d. For $i \in \{1, 2\}$, if $\smash{s^{\scriptscriptstyle (i)} = \mathrm{Var}(1 / \hat{\lambda}^{\scriptscriptstyle (i)}_\alpha)^{1/2} < \infty}$, then
    \begin{equation}
        \left|\mathbb{E}\!\left[\frac{\varepsilon^{\scriptscriptstyle \Delta}}{\hat\lambda_\alpha^{\scriptscriptstyle (i)}}\right]\right| \leq \frac{s^{\scriptscriptstyle (i)}}{\sqrt{n}}.
    \end{equation}
\end{lemma}
\begin{proof}
    Fix $i \in \{1, 2\}$. Let $\mathrm{Var}(\varepsilon^{\scriptscriptstyle (i)}) = \sigma^2$. By Lemma~\ref{lemma:coverage_beta},
    \begin{equation}
        \sigma^2 = \frac{(n+1-k)\,k}{(n+1)^2(n+2)} \leq \frac{1}{4(n+2)} \leq \frac{1}{4n}.
    \end{equation}
    Expanding the variance of $\varepsilon^{\scriptscriptstyle \Delta} = \varepsilon^{\scriptscriptstyle (2)} - \varepsilon^{\scriptscriptstyle (1)}$ and applying Cauchy--Schwarz to the cross term,
    \begin{equation} \label{eqn:eps_delta_variance}
        \mathrm{Var}(\varepsilon^{\scriptscriptstyle \Delta}) = 2\sigma^2 - 2\,\mathrm{Cov}\bigl(\varepsilon^{\scriptscriptstyle (1)}, \varepsilon^{\scriptscriptstyle (2)}\bigr) \leq 2\sigma^2 + 2\sqrt{\mathrm{Var}(\varepsilon^{\scriptscriptstyle (1)})\,\mathrm{Var}(\varepsilon^{\scriptscriptstyle (2)})} = 4\sigma^2 \leq 1/n.
    \end{equation}
    Since $\mathbb{E}[\varepsilon^{\scriptscriptstyle \Delta}] = 0$, Equation~\ref{eqn:covariance_identity} and the Cauchy--Schwarz inequality give
    \begin{equation} \label{eqn:expectation_calibration_term_bound}
        \left|\mathbb{E}\!\left[\frac{\varepsilon^{\scriptscriptstyle \Delta}}{\hat\lambda_\alpha^{\scriptscriptstyle (i)}}\right]\right| = \left|\mathrm{Cov}\!\left(\frac{1}{\hat\lambda_\alpha^{\scriptscriptstyle (i)}},\; \varepsilon^{\scriptscriptstyle \Delta}\right)\right| \leq \sqrt{\mathrm{Var} (1 / \hat\lambda_\alpha^{\scriptscriptstyle (i)}) \mathrm{Var}(\varepsilon^{\scriptscriptstyle \Delta})} \leq \frac{s^{\scriptscriptstyle (i)}}{\sqrt{n}}.
    \end{equation}
\end{proof}
Lemma~\ref{lemma:expectation_bound} only requires $\smash{1 / \hat{\lambda}^{\scriptscriptstyle (i)}_\alpha}$ to have finite variance.
Under additional smoothness assumptions (e.g., Lipschitz continuity of the reciprocal quantile function of the scores), the expectation can instead be bounded at a faster $O(1/n)$ rate. Alternatively, we can relax the condition to almost-sure positivity and control the calibration term with high probability.
\begin{lemma}[High-probability bound on the calibration term] \label{lemma:hp_calibration_term}
    Fix a nominal coverage level $1 - \alpha$ and $\delta \in (0,1)$. Assume that the calibration scores are almost surely distinct and that the calibration data and future test samples are i.i.d. Suppose that there exists $\lambda_{\delta/2} > 0$ such that
    \begin{equation}
        \mathbb{P}\bigl(\min\{\hat\lambda_\alpha^{\scriptscriptstyle (1)},\, \hat\lambda_\alpha^{\scriptscriptstyle (2)}\} \geq \lambda_{\delta/2}\bigr) \geq 1 - \tfrac{\delta}{2}.
    \end{equation}
    Then, with probability at least $1 - \delta$ over the draw of $\mathcal{D}_\mathrm{cal}$,
    \begin{equation}
        \left|\frac{\varepsilon^{\scriptscriptstyle \Delta}}{\hat\lambda_\alpha^{\scriptscriptstyle (i)}}\right| \leq \frac{1}{\lambda_{\delta/2}}\sqrt{\frac{2\log(8/\delta)}{n}}, \qquad \text{for both } i \in \{1, 2\}.
    \end{equation}
\end{lemma}
\begin{proof}
    By Lemma~\ref{lemma:coverage_beta}, we have $\varepsilon^{\scriptscriptstyle (i)} \sim \mathrm{Beta}(n+1-k,\, k)$, and let $\mu$ denote their common mean. The $\mathrm{Beta}(\alpha, \beta)$ distribution is sub-Gaussian with proxy variance $\sigma_p^2 = 1/(4(\alpha+\beta+1))$~\citep{marchal2017subgaussianity}. Hence, by the sub-Gaussian tail bound,
    \begin{equation}
        \mathbb{P}\bigl(|\varepsilon^{\scriptscriptstyle (i)} - \mu| \geq t\bigr) \leq 2\exp \bigl(-t^2 / (2\sigma^2_p)\bigr) = 2\exp \bigl(-2(n+2) t^2 \bigr) \leq 2\exp \bigl(-2n t^2\bigr).
    \end{equation}
    Setting the right-hand side equal to $\delta/4$ gives $t_\delta = \sqrt{\log(8/\delta) / (2n)}$. By a union bound over $i \in \{1,2\}$, with probability at least $1 - \frac{\delta}{2}$ we have $|\varepsilon^{\scriptscriptstyle (i)} - \mu| \leq t_\delta$, and therefore
    \begin{equation}
        |\varepsilon^{\scriptscriptstyle \Delta}| = \bigl|(\varepsilon^{\scriptscriptstyle (2)} - \mu) - (\varepsilon^{\scriptscriptstyle (1)} - \mu)\bigr| \leq 2 t_\delta = \sqrt{\frac{2\log(8/\delta)}{n}}.
    \end{equation}
    By assumption, with probability at least $1 - \frac{\delta}{2}$ we have $\hat\lambda_\alpha^{\scriptscriptstyle (i)} \geq \lambda_{\delta/2}$ for both $i$. By a union bound, with probability at least $1 - \delta$,
    \begin{equation}
        \left|\frac{\varepsilon^{\scriptscriptstyle \Delta}}{\hat\lambda_\alpha^{\scriptscriptstyle (i)}}\right| \leq \frac{|\varepsilon^{\scriptscriptstyle \Delta}|}{\lambda_{\delta/2}} \leq \frac{1}{\lambda_{\delta/2}}\sqrt{\frac{2\log(8/\delta)}{n}}. \qedhere
    \end{equation}
\end{proof}

\subsubsection{Corollary~\ref{cor:expectation_sandwich_prob_score}: Probability score sandwich bound in expectation}

The result follows by applying Lemma~\ref{lemma:expectation_bound} to Theorem~\ref{thm:sandwich_prob_score}. For $i \in \{1, 2\}$,
\begin{equation} \label{eqn:sandwich_prob_score_expectation_calibration_term}
    -\frac{s^{\scriptscriptstyle (i)}}{\sqrt{n}} \leq \mathbb{E}\!\left[\frac{\varepsilon^{\scriptscriptstyle \Delta}}{\hat\lambda_\alpha^{\scriptscriptstyle (i)}}\right] \leq \frac{s^{\scriptscriptstyle (i)}}{\sqrt{n}}.
\end{equation}
Taking expectations over the calibration data in Theorem~\ref{thm:sandwich_prob_score} gives
\begin{equation} \label{eqn:sandwich_prob_score_expectation_helper}
    \mathbb{E}[I^{\scriptscriptstyle (2)}] - \mathbb{E}[D^{\scriptscriptstyle (2)}] + \mathbb{E} \left[ \frac{\varepsilon^{\scriptscriptstyle \Delta}}{\hat{\lambda}_\alpha^{\scriptscriptstyle (2)}} \right] \;\leq\; \mathbb{E}[\Delta_\alpha] \;\leq\; \mathbb{E}[I^{\scriptscriptstyle (1)}] + \mathbb{E}[D^{\scriptscriptstyle (1)}] + \mathbb{E} \left[ \frac{\varepsilon^{\scriptscriptstyle \Delta}}{\hat{\lambda}_\alpha^{\scriptscriptstyle (1)}} \right].
\end{equation}
Substituting the appropriate lower and upper bounds from Equation~\ref{eqn:sandwich_prob_score_expectation_calibration_term} into~\ref{eqn:sandwich_prob_score_expectation_helper} completes the proof.

\subsubsection{Corollary~\ref{cor:hp_sandwich_prob_score}: Probability score sandwich bound with high probability}

\begin{corollary}[Probability score sandwich bound with high probability] \label{cor:hp_sandwich_prob_score}
    Fix a nominal coverage level $1 - \alpha$ and $\delta \in (0,1)$. Assume that the calibration scores are almost surely distinct and that the calibration data and future test samples are i.i.d. Suppose that there exists $\lambda_{\delta/2} > 0$ such that
    \begin{equation}
        \mathbb{P}(\min\{\hat\lambda_\alpha^{\scriptscriptstyle (1)}, \hat\lambda_\alpha^{\scriptscriptstyle (2)}\} \geq \lambda_{\delta/2}) \geq 1 - \tfrac{\delta}{2}. 
    \end{equation}
    Then, with probability at least $1 - \delta$ over the draw of $\mathcal{D}_\mathrm{cal}$,
    \begin{equation}
        I^{\scriptscriptstyle (2)} - D^{\scriptscriptstyle (2)} - \sqrt{\frac{2\log(8/\delta)}{n\lambda_{\delta/2}^2}} \;\leq\; \Delta_\alpha \;\leq\; I^{\scriptscriptstyle (1)} + D^{\scriptscriptstyle (1)} + \sqrt{\frac{2\log(8/\delta)}{n\lambda_{\delta/2}^2}}.
    \end{equation}
\end{corollary}
The result follows by applying Lemma~\ref{lemma:hp_calibration_term} to Theorem~\ref{thm:sandwich_prob_score}. By Lemma~\ref{lemma:hp_calibration_term}, with probability at least $1 - \delta$,
\begin{equation} \label{eqn:sandwich_prob_score_hp_calibration_term}
    -\sqrt{\frac{2\log(8/\delta)}{n\lambda_{\delta/2}^2}} \leq \frac{\varepsilon^{\scriptscriptstyle \Delta}}{\hat\lambda_\alpha^{\scriptscriptstyle (i)}} \leq \sqrt{\frac{2\log(8/\delta)}{n\lambda_{\delta/2}^2}}.
\end{equation}
Since Theorem~\ref{thm:sandwich_prob_score} holds for every fixed calibration dataset with $\smash{\hat\lambda_\alpha^{\scriptscriptstyle (i)}} > 0$, on this event
\begin{equation} \label{eqn:sandwich_prob_score_hp_helper}
    I^{\scriptscriptstyle (2)} - D^{\scriptscriptstyle (2)} + \frac{\varepsilon^{\scriptscriptstyle \Delta}}{\hat{\lambda}_\alpha^{\scriptscriptstyle (2)}} \;\leq\; \Delta_\alpha \;\leq\; I^{\scriptscriptstyle (1)} + D^{\scriptscriptstyle (1)} + \frac{\varepsilon^{\scriptscriptstyle \Delta}}{\hat{\lambda}_\alpha^{\scriptscriptstyle (1)}}.
\end{equation}
Substituting the appropriate lower and upper bounds from Equation~\ref{eqn:sandwich_prob_score_hp_calibration_term} into~\ref{eqn:sandwich_prob_score_hp_helper} completes the proof.

\subsubsection{Theorem~\ref{thm:dpi_prob_score}: Probability score dpi}

We convert the data processing inequality for generalized mutual information into a corresponding approximate result for conformal prediction sets for a fixed calibration dataset. Since $\Delta_\alpha$ and $\smash{\tilde{\Delta}_\alpha}$ share the term $\mathbb{E}[|\mathcal{C}_\alpha(X^{\scriptscriptstyle (1)})|]$,
\begin{equation}
    \tilde{\Delta}_\alpha - \Delta_\alpha = \mathbb{E}\bigl[|\mathcal{C}_\alpha(X^{\scriptscriptstyle (2)})|\bigr] - \mathbb{E}\bigl[|\mathcal{C}_\alpha(\tilde{X}^{\scriptscriptstyle (2)})|\bigr].
\end{equation}
Using Proposition~\ref{prop:optimal_predictor_prob_score} and the definition of the conformal information divergence, we can write
\begin{equation} \label{eqn:dpi_prob_score_H2_identity}
    \mathbb{E}\bigl[|\mathcal{C}_\alpha(X^{\scriptscriptstyle (2)})|\bigr] = H_{\hat\lambda_\alpha^{\scriptscriptstyle (2)}}(Y \mid X^{\scriptscriptstyle (2)}) - \frac{\varepsilon^{\scriptscriptstyle (2)}}{\hat\lambda_\alpha^{\scriptscriptstyle (2)}} + D^{\scriptscriptstyle (2)}.
\end{equation}
For the degraded features, $\mathcal{C}_\alpha(\tilde{X}^{\scriptscriptstyle (2)})$ is a feasible decision rule under the $\ell_{\hat\lambda_\alpha^{\scriptscriptstyle (2)}}$ loss, so its risk is no smaller than $\smash{H_{\hat\lambda_\alpha^{\scriptscriptstyle (2)}}(Y \mid \tilde{X}^{\scriptscriptstyle (2)})}$. Evaluating the risk gives
\begin{equation} \label{eqn:dpi_prob_score_H2tilde_bound}
    H_{\hat\lambda_\alpha^{\scriptscriptstyle (2)}}(Y \mid \tilde{X}^{\scriptscriptstyle (2)}) \leq \mathbb{E}\bigl[|\mathcal{C}_\alpha(\tilde{X}^{\scriptscriptstyle (2)})|\bigr] + \frac{\tilde\varepsilon^{\scriptscriptstyle (2)}}{\hat\lambda_\alpha^{\scriptscriptstyle (2)}}.
\end{equation}
Combining Equations~\ref{eqn:dpi_prob_score_H2tilde_bound}~and~\ref{eqn:dpi_prob_score_H2_identity},
\begin{equation} \label{eqn:dpi_prob_score_inequality}
    \tilde{\Delta}_\alpha - \Delta_\alpha \leq H_{\hat\lambda_\alpha^{\scriptscriptstyle (2)}}(Y \mid X^{\scriptscriptstyle (2)}) - H_{\hat\lambda_\alpha^{\scriptscriptstyle (2)}}(Y \mid \tilde{X}^{\scriptscriptstyle (2)}) + D^{\scriptscriptstyle (2)} + \frac{\tilde\varepsilon^{\scriptscriptstyle (2)} - \varepsilon^{\scriptscriptstyle (2)}}{\hat\lambda_\alpha^{\scriptscriptstyle (2)}}.
\end{equation}
By Lemma~\ref{lemma:dpi_generalized_entropy}, $H_{\hat\lambda_\alpha^{\scriptscriptstyle (2)}}(Y \mid X^{\scriptscriptstyle (2)}) \leq H_{\hat\lambda_\alpha^{\scriptscriptstyle (2)}}(Y \mid \tilde{X}^{\scriptscriptstyle (2)})$. Thus,
\begin{equation}
    \tilde{\Delta}_\alpha - \Delta_\alpha \leq D^{\scriptscriptstyle (2)} + \frac{\tilde\varepsilon^{\scriptscriptstyle (2)} - \varepsilon^{\scriptscriptstyle (2)}}{\hat\lambda_\alpha^{\scriptscriptstyle (2)}}.
\end{equation}
This completes the proof.

\subsubsection{Corollary~\ref{cor:expectation_dpi_prob_score}: Probability score dpi in expectation}

Since $\tilde\varepsilon^{\scriptscriptstyle (2)}$ and $\varepsilon^{\scriptscriptstyle (2)}$ are both training-conditional miscoverages at coverage level $1 - \alpha$, they share the same marginal distribution by Lemma~\ref{lemma:coverage_beta}. The result follows by applying Lemma~\ref{lemma:expectation_bound} to Theorem~\ref{thm:dpi_prob_score} with $\varepsilon^{\scriptscriptstyle \Delta} = \tilde\varepsilon^{\scriptscriptstyle (2)} - \varepsilon^{\scriptscriptstyle (2)}$. Thus, we get
\begin{equation} \label{eqn:dpi_prob_score_expectation_calibration_term}
    \mathbb{E}\left[\frac{\varepsilon^{\scriptscriptstyle \Delta}}{\hat\lambda_\alpha^{\scriptscriptstyle (2)}}\right] \leq \frac{s^{\scriptscriptstyle (2)}}{\sqrt{n}}.
\end{equation}
Taking expectations over the calibration data in Theorem~\ref{thm:dpi_prob_score} gives
\begin{equation} \label{eqn:dpi_prob_score_expectation_helper}
    \mathbb{E}[\tilde{\Delta}_\alpha] \;\leq\; \mathbb{E}[\Delta_\alpha] + \mathbb{E}[D^{\scriptscriptstyle (2)}] + \mathbb{E}\!\left[ \frac{\varepsilon^{\scriptscriptstyle \Delta}}{\hat{\lambda}_\alpha^{\scriptscriptstyle (2)}} \right].
\end{equation}
Substituting Equation~\ref{eqn:dpi_prob_score_expectation_calibration_term} into Equation~\ref{eqn:dpi_prob_score_expectation_helper} completes the proof.

\subsubsection{Corollary~\ref{cor:hp_dpi_prob_score}: Probability score dpi with high probability}

\begin{corollary}[Probability score DPI with high probability] \label{cor:hp_dpi_prob_score}
    Fix a nominal coverage level $1 - \alpha$ and $\delta \in (0,1)$. Assume that the calibration scores are almost surely distinct and that the calibration data and future test samples are i.i.d. Suppose that there exists $\lambda_{\delta/2} > 0$ such that
    \begin{equation}
        \mathbb{P}(\hat\lambda_\alpha^{\scriptscriptstyle (2)} \geq \lambda_{\delta/2}) \geq 1 - \tfrac{\delta}{2}.
    \end{equation}
    Then, with probability at least $1 - \delta$ over the draw of $\mathcal{D}_\mathrm{cal}$,
    \begin{equation}
        \tilde{\Delta}_\alpha \;\leq\; \Delta_\alpha + D^{\scriptscriptstyle (2)} + \sqrt{\frac{2\log(8/\delta)}{n\lambda_{\delta/2}^2}}.
    \end{equation}
\end{corollary}
The result follows by applying the proof in Lemma~\ref{lemma:hp_calibration_term} to Theorem~\ref{thm:dpi_prob_score} with $\varepsilon^{\scriptscriptstyle \Delta} = \tilde\varepsilon^{\scriptscriptstyle (2)} - \varepsilon^{\scriptscriptstyle (2)}$. Since $\tilde\varepsilon^{\scriptscriptstyle (2)}$ and $\varepsilon^{\scriptscriptstyle (2)}$ are both training-conditional miscoverages at level $1 - \alpha$, they share the same marginal distribution by Lemma~\ref{lemma:coverage_beta}. Then, Lemma~\ref{lemma:hp_calibration_term} gives, with probability at least $1 - \delta$,
\begin{equation} \label{eqn:dpi_prob_score_hp_calibration_term}
    \frac{\varepsilon^{\scriptscriptstyle \Delta}}{\hat\lambda_\alpha^{\scriptscriptstyle (2)}} \leq \sqrt{\frac{2\log(8/\delta)}{n\lambda_{\delta/2}^2}}.
\end{equation}
Since Theorem~\ref{thm:dpi_prob_score} holds for every fixed calibration dataset with $\smash{\hat\lambda_\alpha^{\scriptscriptstyle (2)}} > 0$, on this event
\begin{equation} \label{eqn:dpi_prob_score_hp_helper}
    \tilde{\Delta}_\alpha \;\leq\; \Delta_\alpha + D^{\scriptscriptstyle (2)} + \frac{\varepsilon^{\scriptscriptstyle \Delta}}{\hat{\lambda}_\alpha^{\scriptscriptstyle (2)}}.
\end{equation}
Substituting the upper bound from Equation~\ref{eqn:dpi_prob_score_hp_calibration_term} into Equation~\ref{eqn:dpi_prob_score_hp_helper} completes the proof.

\subsection{APS score results} \label{app:subsec:aps_score_proofs}

We now prove our results using the APS score from Section~\ref{subsec:aps_construction}. The arguments mirror the proofs in Appendix~\ref{app:subsec:prob_score_proofs}, with the fixed calibrated thresholds replaced by the input-dependent thresholds induced by the APS score.  As in Section~\ref{sec:background}, we assume without loss of generality that there are no ties among the predicted probabilities, since any ties can be broken at random following \citet{romano2020classification}. 

\subsubsection{Regularity of the induced aps threshold}

We first show that $\smash{\hat\tau_\alpha < 1}$ ensures $\hat\Lambda_\alpha \in (0,1]$ as well as boundedness of $\smash{\mathbb{E}[1/\hat\Lambda_\alpha]}$ and $D_{\hat\Lambda_\alpha}$.

\begin{lemma}[APS threshold and divergence bounds] \label{lemma:aps_threshold_lower_bound}
    Fix a calibration dataset $\mathcal{D}_{\mathrm{cal}}$ and nominal coverage level $1 - \alpha$. Let $\smash{\hat\tau_\alpha}$ be the calibrated APS threshold and let $\smash{\hat\Lambda_\alpha = \xi(X)}$ be the induced probability threshold. If $\hat\tau_\alpha < 1$, then for every $x \in \mathcal{X}$,
    \begin{equation} \label{eqn:aps_threshold_lower_bound}
        \xi(x) \geq \frac{1 - \hat\tau_\alpha}{K}.
    \end{equation}
    Consequently, $\hat\Lambda_\alpha \in (0, 1]$ and $\mathbb{E}[1 / \hat\Lambda_\alpha] \leq K / (1 - \hat\tau_\alpha)$. Moreover,
    \begin{equation} \label{eqn:aps_divergence_bound}
        D_{\hat\Lambda_\alpha}\bigl(p(\cdot \mid X) \parallel q(\cdot \mid X)\bigr) \leq K + \frac{K}{1 - \hat\tau_\alpha}.
    \end{equation}
\end{lemma}
\begin{proof}
    Fix $x$ and order the classes as in Section~\ref{sec:background}. Define $P_j = \sum_{\ell \leq j} q(y^{\scriptscriptstyle (\ell)} \mid x)$ with $P_0 = 0$. 
    
    Since $P_K = 1 > \hat\tau_\alpha$, there is a first index $j$ with $P_j > \hat\tau_\alpha$, and $\smash{\xi(x) = q(y^{\scriptscriptstyle (j)} \mid x)}$ by definition. Furthermore, we have $P_{j-1} \leq \hat\tau_\alpha$, so $P_j = P_{j-1} + \xi(x) \leq \hat\tau_\alpha + \xi(x)$. Each of the $K - j$ remaining labels has predicted probability at most $\xi(x)$, so
    \begin{equation}
        1 - \hat\tau_\alpha - \xi(x) \;\leq\; 1 - P_j \;=\; \sum_{\ell > j} q(y^{\scriptscriptstyle (\ell)} \mid x) \;\leq\; (K - j)\, \xi(x).
    \end{equation}
    Rearranging gives $1 - \hat\tau_\alpha \leq (K - j + 1)\, \xi(x) \leq K \xi(x)$, which is Equation~\ref{eqn:aps_threshold_lower_bound}. Taking expectations over $X$ gives the bound on $\mathbb{E}[1 / \hat\Lambda_\alpha]$.

    For the divergence bound, the decision rule $\Gamma^*_{\hat\Lambda_\alpha}(X, U; q)$ has risk at most $K + \mathbb{E}[1 / \hat\Lambda_\alpha]$ when the conditional distribution is $p(\cdot \mid X)$. Equation~\ref{eqn:aps_divergence_bound} then follows from the definition of $D_{\hat\Lambda_\alpha}$.
\end{proof}

\subsubsection{Theorem~\ref{thm:sandwich_aps}: APS score sandwich bound}

We wish to bound the expected conformal set-size reduction
\begin{equation}
    \Delta_{\alpha} = \mathbb{E}\bigl[|\mathcal{C}_\alpha(X^{\scriptscriptstyle (1)})|\bigr] - \mathbb{E}\bigl[|\mathcal{C}_\alpha(X^{\scriptscriptstyle (2)})|\bigr]
\end{equation}
in terms of the conformal mutual information $I_\Lambda$ under a random index, for a fixed realization of the calibration data $\mathcal{D}_{\mathrm{cal}}$. For $i \in \{1, 2\}$, let $\xi^{\scriptscriptstyle (i)}$ be the fixed measurable map sending an input to its induced probability threshold, so that $\smash{\hat{\Lambda}_\alpha^{\scriptscriptstyle (i)} = \xi^{\scriptscriptstyle (i)}(X^{\scriptscriptstyle (i)})}$. 

Since $\mathcal{D}_{\mathrm{cal}}$ is fixed, the conditional miscoverages $\varepsilon^{\scriptscriptstyle (1)}, \varepsilon^{\scriptscriptstyle (2)}$ are deterministic functions of $X^{\scriptscriptstyle (2)}$, where $\varepsilon^{\scriptscriptstyle (i)} = \mathbb{P}(Y \notin \mathcal{C}_\alpha(X^{\scriptscriptstyle (i)}) \mid X^{\scriptscriptstyle (2)}; \mathcal{D}_{\mathrm{cal}})$ with the conditional probability taken over the fresh test label $Y \sim p(\cdot \mid X^{\scriptscriptstyle (2)})$ and the independent randomization $U \sim \mathrm{Unif}[0,1]$.

A key intermediate quantity in our analysis is
\begin{equation}
    \rho \coloneqq H_{\hat{\Lambda}_\alpha^{\scriptscriptstyle (1)}}(Y \mid X^{\scriptscriptstyle (1)}) - H_{\hat{\Lambda}_\alpha^{\scriptscriptstyle (2)}}(Y \mid X^{\scriptscriptstyle (2)}),
\end{equation}
which compares the conformal entropies at the two different calibrated thresholds. We first establish a sandwich bound on $\rho$, then convert it into a bound on $\Delta_{\alpha}$ by making use of the fact that conformal prediction sets using the APS score satisfy the Bayes-optimality criterion in Proposition~\ref{prop:optimal_predictor_aps}. 
\begin{lemma}[APS score sandwich bound on $\rho$] \label{lemma:rho_sandwich_aps}
    Fix a calibration dataset $\mathcal{D}_\mathrm{cal}$ and nominal coverage level $1 - \alpha$. For $i \in \{1, 2\}$, assume that $\hat{\tau}_\alpha^{\scriptscriptstyle (i)} < 1$, and define $\smash{\hat{\Lambda}_\alpha^{\scriptscriptstyle (i)}}$, $I^{\scriptscriptstyle (i)}$, $D^{\scriptscriptstyle (i)}$, and $\varepsilon^{\scriptscriptstyle (i)}$ as in Theorem~\ref{thm:sandwich_aps}. Then,
    \begin{equation}\label{eqn:rho_bound_aps}
        I^{\scriptscriptstyle (2)} - D^{\scriptscriptstyle (1)} + \mathbb{E}\!\left[\varepsilon^{\scriptscriptstyle (1)} \!\left(\frac{1}{\hat{\Lambda}_\alpha^{\scriptscriptstyle (1)}} - \frac{1}{\hat{\Lambda}_\alpha^{\scriptscriptstyle (2)}}\right)\right] \;\leq\; \rho \;\leq\; I^{\scriptscriptstyle (1)} + D^{\scriptscriptstyle (2)} + \mathbb{E}\!\left[\varepsilon^{\scriptscriptstyle (2)} \!\left(\frac{1}{\hat{\Lambda}_\alpha^{\scriptscriptstyle (1)}} - \frac{1}{\hat{\Lambda}_\alpha^{\scriptscriptstyle (2)}}\right)\right],
    \end{equation}
    where the expectations are over $X^{\scriptscriptstyle (2)}$, and both bounds are finite.
\end{lemma}
\begin{proof}
    The proof closely mirrors the argument in Lemma~\ref{lemma:rho_sandwich_prob_score}. Let $p(\cdot \mid X)$ be the true conditional distribution of $Y$ given $X$ and $q(\cdot \mid X)$ be the predicted conditional distribution. By Lemma~\ref{lemma:aps_threshold_lower_bound}, we have $\smash{\hat{\Lambda}_\alpha^{\scriptscriptstyle (i)} \in (0,1]}$ with $\smash{\mathbb{E}[1/\hat{\Lambda}_\alpha^{\scriptscriptstyle (i)}]}$ and $D^{\scriptscriptstyle (i)}$ finite, so we can make use of Proposition~\ref{prop:optimal_predictor_aps} and every term in our analysis is finite.

    \textit{Upper bound:} Adding and subtracting $H_{\hat{\Lambda}_\alpha^{\scriptscriptstyle (1)}}(Y \mid X^{\scriptscriptstyle (2)})$ decomposes $\rho$ as
    \begin{equation} \label{eqn:rho_upper_bound_aps_decomp}
        \rho = I^{\scriptscriptstyle (1)} + H_{\hat{\Lambda}_\alpha^{\scriptscriptstyle (1)}}(Y \mid X^{\scriptscriptstyle (2)}) - H_{\hat{\Lambda}_\alpha^{\scriptscriptstyle (2)}}(Y \mid X^{\scriptscriptstyle (2)}).
    \end{equation}
    It remains to upper bound the cross-threshold conformal entropy difference. By the definition of the conformal information divergence,
    {\small \begin{equation} \label{eqn:rho_upper_bound_aps_lam2_identity}
        H_{\hat{\Lambda}_\alpha^{\scriptscriptstyle (2)}}(Y \mid X^{\scriptscriptstyle (2)}) = \mathbb{E}\bigl[|\Gamma^*_{\hat{\Lambda}_\alpha^{\scriptscriptstyle (2)}}(X^{\scriptscriptstyle (2)}, U; q)|\bigr] + \mathbb{E}\!\left[\frac{1}{\hat{\Lambda}_\alpha^{\scriptscriptstyle (2)}}\, \mathbb{P}\bigl(Y \notin \Gamma^*_{\hat{\Lambda}_\alpha^{\scriptscriptstyle (2)}}(X^{\scriptscriptstyle (2)}, U; q) \mid X^{\scriptscriptstyle (2)}\bigr)\right] - D^{\scriptscriptstyle (2)}.
    \end{equation}
    }Under the $\ell_{\hat{\Lambda}_\alpha^{\scriptscriptstyle (1)}}$ loss, the same predictor has risk no smaller than $\smash{H_{\hat{\Lambda}_\alpha^{\scriptscriptstyle (1)}}(Y \mid X^{\scriptscriptstyle (2)})}$. Thus,
    \begin{equation} \label{eqn:rho_upper_bound_aps_lam1_bound}
        H_{\hat{\Lambda}_\alpha^{\scriptscriptstyle (1)}}(Y \mid X^{\scriptscriptstyle (2)}) \leq \mathbb{E}\bigl[|\Gamma^*_{\hat{\Lambda}_\alpha^{\scriptscriptstyle (2)}}(X^{\scriptscriptstyle (2)}, U; q)|\bigr] + \mathbb{E}\!\left[\frac{1}{\hat{\Lambda}_\alpha^{\scriptscriptstyle (1)}}\, \mathbb{P}\bigl(Y \notin \Gamma^*_{\hat{\Lambda}_\alpha^{\scriptscriptstyle (2)}}(X^{\scriptscriptstyle (2)}, U; q) \mid X^{\scriptscriptstyle (2)}\bigr)\right].
    \end{equation}
    By Proposition~\ref{prop:optimal_predictor_aps}, we can choose $\Gamma^*_{\hat{\Lambda}_\alpha^{\scriptscriptstyle (2)}}(X^{\scriptscriptstyle (2)}, U; q) = \mathcal{C}_\alpha(X^{\scriptscriptstyle (2)})$ since $\smash{\hat{\Lambda}_\alpha^{\scriptscriptstyle (2)} = \xi^{\scriptscriptstyle (2)}(X^{\scriptscriptstyle (2)})}$. 
    
    Hence, the miscoverage terms in Equations~\ref{eqn:rho_upper_bound_aps_lam2_identity}~and~\ref{eqn:rho_upper_bound_aps_lam1_bound} become $\smash{\mathbb{E}[\varepsilon^{\scriptscriptstyle (2)} / \hat{\Lambda}_\alpha^{\scriptscriptstyle (2)}]}$ and $\smash{\mathbb{E}[\varepsilon^{\scriptscriptstyle (2)} / \hat{\Lambda}_\alpha^{\scriptscriptstyle (1)}]}$, respectively. Subtracting Equation~\ref{eqn:rho_upper_bound_aps_lam2_identity} from Equation~\ref{eqn:rho_upper_bound_aps_lam1_bound} gives
    \begin{equation} \label{eqn:rho_upper_bound_aps_crossH_bound}
        H_{\hat{\Lambda}_\alpha^{\scriptscriptstyle (1)}}(Y \mid X^{\scriptscriptstyle (2)}) - H_{\hat{\Lambda}_\alpha^{\scriptscriptstyle (2)}}(Y \mid X^{\scriptscriptstyle (2)}) \leq D^{\scriptscriptstyle (2)} + \mathbb{E}\!\left[\varepsilon^{\scriptscriptstyle (2)} \!\left(\frac{1}{\hat{\Lambda}_\alpha^{\scriptscriptstyle (1)}} - \frac{1}{\hat{\Lambda}_\alpha^{\scriptscriptstyle (2)}}\right)\right].
    \end{equation}
    Substituting Equation~\ref{eqn:rho_upper_bound_aps_crossH_bound} into~\ref{eqn:rho_upper_bound_aps_decomp} establishes the upper bound.

    \textit{Lower bound:} Adding and subtracting $H_{\hat{\Lambda}_\alpha^{\scriptscriptstyle (2)}}(Y \mid X^{\scriptscriptstyle (1)})$ instead decomposes $\rho$ as
    \begin{equation} \label{eqn:rho_lower_bound_aps_decomp}
        \rho = I^{\scriptscriptstyle (2)} + H_{\hat{\Lambda}_\alpha^{\scriptscriptstyle (1)}}(Y \mid X^{\scriptscriptstyle (1)}) - H_{\hat{\Lambda}_\alpha^{\scriptscriptstyle (2)}}(Y \mid X^{\scriptscriptstyle (1)}).
    \end{equation}
    It remains to lower bound the cross-threshold conformal entropy difference. By the definition of the conformal information divergence and the law of iterated expectation,
    {\small \begin{equation} \label{eqn:rho_lower_bound_aps_lam1_identity}
        H_{\hat{\Lambda}_\alpha^{\scriptscriptstyle (1)}}(Y \mid X^{\scriptscriptstyle (1)}) = \mathbb{E}\bigl[|\Gamma^*_{\hat{\Lambda}_\alpha^{\scriptscriptstyle (1)}}(X^{\scriptscriptstyle (1)}, U; q)|\bigr] + \mathbb{E}\!\left[\frac{1}{\hat{\Lambda}_\alpha^{\scriptscriptstyle (1)}}\, \mathbb{P}\bigl(Y \notin \Gamma^*_{\hat{\Lambda}_\alpha^{\scriptscriptstyle (1)}}(X^{\scriptscriptstyle (1)}, U; q) \mid X^{\scriptscriptstyle (2)}\bigr)\right] - D^{\scriptscriptstyle (1)}.
    \end{equation}
    }Under the $\ell_{\hat{\Lambda}_\alpha^{\scriptscriptstyle (2)}}$ loss, the same predictor has risk no smaller than $\smash{H_{\hat{\Lambda}_\alpha^{\scriptscriptstyle (2)}}(Y \mid X^{\scriptscriptstyle (1)})}$. Thus,
    \begin{equation} \label{eqn:rho_lower_bound_aps_lam2_bound}
        H_{\hat{\Lambda}_\alpha^{\scriptscriptstyle (2)}}(Y \mid X^{\scriptscriptstyle (1)}) \leq \mathbb{E}\bigl[|\Gamma^*_{\hat{\Lambda}_\alpha^{\scriptscriptstyle (1)}}(X^{\scriptscriptstyle (1)}, U; q)|\bigr] + \mathbb{E}\!\left[\frac{1}{\hat{\Lambda}_\alpha^{\scriptscriptstyle (2)}}\, \mathbb{P}\bigl(Y \notin \Gamma^*_{\hat{\Lambda}_\alpha^{\scriptscriptstyle (1)}}(X^{\scriptscriptstyle (1)}, U; q) \mid X^{\scriptscriptstyle (2)}\bigr)\right].
    \end{equation}
    By Proposition~\ref{prop:optimal_predictor_aps}, we can choose $\Gamma^*_{\hat{\Lambda}_\alpha^{\scriptscriptstyle (1)}}(X^{\scriptscriptstyle (1)}, U; q) = \mathcal{C}_\alpha(X^{\scriptscriptstyle (1)})$ since $\smash{\hat{\Lambda}_\alpha^{\scriptscriptstyle (1)} = \xi^{\scriptscriptstyle (1)}(X^{\scriptscriptstyle (1)})}$. 
    
    Hence, the miscoverage terms in Equations~\ref{eqn:rho_lower_bound_aps_lam1_identity} and~\ref{eqn:rho_lower_bound_aps_lam2_bound} become $\smash{\mathbb{E}[\varepsilon^{\scriptscriptstyle (1)} / \hat{\Lambda}_\alpha^{\scriptscriptstyle (1)}]}$ and $\smash{\mathbb{E}[\varepsilon^{\scriptscriptstyle (1)} / \hat{\Lambda}_\alpha^{\scriptscriptstyle (2)}]}$, respectively. Subtracting Equation~\ref{eqn:rho_lower_bound_aps_lam2_bound} from Equation~\ref{eqn:rho_lower_bound_aps_lam1_identity} gives
    \begin{equation} \label{eqn:rho_lower_bound_aps_crossH_bound}
        H_{\hat{\Lambda}_\alpha^{\scriptscriptstyle (1)}}(Y \mid X^{\scriptscriptstyle (1)}) - H_{\hat{\Lambda}_\alpha^{\scriptscriptstyle (2)}}(Y \mid X^{\scriptscriptstyle (1)}) \geq \mathbb{E}\!\left[\varepsilon^{\scriptscriptstyle (1)} \!\left(\frac{1}{\hat{\Lambda}_\alpha^{\scriptscriptstyle (1)}} - \frac{1}{\hat{\Lambda}_\alpha^{\scriptscriptstyle (2)}}\right)\right] - D^{\scriptscriptstyle (1)}.
    \end{equation}
    Substituting Equation~\ref{eqn:rho_lower_bound_aps_crossH_bound} into Equation~\ref{eqn:rho_lower_bound_aps_decomp} establishes the lower bound.
\end{proof}
We now convert the two-sided bound on $\rho$ in Lemma~\ref{lemma:rho_sandwich_aps} into the corresponding bound on $\Delta_\alpha$ in Theorem~\ref{thm:sandwich_aps}. Because $\smash{\hat{\Lambda}_\alpha^{\scriptscriptstyle (i)} = \xi^{\scriptscriptstyle (i)}(X^{\scriptscriptstyle (i)})}$ and $\hat{\tau}^{\scriptscriptstyle (i)}_\alpha < 1$, Proposition~\ref{prop:optimal_predictor_aps} lets us write the conformal entropy at $\smash{\hat{\Lambda}_\alpha^{\scriptscriptstyle (1)}}$ and $\smash{\hat{\Lambda}_\alpha^{\scriptscriptstyle (2)}}$ as
\begin{equation} \label{eqn:delta_identity_aps_H1}
    H_{\hat{\Lambda}_\alpha^{\scriptscriptstyle (1)}}(Y \mid X^{\scriptscriptstyle (1)}) = \mathbb{E}\bigl[|\mathcal{C}_\alpha(X^{\scriptscriptstyle (1)})|\bigr] + \mathbb{E}\!\left[\frac{\varepsilon^{\scriptscriptstyle (1)}}{\hat{\Lambda}_\alpha^{\scriptscriptstyle (1)}}\right] - D^{\scriptscriptstyle (1)},
\end{equation}
and
\begin{equation} \label{eqn:delta_identity_aps_H2}
    H_{\hat{\Lambda}_\alpha^{\scriptscriptstyle (2)}}(Y \mid X^{\scriptscriptstyle (2)}) = \mathbb{E}\bigl[|\mathcal{C}_\alpha(X^{\scriptscriptstyle (2)})|\bigr] + \mathbb{E}\!\left[\frac{\varepsilon^{\scriptscriptstyle (2)}}{\hat{\Lambda}_\alpha^{\scriptscriptstyle (2)}}\right] - D^{\scriptscriptstyle (2)},
\end{equation}
where each term is finite by Lemma~\ref{lemma:aps_threshold_lower_bound}. Recall $\smash{\rho = H_{\hat{\Lambda}_\alpha^{\scriptscriptstyle (1)}}(Y \mid X^{\scriptscriptstyle (1)}) - H_{\hat{\Lambda}_\alpha^{\scriptscriptstyle (2)}}(Y \mid X^{\scriptscriptstyle (2)})}$. Subtracting the two identities,
\begin{equation} \label{eqn:delta_identity_aps}
    \Delta_\alpha = \rho - \mathbb{E}\!\left[\frac{\varepsilon^{\scriptscriptstyle (1)}}{\hat{\Lambda}_\alpha^{\scriptscriptstyle (1)}}\right] + \mathbb{E}\!\left[\frac{\varepsilon^{\scriptscriptstyle (2)}}{\hat{\Lambda}_\alpha^{\scriptscriptstyle (2)}}\right] + D^{\scriptscriptstyle (1)} - D^{\scriptscriptstyle (2)}.
\end{equation}
We combine Lemma~\ref{lemma:rho_sandwich_aps} with the identity in Equation~\ref{eqn:delta_identity_aps} to obtain the corresponding bounds on $\Delta_\alpha$.

\textit{Upper bound:} Substituting the upper bound from Equation~\ref{eqn:rho_bound_aps},
\begin{equation}
    \Delta_{\alpha} \leq I^{\scriptscriptstyle (1)} + D^{\scriptscriptstyle (2)} + \mathbb{E}\!\left[\varepsilon^{\scriptscriptstyle (2)} \!\left(\frac{1}{\hat{\Lambda}_\alpha^{\scriptscriptstyle (1)}} - \frac{1}{\hat{\Lambda}_\alpha^{\scriptscriptstyle (2)}}\right)\right] - \mathbb{E}\!\left[\frac{\varepsilon^{\scriptscriptstyle (1)}}{\hat{\Lambda}_\alpha^{\scriptscriptstyle (1)}}\right] + \mathbb{E}\!\left[\frac{\varepsilon^{\scriptscriptstyle (2)}}{\hat{\Lambda}_\alpha^{\scriptscriptstyle (2)}}\right] + D^{\scriptscriptstyle (1)} - D^{\scriptscriptstyle (2)}.
\end{equation}
Therefore,
\begin{equation} \label{eqn:sandwich_upper_bound_aps}
    \Delta_{\alpha} \leq I^{\scriptscriptstyle (1)} + D^{\scriptscriptstyle (1)} + \mathbb{E}\!\left[\frac{\varepsilon^{\scriptscriptstyle (2)} - \varepsilon^{\scriptscriptstyle (1)}}{\hat{\Lambda}_\alpha^{\scriptscriptstyle (1)}}\right].
\end{equation}

\textit{Lower bound:} Substituting the lower bound from Equation~\ref{eqn:rho_bound_aps},
\begin{equation}
    \Delta_{\alpha} \geq I^{\scriptscriptstyle (2)} - D^{\scriptscriptstyle (1)} + \mathbb{E}\!\left[\varepsilon^{\scriptscriptstyle (1)} \!\left(\frac{1}{\hat{\Lambda}_\alpha^{\scriptscriptstyle (1)}} - \frac{1}{\hat{\Lambda}_\alpha^{\scriptscriptstyle (2)}}\right)\right] - \mathbb{E}\!\left[\frac{\varepsilon^{\scriptscriptstyle (1)}}{\hat{\Lambda}_\alpha^{\scriptscriptstyle (1)}}\right] + \mathbb{E}\!\left[\frac{\varepsilon^{\scriptscriptstyle (2)}}{\hat{\Lambda}_\alpha^{\scriptscriptstyle (2)}}\right] + D^{\scriptscriptstyle (1)} - D^{\scriptscriptstyle (2)}.
\end{equation}
Therefore,
\begin{equation} \label{eqn:sandwich_lower_bound_aps}
    \Delta_{\alpha} \geq I^{\scriptscriptstyle (2)} - D^{\scriptscriptstyle (2)} + \mathbb{E}\!\left[\frac{\varepsilon^{\scriptscriptstyle (2)} - \varepsilon^{\scriptscriptstyle (1)}}{\hat{\Lambda}_\alpha^{\scriptscriptstyle (2)}}\right].
\end{equation}
This completes the proof.

\subsubsection{Decomposition of the aps calibration term}

Unlike for the probability score, the calibration terms involve a threshold $\hat\Lambda_\alpha^{\scriptscriptstyle (i)}$ and conditional miscoverage $\varepsilon^{\scriptscriptstyle (i)}$ which vary with the test input $X^{\scriptscriptstyle (2)}$. We next isolate the part of the calibration term that concentrates with $n$ from an input-dependent covariance that does not.
\begin{lemma}[Decomposition of the APS calibration term] \label{lemma:aps_calibration_decomp}
    Fix a calibration dataset $\mathcal{D}_{\mathrm{cal}}$ and nominal coverage level $1 - \alpha$. For $i \in \{1, 2\}$, assume $\hat{\tau}_\alpha^{\scriptscriptstyle (i)} < 1$ and define
    \begin{equation}
        \bar\varepsilon^{\scriptscriptstyle (i)} = \mathbb{P}\bigl(Y \notin \mathcal{C}_\alpha(X^{\scriptscriptstyle (i)}) ; \mathcal{D}_{\mathrm{cal}}\bigr), \qquad M^{\scriptscriptstyle (i)} = \mathbb{E}\!\left[\frac{1}{\hat\Lambda_\alpha^{\scriptscriptstyle (i)}}\right],
    \end{equation}
    and let $\bar\varepsilon^{\scriptscriptstyle \Delta} = \bar\varepsilon^{\scriptscriptstyle (2)} - \bar\varepsilon^{\scriptscriptstyle (1)}$. Then,
    \begin{equation} \label{eqn:aps_calibration_decomp}
        \mathbb{E}_{X^{\scriptscriptstyle (2)}}\!\left[\frac{\varepsilon^{\scriptscriptstyle \Delta}}{\hat\Lambda_\alpha^{\scriptscriptstyle (i)}}\right] = \bar\varepsilon^{\scriptscriptstyle \Delta}\, M^{\scriptscriptstyle (i)} + R^{\scriptscriptstyle (i)}, \qquad \text{where } R^{\scriptscriptstyle (i)} \coloneqq \mathrm{Cov}_{X^{\scriptscriptstyle (2)}}\!\left(\varepsilon^{\scriptscriptstyle \Delta},\, \frac{1}{\hat\Lambda_\alpha^{\scriptscriptstyle (i)}}\right).
    \end{equation}
    Moreover, let $\bar\varepsilon^{\scriptscriptstyle (i)}_Z = \mathbb{P}(Y \notin \mathcal{C}_\alpha(X^{\scriptscriptstyle (i)}) \mid X^{\scriptscriptstyle (1)}; \mathcal{D}_{\mathrm{cal}})$ denote the $X^{\scriptscriptstyle (1)}$-conditional miscoverage. Define
    \begin{equation}
        \kappa^{\scriptscriptstyle (2)} = \sqrt{\mathrm{Var}_{X^{\scriptscriptstyle (2)}}(\varepsilon^{\scriptscriptstyle (2)} - \varepsilon^{\scriptscriptstyle (1)})}, \qquad \bar\kappa^{\scriptscriptstyle (1)} = \sqrt{\mathrm{Var}_{X^{\scriptscriptstyle (1)}}(\bar\varepsilon^{\scriptscriptstyle (2)}_Z - \bar\varepsilon^{\scriptscriptstyle (1)}_Z)}
    \end{equation}
    and
    \begin{equation}
        v^{\scriptscriptstyle (2)} = \sqrt{\mathrm{Var}_{X^{\scriptscriptstyle (2)}}(1/\hat\Lambda_\alpha^{\scriptscriptstyle (2)})}, \qquad \bar v^{\scriptscriptstyle (1)} = \sqrt{\mathrm{Var}_{X^{\scriptscriptstyle (1)}}(1/\hat\Lambda_\alpha^{\scriptscriptstyle (1)})}.
    \end{equation}
    Then, $|R^{\scriptscriptstyle (1)}| \leq \bar\kappa^{\scriptscriptstyle (1)}\, \bar v^{\scriptscriptstyle (1)}$ and $|R^{\scriptscriptstyle (2)}| \leq \kappa^{\scriptscriptstyle (2)}\, v^{\scriptscriptstyle (2)}$.
\end{lemma}
\begin{proof}
    By the law of iterated expectation, $\bar\varepsilon^{\scriptscriptstyle (i)} = \mathbb{E}_{X^{\scriptscriptstyle (2)}}[\varepsilon^{\scriptscriptstyle (i)}]$. Thus, Equation~\ref{eqn:aps_calibration_decomp} follows from the identity $\mathbb{E}[AB] = \mathbb{E}[A]\,\mathbb{E}[B] + \mathrm{Cov}(A, B)$ with $A = \varepsilon^{\scriptscriptstyle \Delta}$ and $B = 1/\hat\Lambda_\alpha^{\scriptscriptstyle (i)}$, together with Lemma~\ref{lemma:aps_threshold_lower_bound}.
    
    By the law of total covariance,
    \begin{equation}
        \mathrm{Cov}_{X^{\scriptscriptstyle (2)}} \!\left( \varepsilon^{\scriptscriptstyle \Delta},\, \frac{1}{\hat\Lambda_\alpha^{\scriptscriptstyle (1)}} \right) = \mathrm{Cov}_{Z}\!\left(\mathbb{E}_X[\varepsilon^{\scriptscriptstyle \Delta} \mid Z], \frac{1}{\hat\Lambda_\alpha^{\scriptscriptstyle (1)}}\right).
    \end{equation}
    By the law of iterated expectation, $\mathbb{E}_X[\varepsilon^{\scriptscriptstyle (i)} \mid Z] = \bar\varepsilon^{\scriptscriptstyle (i)}_Z$. The bounds on $R^{\scriptscriptstyle (i)}$ now follow directly from the Cauchy--Schwarz inequality for covariance, $|\mathrm{Cov}(A, B)| \leq \sqrt{\mathrm{Var}(A)\,\mathrm{Var}(B)}$, together with Lemma~\ref{lemma:aps_threshold_lower_bound} to ensure that $\mathrm{Var}(1 / \hat\Lambda_\alpha^{\scriptscriptstyle (i)})$ is finite.
\end{proof}
The $\kappa$ factors quantify the spread of the coverage gap across inputs, while the $v$ factors measure the spread of the induced APS threshold. These typically will not shrink with $n$. However, under an oracle model, $R^{\scriptscriptstyle (1)}$ vanishes since the randomized APS score achieves exact conditional coverage. We formalize this next.
\begin{lemma}[Oracle conditional coverage] \label{lemma:aps_conditional_coverage}
    Assume the oracle model $q(\cdot \mid X^{\scriptscriptstyle (i)}) = p(\cdot \mid X^{\scriptscriptstyle (i)})$ for $i \in \{1, 2\}$. Then, for each $i \in \{1, 2\}$ and every realization of $\mathcal{D}_{\mathrm{cal}}$,
    \begin{equation} \label{eqn:aps_oracle_conditional_coverage}
        \mathbb{P}\bigl(Y \notin \mathcal{C}_\alpha(X^{\scriptscriptstyle (i)}) \bigm| X^{\scriptscriptstyle (i)} ; \mathcal{D}_{\mathrm{cal}}\bigr) = 1 - \hat\tau_\alpha^{\scriptscriptstyle (i)} \qquad \text{almost surely,}
    \end{equation}
    where the conditional probability is over a fresh label $Y \sim p(\cdot \mid X^{\scriptscriptstyle (i)})$ and $U \sim \mathrm{Unif}[0,1]$.
\end{lemma}
\begin{proof}
    Fix $x$ and order the classes so that $p(y^{\scriptscriptstyle (1)} \mid x) > \cdots > p(y^{\scriptscriptstyle (K)} \mid x)$. Write 
    \begin{equation}
        P_j = \sum_{\ell \leq j} p(y^{\scriptscriptstyle (\ell)} \mid x),
    \end{equation}
    with $P_0 = 0$. If $Y = y^{\scriptscriptstyle (j)}$, the randomized APS score is $s(x, Y) = P_{j-1} + U\, p(y^{\scriptscriptstyle (j)} \mid x)$, which is uniform on $(P_{j-1}, P_j]$. Since $\mathbb{P}(Y = y^{\scriptscriptstyle (j)} \mid X = x) = P_j - P_{j-1}$, mixing over all $j$ gives $s(x, Y) \mid X = x \sim \mathrm{Unif}(0, 1]$. Hence,
    \begin{equation}
        \mathbb{P}\bigl(Y \notin \mathcal{C}_\alpha(x) \mid X = x\bigr) = \mathbb{P}\bigl(s(x, Y) > \hat\tau_\alpha^{\scriptscriptstyle (i)} \mid X = x\bigr) = 1 - \hat\tau_\alpha^{\scriptscriptstyle (i)}.
    \end{equation}
    Thus, Equation~\ref{eqn:aps_oracle_conditional_coverage} holds almost surely. 
\end{proof}
Using Lemma~\ref{lemma:aps_conditional_coverage} with $X^{\scriptscriptstyle (1)}$, we get $\bar\varepsilon^{\scriptscriptstyle (1)}_Z = 1 - \hat\tau_\alpha^{\scriptscriptstyle (1)}$. With $X^{\scriptscriptstyle (2)}$ and the law of total probability,
\begin{equation}
    \bar\varepsilon^{\scriptscriptstyle (2)}_Z = \mathbb{E}\bigl[\varepsilon^{\scriptscriptstyle (2)}\mid X^{\scriptscriptstyle (1)}\bigr] = \mathbb{E}\bigl[1-\hat\tau^{\scriptscriptstyle (2)}_\alpha \mid X^{\scriptscriptstyle (1)}\bigr] = 1-\hat\tau^{\scriptscriptstyle (2)}_\alpha.
\end{equation}
Given a fixed calibration dataset, $\bar\varepsilon^{\scriptscriptstyle (2)}_Z - \bar\varepsilon^{\scriptscriptstyle (1)}_Z = \hat\tau_\alpha^{\scriptscriptstyle (1)} - \hat\tau_\alpha^{\scriptscriptstyle (2)}$ has zero variance and $\bar\kappa^{\scriptscriptstyle (1)} = 0$. Thus, $R^{\scriptscriptstyle (1)} = 0$ under an oracle model. The same argument does not apply to $R^{\scriptscriptstyle (2)}$, since $\varepsilon^{\scriptscriptstyle (1)}$ is conditioned on $X^{\scriptscriptstyle (2)}$, whereas the corresponding predictor only observes $X^{\scriptscriptstyle (1)}$.

\subsubsection{Corollary~\ref{cor:expectation_sandwich_aps}: APS score sandwich bound in expectation}

\begin{corollary}[APS score sandwich bound in expectation] \label{cor:expectation_sandwich_aps}
    Fix a nominal coverage level $1 - \alpha$. Assume that the calibration scores are almost surely distinct and that the calibration data and future test samples are i.i.d. For $i \in \{1, 2\}$, define $I^{\scriptscriptstyle (i)}$ and $D^{\scriptscriptstyle (i)}$ as in Theorem~\ref{thm:sandwich_aps} and $\kappa^{\scriptscriptstyle (2)}$, $\bar\kappa^{\scriptscriptstyle (1)}$, $v^{\scriptscriptstyle (2)}$, $\bar v^{\scriptscriptstyle (1)}$, and $M^{\scriptscriptstyle (i)}$ as in Lemma~\ref{lemma:aps_calibration_decomp}. Furthermore, suppose that $\smash{\mathbb{E}[(1 - \hat\tau_\alpha^{\scriptscriptstyle (i)})^{-2}] < \infty}$ and let $\smash{m^{\scriptscriptstyle (i)} = \mathrm{Var}(M^{\scriptscriptstyle (i)})^{1/2}}$. Then,
    \begin{equation} \label{eqn:expectation_sandwich_aps}
        \mathbb{E}[I^{\scriptscriptstyle (2)} - D^{\scriptscriptstyle (2)} - \kappa^{\scriptscriptstyle (2)}\, v^{\scriptscriptstyle (2)}] - \frac{m^{\scriptscriptstyle (2)}}{\sqrt{n}} \;\leq\; \mathbb{E}[\Delta_\alpha] \;\leq\; \mathbb{E}[I^{\scriptscriptstyle (1)} + D^{\scriptscriptstyle (1)} + \bar\kappa^{\scriptscriptstyle (1)}\, \bar v^{\scriptscriptstyle (1)}] + \frac{m^{\scriptscriptstyle (1)}}{\sqrt{n}},
    \end{equation}
    where all expectations are over the calibration dataset $\mathcal{D}_{\mathrm{cal}}$, and both bounds are finite.
\end{corollary}
\begin{proof}
    By Lemma~\ref{lemma:aps_calibration_decomp},
    \begin{equation} \label{eqn:aps_expectation_decomp}
        \mathbb{E}_{X^{\scriptscriptstyle (2)}}\!\left[\frac{\varepsilon^{\scriptscriptstyle \Delta}}{\hat\Lambda_\alpha^{\scriptscriptstyle (i)}}\right] = \bar\varepsilon^{\scriptscriptstyle \Delta}\, M^{\scriptscriptstyle (i)} + R^{\scriptscriptstyle (i)}.
    \end{equation}
    By Lemma~\ref{lemma:coverage_beta}, we have $1 - \bar\varepsilon^{\scriptscriptstyle (i)} \sim \mathrm{Beta}(k, n + 1 - k)$, where $k = \lceil (1 - \alpha)(n + 1) \rceil$. Hence, $\mathbb{E}[\bar\varepsilon^{\scriptscriptstyle \Delta}] = 0$ and, following the proof of Lemma~\ref{lemma:expectation_bound}, we have $\mathrm{Var}(\bar\varepsilon^{\scriptscriptstyle \Delta}) \leq 1/n$. Lemma~\ref{lemma:aps_threshold_lower_bound} also gives that $M^{\scriptscriptstyle (i)} \leq K / (1 - \hat\tau_\alpha^{\scriptscriptstyle (i)})$, so $m^{\scriptscriptstyle (i)}$ is finite. By the Cauchy--Schwarz inequality,
    \begin{equation} \label{eqn:aps_sandwich_eps_delta_bound}
        \bigl|\mathbb{E}[\bar\varepsilon^{\scriptscriptstyle \Delta}\, M^{\scriptscriptstyle (i)}]\bigr| = \bigl|\mathrm{Cov}(\bar\varepsilon^{\scriptscriptstyle \Delta}, M^{\scriptscriptstyle (i)})\bigr| \leq \frac{m^{\scriptscriptstyle (i)}}{\sqrt{n}}.
    \end{equation}
    We also have $\mathbb{E}[R^{\scriptscriptstyle (2)}] \geq -\mathbb{E}[\kappa^{\scriptscriptstyle (2)}\, v^{\scriptscriptstyle (2)}]$ and $\mathbb{E}[R^{\scriptscriptstyle (1)}] \leq \mathbb{E}[\bar\kappa^{\scriptscriptstyle (1)}\, \bar v^{\scriptscriptstyle (1)}]$ by Lemma~\ref{lemma:aps_calibration_decomp}. Taking expectations over $\mathcal{D}_{\mathrm{cal}}$ in Theorem~\ref{thm:sandwich_aps} and then substituting Equation~\ref{eqn:aps_expectation_decomp} together with the bounds on $\mathbb{E}[R^{\scriptscriptstyle (i)}]$ gives Equation~\ref{eqn:expectation_sandwich_aps}.

    It remains to show that all other terms are finite. Since $\smash{\mathbb{E}[(1 - \hat\tau_\alpha^{\scriptscriptstyle (i)})^{-2}] < \infty}$, we have $\hat\tau_\alpha^{\scriptscriptstyle (i)} < 1$ almost surely, so Theorem~\ref{thm:sandwich_aps} applies. Lemma~\ref{lemma:aps_threshold_lower_bound} then gives 
    \begin{equation*}
        \frac{1}{\hat\Lambda_\alpha^{\scriptscriptstyle (i)}} \leq \frac{K}{1 - \hat\tau_\alpha^{\scriptscriptstyle (i)}}, \qquad D^{\scriptscriptstyle (i)} \leq \frac{2K}{1 - \hat\tau_\alpha^{\scriptscriptstyle (i)}},
    \end{equation*}
    where the right-hand sides have finite expectation since $\mathbb E[(1 - \hat\tau_\alpha^{\scriptscriptstyle (i)})^{-2}] < \infty$. Since $v^{\scriptscriptstyle (2)}$ and $\bar v^{\scriptscriptstyle (1)}$ are standard deviations of $\smash{1 / \hat\Lambda_\alpha^{\scriptscriptstyle (i)}}$ over the test input, they are also bounded by $K / (1 - \hat\tau_\alpha^{\scriptscriptstyle (i)})$. Moreover, \mbox{$\kappa^{\scriptscriptstyle (2)}, \bar\kappa^{\scriptscriptstyle (1)} \leq 1$}. Thus, both bounds are finite.
\end{proof}

\subsubsection{Corollary~\ref{cor:hp_sandwich_aps}: APS score sandwich bound with high probability}

\begin{corollary}[APS score sandwich bound with high probability] \label{cor:hp_sandwich_aps}
    Fix a nominal coverage level $1 - \alpha$ and $\delta \in (0, 1)$. Assume that the calibration scores are almost surely distinct and that the calibration data and future test samples are i.i.d. Suppose that there exists a constant $\eta_{\delta/2} < 1$ such that
    \begin{equation}
        \mathbb{P} \bigl( \max \bigl\{ \hat\tau_\alpha^{\scriptscriptstyle (1)}, \, \hat\tau_\alpha^{\scriptscriptstyle (2)} \bigr\} \leq \eta_{\delta/2} \bigr) \geq 1 - \tfrac{\delta}{2}.
    \end{equation}
    Let $c_{\delta/2} = \frac{K}{1 - \eta_{\delta/2}}$. Then, with probability at least $1 - \delta$ over the draw of $\mathcal{D}_{\mathrm{cal}}$,
    {\small \begin{equation}
        I^{\scriptscriptstyle (2)} - D^{\scriptscriptstyle (2)} - \sqrt{\frac{2 c_{\delta/2}^2 \log(8/\delta)}{n}} - \kappa^{\scriptscriptstyle (2)} \, v^{\scriptscriptstyle (2)} \;\leq\; \Delta_\alpha \;\leq\; I^{\scriptscriptstyle (1)} + D^{\scriptscriptstyle (1)} + \sqrt{\frac{2 c_{\delta/2}^2 \log(8/\delta)}{n}} + \bar\kappa^{\scriptscriptstyle (1)}\, \bar v^{\scriptscriptstyle (1)},
    \end{equation}}
    where $I^{\scriptscriptstyle (i)}$ and $D^{\scriptscriptstyle (i)}$ are defined as in Theorem~\ref{thm:sandwich_aps} and $\bar\kappa^{\scriptscriptstyle (1)}$, $\kappa^{\scriptscriptstyle (2)}$, $\bar v^{\scriptscriptstyle (1)}$, and $v^{\scriptscriptstyle (2)}$ are defined as in Lemma~\ref{lemma:aps_calibration_decomp}.
\end{corollary}
\begin{proof}
    By Lemma~\ref{lemma:aps_calibration_decomp}, if $\smash{\hat\tau_\alpha^{\scriptscriptstyle (i)} < 1}$, then
    \begin{equation}
        \mathbb{E}_{X^{\scriptscriptstyle (2)}}\!\left[\frac{\varepsilon^{\scriptscriptstyle \Delta}}{\hat\Lambda_\alpha^{\scriptscriptstyle (i)}}\right] = \bar\varepsilon^{\scriptscriptstyle \Delta}\, M^{\scriptscriptstyle (i)} + R^{\scriptscriptstyle (i)}.
    \end{equation}
    By Lemma~\ref{lemma:coverage_beta}, $\bar\varepsilon^{\scriptscriptstyle (i)} \sim \mathrm{Beta}(n + 1 - k, k)$. Following the proof of Lemma~\ref{lemma:hp_calibration_term}, each $\bar\varepsilon^{\scriptscriptstyle (i)}$ is sub-Gaussian with proxy variance $1/(4(n+2))$, so with probability at least $1 - \frac{\delta}{2}$,
    \begin{equation}
        |\bar\varepsilon^{\scriptscriptstyle \Delta}| \leq \sqrt{\frac{2 \log(8/\delta)}{n}}.
    \end{equation}
    By assumption, with probability at least $1 - \frac{\delta}{2}$ we have $\smash{\hat\tau_\alpha^{\scriptscriptstyle (i)} \leq \eta_{\delta/2}}$ for both $i$, so Lemma~\ref{lemma:aps_threshold_lower_bound} gives $M^{\scriptscriptstyle (i)} \leq c_{\delta/2}$. By a union bound, with probability at least $1 - \delta$,
    \begin{equation} \label{eqn:hp_aps_sandwich_eps_delta_bound}
        |\bar\varepsilon^{\scriptscriptstyle \Delta}\, M^{\scriptscriptstyle (i)}| \leq c_{\delta/2} \sqrt{\frac{2 \log(8/\delta)}{n}} = \sqrt{\frac{2 c_{\delta/2}^2 \log(8/\delta)}{n}}.
    \end{equation}
    Moreover, $R^{\scriptscriptstyle (2)} \geq -\kappa^{\scriptscriptstyle (2)} \, v^{\scriptscriptstyle (2)}$ and $R^{\scriptscriptstyle (1)} \leq \bar\kappa^{\scriptscriptstyle (1)}\, \bar v^{\scriptscriptstyle (1)}$ by Lemma~\ref{lemma:aps_calibration_decomp}. Since Theorem~\ref{thm:sandwich_aps} holds for every fixed calibration dataset with $\smash{\hat\tau_\alpha^{\scriptscriptstyle (i)} < 1}$, substituting these bounds completes the proof.
\end{proof}

\subsubsection{Theorem~\ref{thm:dpi_aps}: APS score dpi}

We extend the data processing inequality to the APS score. As in Section~\ref{subsec:dpi_prob_score}, suppose that for each data point we also observe a \emph{single realization} of a degraded feature $\smash{\tilde{X}}$ obtained from a (possibly noisy) channel, so that $\smash{Y \to (X, Z) \to (\tilde{X}, Z)}$ forms a Markov chain. Let $\smash{\tilde{X}^{\scriptscriptstyle (2)} = (\tilde{X}, Z)}$ and define the set-size reduction with the degraded features as
\begin{equation}
    \smash{\tilde{\Delta}_\alpha = \mathbb{E}[|\mathcal{C}_\alpha(X^{\scriptscriptstyle (1)})|] - \mathbb{E}[|\mathcal{C}_\alpha(\tilde{X}^{\scriptscriptstyle (2)})|]}.
\end{equation}
\begin{theorem}[APS score DPI] \label{thm:dpi_aps}
    Fix a calibration dataset $\mathcal{D}_{\mathrm{cal}}$ and nominal coverage level $1 - \alpha$. Define the conditional miscoverage $\smash{\tilde{\varepsilon}^{\scriptscriptstyle (2)} = \mathbb{P}(Y \notin \mathcal{C}_\alpha(\tilde{X}^{\scriptscriptstyle (2)}) \mid X^{\scriptscriptstyle (2)}; \mathcal{D}_{\mathrm{cal}})}$, where the probability is taken jointly over a noisy realization of $\smash{\tilde{X}}$ and fresh test label $Y$ given $X^{\scriptscriptstyle (2)}$ and the APS randomization $U$. Moreover, let $\smash{\hat{\Lambda}_\alpha^{\scriptscriptstyle (2)}}$, $D^{\scriptscriptstyle(2)}$, and $\varepsilon^{\scriptscriptstyle (2)}$ be defined as in Theorem~\ref{thm:sandwich_aps}, and assume that $\smash{\hat{\tau}_\alpha^{\scriptscriptstyle (2)} < 1}$. Then,
    \begin{equation}
        \tilde{\Delta}_\alpha \leq \Delta_\alpha + D^{\scriptscriptstyle(2)} + \mathbb{E}\bigl[(\tilde{\varepsilon}^{\scriptscriptstyle (2)} - \varepsilon^{\scriptscriptstyle (2)}) / \hat{\Lambda}_\alpha^{\scriptscriptstyle (2)}\bigr],
    \end{equation}
    where the expectation is over $X^{\scriptscriptstyle (2)}$, and the right-hand side is finite.
\end{theorem}
\begin{proof}
    We convert the data processing inequality for generalized mutual information into a corresponding approximate result for conformal prediction sets for a fixed calibration dataset. Since $\Delta_\alpha$ and $\smash{\tilde{\Delta}_\alpha}$ share the term $\mathbb{E}[|\mathcal{C}_\alpha(X^{\scriptscriptstyle (1)})|]$,
    \begin{equation}
        \tilde{\Delta}_\alpha - \Delta_\alpha = \mathbb{E}\bigl[|\mathcal{C}_\alpha(X^{\scriptscriptstyle (2)})|\bigr] - \mathbb{E}\bigl[|\mathcal{C}_\alpha(\tilde{X}^{\scriptscriptstyle (2)})|\bigr].
    \end{equation}
    Because $\smash{\hat{\Lambda}_\alpha^{\scriptscriptstyle (2)} = \xi^{\scriptscriptstyle (2)}(X^{\scriptscriptstyle (2)})}$ is a function of $X^{\scriptscriptstyle (2)}$, using Proposition~\ref{prop:optimal_predictor_aps} gives
    \begin{equation} \label{eqn:dpi_aps_H2_identity}
        \mathbb{E}\bigl[|\mathcal{C}_\alpha(X^{\scriptscriptstyle (2)})|\bigr] = H_{\hat{\Lambda}_\alpha^{\scriptscriptstyle (2)}}(Y \mid X^{\scriptscriptstyle (2)}) - \mathbb{E}\!\left[\frac{\varepsilon^{\scriptscriptstyle (2)}}{\hat{\Lambda}_\alpha^{\scriptscriptstyle (2)}}\right] + D^{\scriptscriptstyle (2)},
    \end{equation}
    where each term is finite by Lemma~\ref{lemma:aps_threshold_lower_bound}. For the degraded features, $\mathcal{C}_\alpha(\tilde{X}^{\scriptscriptstyle (2)})$ is a feasible decision rule under the $\ell_{\hat{\Lambda}_\alpha^{\scriptscriptstyle (2)}}$ loss, so its risk is no smaller than $\smash{H_{\hat{\Lambda}_\alpha^{\scriptscriptstyle (2)}}(Y \mid \tilde{X}^{\scriptscriptstyle (2)})}$. Evaluating the risk and using the law of iterated expectation gives
    \begin{equation} \label{eqn:dpi_aps_H2tilde_bound}
        H_{\hat{\Lambda}_\alpha^{\scriptscriptstyle (2)}}(Y \mid \tilde{X}^{\scriptscriptstyle (2)}) \leq \mathbb{E}\bigl[|\mathcal{C}_\alpha(\tilde{X}^{\scriptscriptstyle (2)})|\bigr] + \mathbb{E}\!\left[\frac{\tilde{\varepsilon}^{\scriptscriptstyle (2)}}{\hat{\Lambda}_\alpha^{\scriptscriptstyle (2)}}\right].
    \end{equation}
    Combining Equations~\ref{eqn:dpi_aps_H2tilde_bound}~and~\ref{eqn:dpi_aps_H2_identity},
    \begin{equation} \label{eqn:dpi_aps_inequality}
        \tilde{\Delta}_\alpha - \Delta_\alpha \leq H_{\hat{\Lambda}_\alpha^{\scriptscriptstyle (2)}}(Y \mid X^{\scriptscriptstyle (2)}) - H_{\hat{\Lambda}_\alpha^{\scriptscriptstyle (2)}}(Y \mid \tilde{X}^{\scriptscriptstyle (2)}) + D^{\scriptscriptstyle (2)} + \mathbb{E}\!\left[\frac{\tilde{\varepsilon}^{\scriptscriptstyle (2)} - \varepsilon^{\scriptscriptstyle (2)}}{\hat{\Lambda}_\alpha^{\scriptscriptstyle (2)}}\right].
    \end{equation}
    Since $\smash{\hat{\Lambda}_\alpha^{\scriptscriptstyle (2)} = \xi^{\scriptscriptstyle (2)}(X^{\scriptscriptstyle (2)})}$ is a function of $X^{\scriptscriptstyle (2)}$, the loss index is fixed given $X^{\scriptscriptstyle (2)}$ and the proof of Lemma~\ref{lemma:dpi_generalized_entropy} applies to $\ell_{\hat{\Lambda}_\alpha^{\scriptscriptstyle (2)}}$, so $\smash{H_{\hat{\Lambda}_\alpha^{\scriptscriptstyle (2)}}(Y \mid X^{\scriptscriptstyle (2)}) \leq H_{\hat{\Lambda}_\alpha^{\scriptscriptstyle (2)}}(Y \mid \tilde{X}^{\scriptscriptstyle (2)})}$. Thus,
    \begin{equation}
        \tilde{\Delta}_\alpha - \Delta_\alpha \leq D^{\scriptscriptstyle (2)} + \mathbb{E}\!\left[\frac{\tilde{\varepsilon}^{\scriptscriptstyle (2)} - \varepsilon^{\scriptscriptstyle (2)}}{\hat{\Lambda}_\alpha^{\scriptscriptstyle (2)}}\right].
    \end{equation}
    Finiteness of each term on the right-hand side follows from Lemma~\ref{lemma:aps_threshold_lower_bound}.
\end{proof}

As with the sandwich bound in Theorem~\ref{thm:sandwich_aps}, we provide versions of the DPI either in expectation or with high probability over the calibration data.

\subsubsection{Corollary~\ref{cor:expectation_dpi_aps}: APS score dpi in expectation}

\begin{corollary}[APS score DPI in expectation] \label{cor:expectation_dpi_aps}
    Fix a nominal coverage level $1 - \alpha$. Assume that the calibration scores are almost surely distinct and that the calibration data and future test samples are i.i.d. Define $D^{\scriptscriptstyle (2)}$ as in Theorem~\ref{thm:sandwich_aps} and $v^{\scriptscriptstyle (2)}$ as in Lemma~\ref{lemma:aps_calibration_decomp}, and let $\tilde\kappa^{\scriptscriptstyle (2)} = \sqrt{\mathrm{Var}_{X^{\scriptscriptstyle (2)}}(\tilde\varepsilon^{\scriptscriptstyle (2)} - \varepsilon^{\scriptscriptstyle (2)})}$. Moreover, suppose $\smash{\hat\tau_\alpha^{\scriptscriptstyle (2)}}$ satisfies the conditions of Corollary~\ref{cor:expectation_sandwich_aps}, with $m^{\scriptscriptstyle (2)}$ defined accordingly. Then,
    \begin{equation}
        \mathbb{E}[\tilde{\Delta}_\alpha] \;\leq\; \mathbb{E}[\Delta_\alpha] + \mathbb{E}[D^{\scriptscriptstyle (2)}] + \frac{m^{\scriptscriptstyle (2)}}{\sqrt{n}} + \mathbb{E}\bigl[\tilde\kappa^{\scriptscriptstyle (2)}\, v^{\scriptscriptstyle (2)}\bigr],
    \end{equation}
    where all expectations are over the calibration dataset $\mathcal{D}_{\mathrm{cal}}$, and the right-hand side is finite.
\end{corollary}
\begin{proof}
    The proof follows the same argument as Corollary~\ref{cor:expectation_sandwich_aps}. Define $\smash{\bar{\tilde\varepsilon}^{\scriptscriptstyle (2)} \coloneqq \mathbb{P}(Y \notin \mathcal{C}_\alpha(\tilde{X}^{\scriptscriptstyle (2)}); \mathcal{D}_{\mathrm{cal}})}$ and $\smash{\bar\varepsilon_{\mathrm{DPI}}^{\scriptscriptstyle \Delta} \coloneqq \bar{\tilde\varepsilon}^{\scriptscriptstyle (2)} - \bar\varepsilon^{\scriptscriptstyle (2)}}$. By Lemma~\ref{lemma:aps_calibration_decomp} with $\smash{\bar\varepsilon_{\mathrm{DPI}}^{\scriptscriptstyle \Delta}}$,
    \begin{equation} \label{eqn:aps_dpi_expectation_decomp}
        \mathbb{E}_{X^{\scriptscriptstyle (2)}}\!\left[\frac{\tilde\varepsilon^{\scriptscriptstyle (2)} - \varepsilon^{\scriptscriptstyle (2)}}{\hat\Lambda_\alpha^{\scriptscriptstyle (2)}}\right] = \bar\varepsilon_{\mathrm{DPI}}^{\scriptscriptstyle \Delta}\, M^{\scriptscriptstyle (2)} + R_{\mathrm{DPI}}^{\scriptscriptstyle (2)}, \qquad |R_{\mathrm{DPI}}^{\scriptscriptstyle (2)}| \leq \tilde\kappa^{\scriptscriptstyle (2)}\, v^{\scriptscriptstyle (2)}.
    \end{equation}
    Since $\smash{\bar{\tilde\varepsilon}^{\scriptscriptstyle (2)}}$ and $\smash{\bar\varepsilon^{\scriptscriptstyle (2)}}$ are both training-conditional miscoverages at coverage level $1 - \alpha$, Lemma~\ref{lemma:coverage_beta} gives that they share the same marginal distribution. Following the argument for Equation~\ref{eqn:aps_sandwich_eps_delta_bound} with $\smash{\bar\varepsilon_{\mathrm{DPI}}^{\scriptscriptstyle \Delta}}$ in place of $\smash{\bar\varepsilon^{\scriptscriptstyle \Delta}}$,
    \begin{equation} \label{eqn:aps_dpi_eps_delta_bound}
        \bigl|\mathbb{E}[\bar\varepsilon_{\mathrm{DPI}}^{\scriptscriptstyle \Delta}\, M^{\scriptscriptstyle (2)}]\bigr| \leq \frac{m^{\scriptscriptstyle (2)}}{\sqrt{n}}.
    \end{equation}
    Taking expectations over $\mathcal{D}_{\mathrm{cal}}$ in Theorem~\ref{thm:dpi_aps} and substituting Equations~\ref{eqn:aps_dpi_expectation_decomp}~and~\ref{eqn:aps_dpi_eps_delta_bound} completes the proof. All terms on the right-hand side are finite using the same argument as in Corollary~\ref{cor:expectation_sandwich_aps}.
\end{proof}

\subsubsection{Corollary~\ref{cor:hp_dpi_aps}: APS score dpi with high probability}

\begin{corollary}[APS score DPI with high probability] \label{cor:hp_dpi_aps}
    Fix a nominal coverage level $1 - \alpha$ and \mbox{$\delta \in (0, 1)$}. Assume that the calibration scores are almost surely distinct and that the calibration data and future test samples are i.i.d. Suppose that there exists a constant $\eta_{\delta/2} < 1$ such that
    \begin{equation}
        \mathbb{P} \bigl( \hat\tau_\alpha^{\scriptscriptstyle (2)} \leq \eta_{\delta/2} \bigr) \geq 1 - \tfrac{\delta}{2}.
    \end{equation}
    Let $c_{\delta/2} = \frac{K}{1 - \eta_{\delta/2}}$. Then, with probability at least $1 - \delta$ over the draw of $\mathcal{D}_{\mathrm{cal}}$,
    \begin{equation}
        \tilde{\Delta}_\alpha \;\leq\; \Delta_\alpha + D^{\scriptscriptstyle (2)} + \sqrt{\frac{2 c_{\delta/2}^2 \log(8/\delta)}{n}} + \tilde\kappa^{\scriptscriptstyle (2)}\, v^{\scriptscriptstyle (2)},
    \end{equation}
    where $D^{\scriptscriptstyle (2)}$ is defined as in Theorem~\ref{thm:sandwich_aps} and $\tilde\kappa^{\scriptscriptstyle (2)}$ and $v^{\scriptscriptstyle (2)}$ are defined as in Corollary~\ref{cor:expectation_dpi_aps}.
\end{corollary}
\begin{proof}
    The proof follows the same argument as Corollary~\ref{cor:hp_sandwich_aps}. Define $\smash{\bar{\tilde\varepsilon}^{\scriptscriptstyle (2)} \coloneqq \mathbb{P}(Y \notin \mathcal{C}_\alpha(\tilde{X}^{\scriptscriptstyle (2)}); \mathcal{D}_{\mathrm{cal}})}$ and $\smash{\bar\varepsilon_{\mathrm{DPI}}^{\scriptscriptstyle \Delta} \coloneqq \bar{\tilde\varepsilon}^{\scriptscriptstyle (2)} - \bar\varepsilon^{\scriptscriptstyle (2)}}$ as in Corollary~\ref{cor:expectation_dpi_aps}. By Lemma~\ref{lemma:aps_calibration_decomp}, if $\smash{\hat\tau_\alpha^{\scriptscriptstyle (2)} < 1}$, then
    \begin{equation}
        \mathbb{E}_{X^{\scriptscriptstyle (2)}}\!\left[\frac{\tilde\varepsilon^{\scriptscriptstyle (2)} - \varepsilon^{\scriptscriptstyle (2)}}{\hat\Lambda_\alpha^{\scriptscriptstyle (2)}}\right] = \bar\varepsilon_{\mathrm{DPI}}^{\scriptscriptstyle \Delta}\, M^{\scriptscriptstyle (2)} + R_{\mathrm{DPI}}^{\scriptscriptstyle (2)}, \qquad |R_{\mathrm{DPI}}^{\scriptscriptstyle (2)}| \leq \tilde\kappa^{\scriptscriptstyle (2)}\, v^{\scriptscriptstyle (2)}.
    \end{equation}
    Since $\smash{\bar{\tilde\varepsilon}^{\scriptscriptstyle (2)}}$ and $\smash{\bar\varepsilon^{\scriptscriptstyle (2)}}$ are both training-conditional miscoverages at coverage level $1 - \alpha$, Lemma~\ref{lemma:coverage_beta} gives that they share the same marginal distribution. Following the proof of Lemma~\ref{lemma:hp_calibration_term}, with probability at least $1 - \frac{\delta}{2}$,
    \begin{equation}
        |\bar\varepsilon_{\mathrm{DPI}}^{\scriptscriptstyle \Delta}| \leq \sqrt{\frac{2 \log(8/\delta)}{n}}.
    \end{equation}
    By assumption, with probability at least $1 - \frac{\delta}{2}$ we have $\smash{\hat\tau_\alpha^{\scriptscriptstyle (2)} \leq \eta_{\delta/2}}$, so Lemma~\ref{lemma:aps_threshold_lower_bound} gives $M^{\scriptscriptstyle (2)} \leq c_{\delta/2}$. By a union bound, with probability at least $1 - \delta$,
    \begin{equation} \label{eqn:hp_aps_dpi_eps_delta_bound}
        |\bar\varepsilon_{\mathrm{DPI}}^{\scriptscriptstyle \Delta}\, M^{\scriptscriptstyle (2)}| \leq c_{\delta/2} \sqrt{\frac{2 \log(8/\delta)}{n}} = \sqrt{\frac{2 c_{\delta/2}^2 \log(8/\delta)}{n}}.
    \end{equation}
    Since Theorem~\ref{thm:dpi_aps} holds for every fixed calibration dataset with $\smash{\hat\tau_\alpha^{\scriptscriptstyle (2)} < 1}$, substituting Equation~\ref{eqn:hp_aps_dpi_eps_delta_bound} together with $R_{\mathrm{DPI}}^{\scriptscriptstyle (2)} \leq \tilde\kappa^{\scriptscriptstyle (2)}\, v^{\scriptscriptstyle (2)}$ completes the proof.
\end{proof}

\clearpage

\section{Datasets} \label{app:sec:datasets}

We provide additional details on the synthetic and real datasets used in our experiments. The dataset statistics, such as sample sizes and feature dimensions, are summarized in Table~\ref{app:tab:dataset_statistics}.

\subsection{Dataset details}

We provide a brief overview of the datasets. Throughout, we will use $\oplus$ as a concatenation operator.

\paragraph{S1.} The first synthetic dataset samples the features $Z \oplus X \sim \mathcal{N}(0, I_d)$, where $d = 4$. The first two coordinates are the base features $Z$ and the remaining two coordinates are the additional features $X$. The label $Y$ is sampled as
\begin{equation}
    Y \mid X, Z \sim \mathrm{Categorical}\bigl(\mathrm{softmax}(W_2 \, \mathrm{ReLU}(W_1 (Z \oplus X) + b_1) + b_2)\bigr),
\end{equation}
with $K = 10$ classes and hidden width $h = 16$. 

The weights $W_1 \in \mathbb{R}^{h \times d}$, $W_2 \in \mathbb{R}^{K \times h}$ and biases $b_1 \in \mathbb{R}^h$, $b_2 \in \mathbb{R}^K$ are drawn once, with i.i.d.\ entries $W_1 \sim \mathcal{N}(0, 1/d)$, $W_2 \sim \mathcal{N}(0, 1/h)$, and $b_1, b_2 \sim \mathcal{N}(0, 0.01)$. Columns of $W_1$ corresponding to $X$ are scaled by a factor of $4$. We compute the marginal oracle models $p(\cdot \mid Z)$ by Gauss--Hermite quadrature, while $p(\cdot \mid X, Z)$ is available in closed form.

\paragraph{S2.} The second synthetic dataset samples a label $Y \sim \mathrm{Unif}\{0, \ldots, K-1\}$ with $K = 50$ classes and features $Z \oplus X \mid Y = k \sim \mathcal{N}(\mu_k, I_d)$, where $d = 20$. The first $10$ coordinates are the base features $Z$ and the remaining $10$ coordinates are the additional features $X$. 

Each class mean $\mu_k = \mu_k^Z \oplus \mu_k^X$ is drawn once by sampling a vector of i.i.d.\ standard normal entries in each block, then rescaling it to a fixed norm where $\|\mu_k^Z\| \approx \sqrt{2}$ and $\|\mu_k^X\| \approx 2$. Since $Y$ is uniform and the features are Gaussian, the oracle posteriors $p(\cdot \mid Z)$ and $p(\cdot \mid X, Z)$ are available in closed form by Bayes' rule.

\paragraph{S3.} The third synthetic dataset is designed for the feature selection experiment. We use $K = 40$ classes separated into a small set of $8$ head classes containing $70\%$ of the prior mass and a large set of $32$ tail classes containing the remaining $30\%$ of the prior mass. The base feature is 
\begin{equation}
    Z \mid Y = k \sim \mathcal{N}(c_k + \delta_k, \, I_{d_Z}),
\end{equation}
with $d_Z = 8$, where $c_k = c_{\mathrm{head}}$ for head classes and $c_k = c_{\mathrm{tail}}$ for tail classes. Each center $c_k$ has norm $2$ and $\delta_k$ is a smaller per-class offset of norm $0.5$.

A pool of $m = 50$ scalar candidate features $X_1, \ldots, X_m$ is observed independently given $Y$, 
\begin{equation}
    X_j \mid Y = k \sim \mathcal{N}(a_j \, \mathrm{dir}_j[k], \, 1),
\end{equation}
where $a_j \in \mathbb{R}$ is a per-candidate amplitude and $\mathrm{dir}_j \in \mathbb{R}^K$ is a per-candidate direction vector over the classes. Candidates are split across three roles, which set $a_j$ and $\mathrm{dir}_j$ as follows:
\begin{itemize}[leftmargin=*, topsep=2pt, itemsep=1pt]
    \item $18$ \emph{head} detectors: the $8$ head classes are randomly partitioned into two groups of $4$. The direction $\mathrm{dir}_j$ is zero for all non-head classes and takes values $\smash{\pm 1/\sqrt{8}}$ on the two head groups, with opposite signs for each group, so that $\mathrm{dir}_j$ has unit norm. The amplitude is $a_j = 2$;
    
    \item $18$ \emph{tail} detectors: the same construction, splitting the $32$ tail classes into two groups of $16$;
    
    \item $14$ \emph{filler} features: either $\mathrm{dir}_j$ is a random unit vector with $a_j = 0.4$ or $\mathrm{dir}_j = 0$.
\end{itemize}
The tail detector amplitude is calibrated by bisection so that the average Shannon information $I(X_j; Y \mid Z)$ across tail detectors matches that across head detectors. Since the base and candidate features are conditionally independent given $Y$, the posterior $p(y \mid Z, X_S)$ is available in closed form via Bayes' rule for any subset $S \subseteq \{1,\ldots,m\}$. 

Note that S3 is relatively complex because it is designed to exhibit clear cases where conformal set-size reduction and Shannon mutual information disagree in feature selection. We observe similar behavior on several real-world datasets, albeit to a lesser extent.

An adapted version of S3 is used for Figure~\ref{fig:crossing_profile}, with $K = 24$ classes split into head, tail, and filler groups and $X_1, X_2 \in \mathbb{R}^8$. The class-conditional mean of $X_1$ is a distinct random vector for each of the $8$ tail classes and zero otherwise, and $X_2$ is defined analogously for the $4$ head classes.

\paragraph{Wine Quality (Wine).} The Wine Quality dataset~\citep{cortez2009wine} contains physicochemical measurements and quality ratings for red and white wines. We use the white wine subset, which contains 4,898 samples with 11 physicochemical features. The prediction target is the wine quality rating ranging from 3 to 9.

\paragraph{Human Activity Recognition (HAR).} The Human Activity Recognition dataset~\citep{reyesortiz2013human} contains smartphone sensor measurements collected from participants performing six activities. We use the preprocessed feature representation, consisting of 10,299 samples with 561 features. The prediction target is one of six activities: walking, walking upstairs, walking downstairs, sitting, standing, or laying.

\paragraph{Letter Recognition (Letter).} The Letter Recognition dataset~\citep{slate1991letter} contains 20,000 samples of capital English letters, each represented by 16 numerical features extracted from the corresponding character image. The prediction target is one of the 26 letters of the English alphabet.

\paragraph{Covertype.} The Covertype dataset~\citep{blackard1998covertype} contains 581,012 samples describing forested areas using 54 cartographic features, comprising 10 quantitative features and 44 binary features encoding wilderness area and soil type information. The prediction target is one of seven forest cover types.

\paragraph{Fashion-MNIST (FMNIST).} Fashion-MNIST~\citep{xiao2017fashion} consists of 70,000 grayscale \mbox{28 $\times$ 28} images of fashion items from 10 classes. The classes are T-shirt/top, trouser, pullover, dress, coat, sandal, shirt, sneaker, bag, and ankle boot. The dataset is split into 60,000 training samples and 10,000 test samples.

\paragraph{CIFAR-10.} CIFAR-10~\citep{krizhevsky2009learning} consists of 60,000 RGB $32 \times 32$ images from 10 object classes: airplane, automobile, bird, cat, deer, dog, frog, horse, ship, and truck. The dataset contains 50,000 training samples and 10,000 test samples.

\paragraph{CIFAR-100.} CIFAR-100~\citep{krizhevsky2009learning} consists of 60,000 RGB $32 \times 32$ images from 100 object classes. The dataset contains 50,000 training samples and 10,000 test samples.

\paragraph{ImageNet-1k.} ImageNet-1k~\citep{deng2009imagenet} is a large-scale image classification dataset containing images from 1,000 object classes. We use the 64 $\times$ 64 downsampled version of the dataset~\citep{chrabaszcz2017downsampled} containing 1,281,167 training samples and 50,000 test samples.

\begin{table}[t]
    \centering
    \caption{Sample sizes, feature dimensions, and number of classes for all datasets. For the synthetic datasets S1–S3, the conditional distributions are known and no training data are required. Here, $n_{\mathrm{eval}}$ denotes the number of samples available for calibration and evaluation. For tabular datasets, the input features are partitioned into $X$ and $Z$. For image datasets, $X$ and $Z$ correspond to separate image representations with the indicated channel dimensions. The number of candidate features in the feature selection experiments is denoted by $m$.}
    \label{app:tab:dataset_statistics}
    \resizebox{\textwidth}{!}{%
    \begin{tabular}{lccccccc}
    \toprule
    Dataset & $n_{\mathrm{total}}$ & $n_{\mathrm{eval}}$ & $d_{\mathrm{total}}$ & $d_X$ & $d_Z$ & $m$ & $K$ \\
    \midrule
    S1 & 200{,}000 & 200{,}000 & 4 & 2 & 2 & --- & 10 \\
    S2 & 200{,}000 & 200{,}000 & 20 & 10 & 10 & 10 & 50 \\
    S3 & 200{,}000 & 200{,}000 & 58 & 50 & 8 & 50 & 40 \\
    \midrule
    Wine      & 4{,}898   & 1{,}224   & 11  & 6   & 5   & 11 & 7 \\
    HAR       & 10{,}299  & 2{,}575   & 561 & 348 & 213 & 51 & 6 \\
    Letter    & 20{,}000  & 5{,}000   & 16  & 9   & 7   & 16 & 26 \\
    Covertype & 581{,}012 & 145{,}253 & 54  & 10  & 44  & 54 & 7 \\
    \midrule
    FMNIST     & 70{,}000    & 10{,}000  & 2 $\times$ 28 $\times$ 28 & 1 $\times$ 28 $\times$ 28 & 1 $\times$ 28 $\times$ 28 & 64 & 10 \\
    CIFAR-10   & 60{,}000    & 10{,}000  & 6 $\times$ 32 $\times$ 32 & 3 $\times$ 32 $\times$ 32 & 3 $\times$ 32 $\times$ 32 & 64 & 10 \\
    CIFAR-100  & 60{,}000    & 10{,}000  & 6 $\times$ 32 $\times$ 32 & 3 $\times$ 32 $\times$ 32 & 3 $\times$ 32 $\times$ 32 & 64 & 100 \\
    ImageNet-1k & 1{,}331{,}167 & 50{,}000 & 6 $\times$ 64 $\times$ 64 & 3 $\times$ 64 $\times$ 64 & 3 $\times$ 64 $\times$ 64 & 64 & 1000 \\
    \bottomrule
    \end{tabular}%
    }
\end{table}

\begin{figure}[t]
    \centering
    \includegraphics[width=0.94\textwidth]{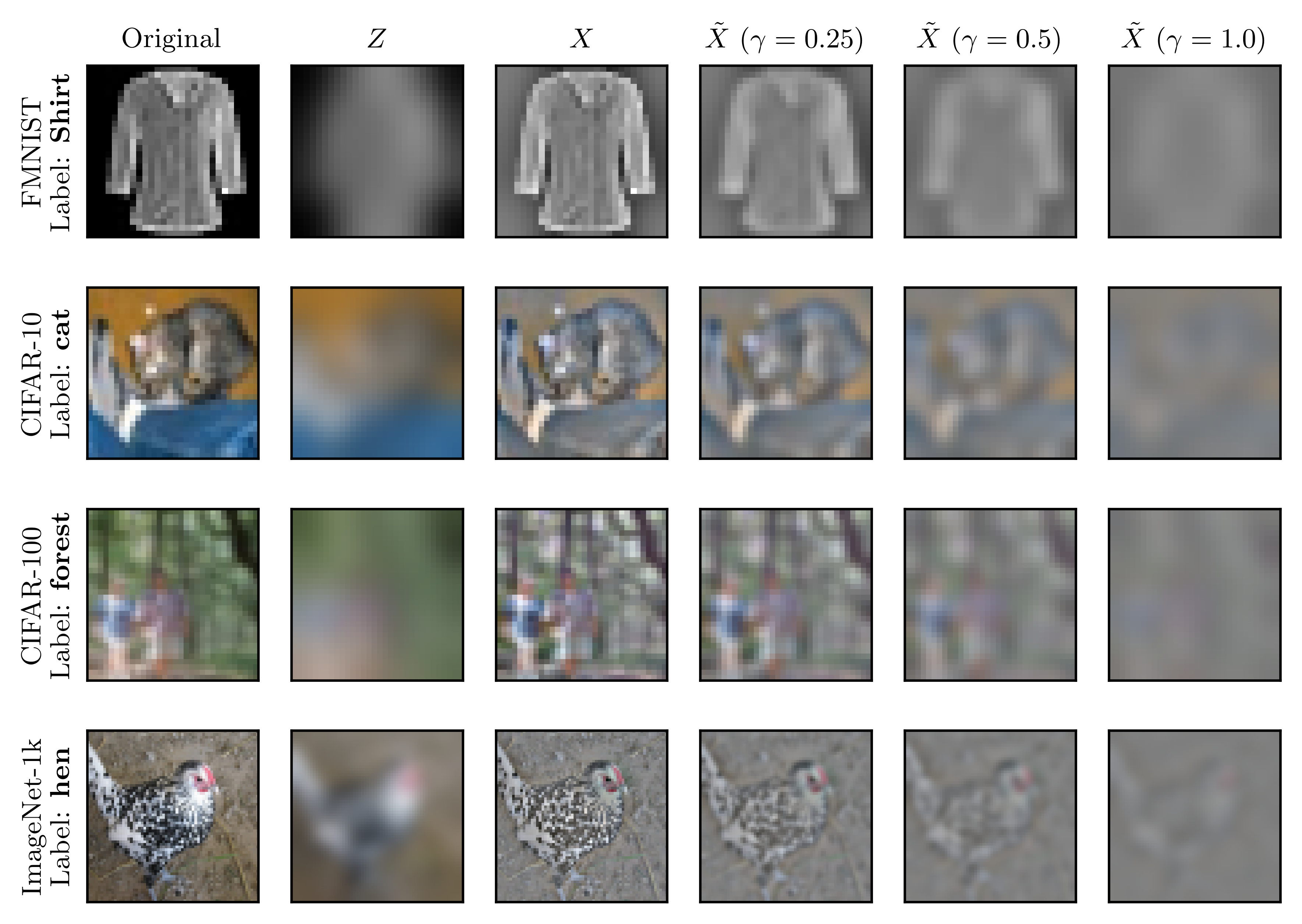}
    \caption{Examples of the original image, base, new, and degraded features for each image dataset. The degraded features apply an additional Gaussian blur with standard deviation $\gamma\sigma$ to $X$.}
    \label{app:fig:viz}
\end{figure}

\subsection{Feature construction details}

We provide additional details on the construction of the base features $Z$ and additional features $X$ for each dataset. For the synthetic datasets, $Z$ consists of the first $d_Z$ feature coordinates and $X$ consists of the remaining $d_X$ coordinates.

For the tabular datasets, we partition the input features into $X$ and $Z$ based on natural groupings of the feature types, as detailed below.

\paragraph{Wine Quality (Wine).} The base features $Z$ consist of five physical properties: chlorides, density, residual sugar, alcohol content, and pH. The additional features $X$ consist of the remaining six acidity- and sulfur-related measurements.

\paragraph{Human Activity Recognition (HAR).} The base features $Z$ consist of the 213 gyroscope-derived measurements, including time- and frequency-domain statistics computed from the gyroscope signals. The additional features $X$ consist of the remaining 348 features, including accelerometer-derived measurements and the three angle-between-vectors features.

\paragraph{Letter Recognition (Letter).} The base features $Z$ consist of seven statistical moments of the $x$- and $y$-pixel positions. The additional features $X$ consist of the remaining nine features: five bounding-box features describing position, width, height, and pixel count, and four edge-count and edge-correlation features.

\paragraph{Covertype.} The base features $Z$ consist of the 44 binary features encoding wilderness area and soil type. The additional features $X$ consist of the 10 quantitative cartographic variables, including elevation, slope, and distances to hydrology.

For all image datasets, we construct $Z$ by applying a Gaussian blur with standard deviation $\sigma = 3$ to the original image, with the Gaussian kernel truncated at four standard deviations on each side. The additional features $X$ are given by the residual between the original and blurred images. An example of the features for each image dataset is shown in Figure~\ref{app:fig:viz}.

\subsection{Degraded feature construction details}

We provide additional details on the construction of the degraded feature $\tilde{X}$ used in the data processing inequality experiments (Section~\ref{subsec:dpi_prob_score}). We apply degradation only to the additional features $X$, leaving the base features $Z$ unchanged. We consider three degradation levels, \emph{low}, \emph{medium}, and \emph{high}, corresponding to degradation multipliers $\gamma \in \{0.25, 0.5, 1.0\}$, respectively. We report results for all three degradation levels in Appendix~\ref{app:sec:experiments}. In the main text, we report results only for the medium degradation level, corresponding to $\gamma = 0.5$.

For the synthetic and tabular datasets, we degrade $X$ by adding independent Gaussian noise to each coordinate. For a degradation level $\gamma$, the noise standard deviation for each coordinate is $\gamma$ times its empirical standard deviation, estimated using the training data. The degradation noise is sampled once for each dataset and held fixed across all noise levels and subsequent experimental replicates in the analysis.

For the image datasets, we degrade $X$ by applying an additional Gaussian blur to the residual features, while leaving the already-blurred base features $Z$ unchanged. For a degradation level $\gamma$, we use a blur standard deviation of $\gamma \sigma$, where $\sigma=3$ is the blur scale used to construct $Z$. An example of the degraded features for each image dataset is shown in Figure~\ref{app:fig:viz}.

The synthetic datasets (S1 and S2) have known oracle distributions, so we compute the degraded posterior in closed form or using Gauss--Hermite quadrature. For the real-world datasets, we train a separate model on $\smash{(\tilde{X}, Z)}$ for each degradation level.

\subsection{Candidate feature details}

We describe the candidate features $\{X_1, \ldots, X_m\}$ used in the feature selection experiments. In the synthetic settings, the base features $Z$ are always observed, and each coordinate of the additional features $X$ forms a separate candidate feature. This yields $10$ and $50$ candidate features for S2 and S3, respectively, for which we can compute oracle conditional probabilities given any subset of features. We exclude S1 from the feature selection experiments due to the low dimensionality of $X$.

For the tabular datasets, no features are observed initially and each input feature (including those previously in $Z$) is treated as a separate candidate. The only exception is HAR, which contains 561 input features. To keep the number of candidates manageable, we randomly partition the features into 51 groups of 11, with each group forming a candidate feature.

The image candidates are constructed from the final convolutional representation of the frozen backbone trained on $(X, Z)$. As with the tabular datasets, no features are observed initially. Since the purpose of the feature selection experiment is to empirically probe the differences between set-size reduction and Shannon information, this representation serves to provide a collection of meaningful candidates, rather than physically acquirable features.

On FMNIST, CIFAR-10, and CIFAR-100, we partition the channels of the final convolutional feature map into 64 contiguous groups. Each group contains 2 channels of size $3 \times 3$ on FMNIST and 4 channels of size $4 \times 4$ on CIFAR-10 and CIFAR-100. The ResNet backbone on ImageNet-1k instead produces 512 globally average-pooled channels, which we partition into 64 groups of 8.

The number of candidate features $m$ for each dataset is summarized in Table~\ref{app:tab:dataset_statistics}.

\clearpage

\section{Training details} \label{app:sec:training}

We provide details on the model architectures, training procedures, hyperparameter selection, and final pre-calibration accuracy of our classification models. The synthetic datasets use oracle probabilities and therefore do not require model training.

\begin{table}[t]
    \centering
    \renewcommand{\arraystretch}{1.3}
    \caption{Selected hyperparameters and training configurations for the classification models. The learning rate, weight decay, and dropout rate are selected via grid search. The maximum number of epochs and patience correspond to the full training procedure after hyperparameter selection.}
    \label{app:tab:training_hparams}
    \begin{tabular}{llllccc}
        \toprule
        Dataset & Params. & Max. epochs & Patience & Learning rate & Weight decay & Dropout \\
        \midrule
        Wine        & 600   & 1,000 & 100 & $4\times10^{-4}$ & $1\times10^{-4}$ & $0.3$ \\
        HAR         & 178K  & 1,000 & 100 & $4\times10^{-4}$ & $1\times10^{-4}$ & $0.1$ \\
        Letter      & 106K  & 1,000 & 100 & $1\times10^{-3}$ & $1\times10^{-4}$ & $0.2$ \\
        Covertype   & 424K  & 500   & 100 & $4\times10^{-4}$ & $1\times10^{-5}$ & $0.1$ \\
        FMNIST      & 105K  & 1,000 & 100 & $1\times10^{-3}$ & $4\times10^{-5}$ & $0.3$ \\
        CIFAR-10    & 414K  & 500   & 100 & $1\times10^{-3}$ & $1\times10^{-4}$ & $0.1$ \\
        CIFAR-100   & 783K  & 500   & 100 & $1\times10^{-3}$ & $1\times10^{-4}$ & $0.3$ \\
        ImageNet-1k & 11.7M & 25    & 5   & $4\times10^{-4}$ & $4\times10^{-5}$ & $0.1$ \\
        \bottomrule
    \end{tabular}
\end{table}

\subsection{Model architectures}

For the tabular datasets, we use multi-layer perceptrons (MLPs) with dropout after each hidden layer. The hidden-layer widths are $(16, 16)$ for Wine, $(256, 128)$ for HAR, $(256, 256, 128)$ for Letter, and $(512, 512, 256)$ for Covertype. All non-binary features are independently standardized to zero mean and unit variance using statistics computed from the training data.

For FMNIST, CIFAR-10, and CIFAR-100, we use three-layer CNNs with $3 \times 3$ convolutional kernels, batch normalization, and ReLU activations. The convolutional layers have channel widths $(32, 64, 128)$ for FMNIST and $(64, 128, 256)$ for CIFAR-10 and CIFAR-100. Each convolutional layer is followed by $2 \times 2$ max pooling. After the final layer, we flatten the spatial representation, apply dropout, and use a linear classification head to produce the predicted probabilities.

For ImageNet-1k, we use the ResNet-18 architecture. We directly fine-tune the pretrained network on the base features $Z$ using standard ImageNet normalization and modify the input layer to accommodate the additional channels when using both $X$ and $Z$. For all ResNet models, the final convolutional representation is passed through global adaptive pooling and dropout, followed by a linear classification head.

\subsection{Training optimization} \label{app:training:optimization}

We train all models on the real datasets using the Adam optimizer with weight decay to minimize the standard cross-entropy loss. All newly initialized model parameters use the default PyTorch initialization. For the image datasets, we additionally apply data augmentation consisting of horizontal flips with probability $0.5$ and random crops after zero-padding the images by 2 pixels for FMNIST and 4 pixels for CIFAR-10, CIFAR-100, and ImageNet-1k. We also use cosine annealing to decay the learning rate from its initial value to zero over the maximum number of training epochs.

Training uses a batch size of 256 and early stopping based on validation loss, with the checkpoint attaining the lowest validation loss restored at the end of training. Separate models are trained for the base, degraded, and augmented feature sets. We set the maximum number of training epochs to 1,000 for Wine, HAR, Letter, and FMNIST; 500 for CIFAR-10, CIFAR-100, and Covertype; and 25 for ImageNet-1k. We use an early-stopping patience of 100 epochs for all datasets except ImageNet-1k, for which we use a patience of 5 epochs.

\subsection{Mask-conditioned model training}

The feature selection experiments require predicted conditional probabilities given arbitrary subsets of the candidate features. We therefore modify the model architectures to accommodate arbitrary conditioning masks. For the tabular datasets, we modify the MLP to take inputs of the form \mbox{$[x \odot M, M]$}, where $x$ is the full feature vector and $M$ is a binary mask indicating which features are observed, following the masked predictor construction in \citet{yoon2019invase}.

For the image datasets, we use the pretrained CNN or ResNet model previously trained on the full set of features $(X,Z)$. We freeze the backbone through the final convolutional layer, fixing the representation from which the candidate features are constructed. For FMNIST, CIFAR-10, and CIFAR-100, we train a masked MLP on this representation using the same masking construction as for the tabular datasets. We instead train a linear classification head on the masked representation on ImageNet-1k for scalability.

During training, we independently sample a conditioning mask for each example. Specifically, we first sample $u \sim \mathrm{Unif}[0,1]$ and then independently mask each candidate feature with probability $u$. To ensure sufficient exposure to the two extreme cases, we additionally mask all candidate features with probability $0.05$ and observe all candidate features with probability $0.15$. 

We double both the maximum number of training epochs and the patience relative to the full-feature training procedure described in Appendix~\ref{app:training:optimization}. For early stopping, we average the validation loss across masks generated with fixed masking probabilities $u \in \{0, 0.25, 0.5, 0.75, 1\}$.

\subsection{Hyperparameter selection}

We select the learning rate, weight decay, and dropout rate via grid search. We consider learning rates in $(10^{-3}, 4 \times 10^{-4}, 10^{-4})$, weight decay values in $(10^{-4}, 4 \times 10^{-5}, 10^{-5})$, and dropout rates in $(0.1, 0.2, 0.3)$, yielding 27 configurations in total. Hyperparameters are selected using the augmented model with both $X$ and $Z$. The selected configuration is then shared across the augmented model, the base model using only $Z$, and the models using degraded features $\smash{\tilde{X}}$ and $Z$.

For each configuration, we train the augmented model using 20\% of both the maximum epoch budget and early-stopping patience specified above. The image models use the same cosine annealing schedule during hyperparameter selection. For ImageNet-1k, each configuration is trained for the reduced epoch budget without early stopping.

\begin{table}[t]
    \centering
    \renewcommand{\arraystretch}{1.2}
    \caption{Top-1 classification accuracy (\%) on the held-out evaluation pool for the augmented, degraded, and base feature sets, and for the mask-conditioned model with all candidate features observed. For the synthetic datasets S1 and S2, we report the accuracy of the oracle model.}
    \label{app:tab:acc}
    \begin{tabular}{lcccccc}
        \toprule
        Dataset 
        & $(X,Z)$ 
        & \makecell{$(\tilde{X},Z)$ \\ low noise}
        & \makecell{$(\tilde{X},Z)$ \\ medium noise}
        & \makecell{$(\tilde{X},Z)$ \\ high noise}
        & $Z$
        & \makecell{Mask-conditioned \\ (all)} \\
        \midrule
        S1          & 56.52 & 54.86 & 51.83 & 46.69 & 39.33 & --- \\
        S2          & 50.83 & 48.84 & 44.00 & 34.21 & 18.43 & --- \\
        Wine        & 55.15 & 54.98 & 53.43 & 52.78 & 50.33 & 55.80 \\
        HAR         & 98.29 & 97.40 & 96.54 & 93.86 & 88.08 & 98.72 \\
        Letter      & 97.68 & 95.52 & 92.00 & 86.76 & 86.52 & 97.36 \\
        Covertype   & 96.24 & 83.11 & 74.93 & 68.85 & 65.16 & 94.12 \\
        FMNIST      & 93.06 & 92.77 & 91.43 & 89.92 & 89.85 & 93.11 \\
        CIFAR-10    & 86.40 & 84.10 & 77.55 & 70.31 & 68.80 & 86.30 \\
        CIFAR-100   & 60.97 & 56.77 & 50.14 & 42.94 & 40.36 & 56.79 \\
        ImageNet-1k & 58.79 & 57.84 & 54.57 & 47.06 & 43.70 & 57.99 \\
        \bottomrule
    \end{tabular}
\end{table}

Among the 27 configurations, we select the one achieving the highest validation accuracy after early stopping. We then train all models using this configuration with their respective epoch budget and early-stopping patience. The selected hyperparameters for each dataset are summarized in Table~\ref{app:tab:training_hparams}.

\begin{figure}[t]
    \centering
    \includegraphics[width=\textwidth]{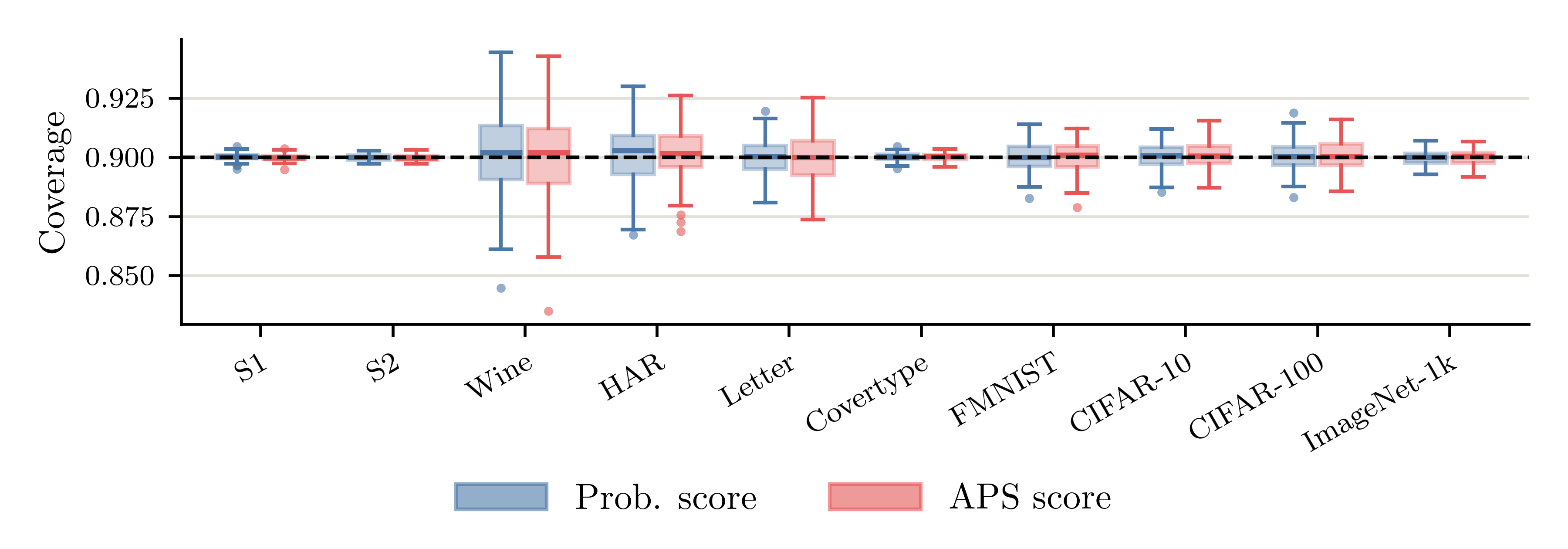}
    \caption{Empirical marginal coverage of the conformal predictors using $(X,Z)$ over 100 random equal calibration and test splits of the evaluation pool. The dashed line indicates the nominal coverage level $1 - \alpha = 0.9$.}
    \label{app:fig:coverage}
\end{figure}

\subsection{Accuracy and coverage metrics}

We report the top-1 classification accuracy of all models on the held-out evaluation pool. We evaluate models using the augmented features $(X,Z)$, the degraded features $\smash{(\tilde{X},Z)}$ at each noise level, the base features $Z$, as well as the mask-conditioned model with all candidate features observed. The results are summarized in Table~\ref{app:tab:acc}.

As expected, classification accuracy generally decreases as features are degraded or removed. The mask-conditioned models also obtain comparable performance to models trained on the full feature set $(X,Z)$, despite supporting predictions given arbitrary subsets of features. The models using the augmented features do not achieve state-of-the-art performance on these classification tasks, since we use simplified architectures and training procedures. Our primary goal is not to maximize predictive accuracy, but to evaluate our theoretical results under realistic levels of model imperfection.

We additionally evaluate the marginal coverage of the conformal predictors using $(X,Z)$ at the nominal coverage level $1 - \alpha = 0.9$, which is used throughout most of our experimental results. For each dataset, we randomly split the evaluation pool in half into calibration and test sets over 100 replicates. As shown in Figure~\ref{app:fig:coverage}, both the probability and APS scores achieve empirical coverage close to the nominal level on all datasets, with larger variability on the datasets with smaller evaluation pools, such as Wine and HAR.

\clearpage

\section{Analysis details} \label{app:sec:analysis}

We provide additional details on the estimation of the bounds, the evaluation protocols, and the analysis metrics used in our validation and feature selection experiments.

\subsection{Estimation of the bounds}

To evaluate the conformal mutual information terms in Theorem~\ref{thm:sandwich_prob_score}, we estimate $I_\lambda(X; Y \mid Z)$ with the plug-in (or oracle) probability distributions on the test set at the calibrated $\smash{\hat\lambda_\alpha^{\scriptscriptstyle (i)}}$. For Theorem~\ref{thm:sandwich_aps}, we compute $I^{\scriptscriptstyle (1)}$ in the same way, since $\smash{\hat\Lambda_\alpha^{\scriptscriptstyle (1)}}$ depends only on $Z$. However, the term $I^{\scriptscriptstyle (2)}$ in the lower bound cannot be computed using a plug-in estimate. Recall that
\begin{equation}
    I^{\scriptscriptstyle (2)} = H_{\hat\Lambda_\alpha^{\scriptscriptstyle (2)}}(Y \mid Z) - H_{\hat\Lambda_\alpha^{\scriptscriptstyle (2)}}(Y \mid X, Z).
\end{equation}
The second term can be evaluated separately for each test point, but the first term is problematic because $\smash{\hat\Lambda_\alpha^{\scriptscriptstyle (2)}}$ depends on both $X$ and $Z$. In particular,
\begin{equation}
    H_{\hat\Lambda_\alpha^{\scriptscriptstyle (2)}}(Y \mid Z) = \mathbb{E}_Z\Bigl[\sum\nolimits_{y} \min\Bigl\{1,\, \mathbb{E}\Bigl[\tfrac{p(y \mid X, Z)}{\hat\Lambda_\alpha^{\scriptscriptstyle (2)}} \Bigm| Z\Bigr]\Bigr\}\Bigr].
\end{equation}
Since the conditional expectation over $X \mid Z$ lies inside the minimum, computing it requires multiple draws of $X$ for each $Z$. In the synthetic settings, $p(x \mid z)$ is known, so we estimate it by Monte Carlo sampling with $M = 256$ draws of $X \mid Z$ for each test point. For the real-world datasets, we instead average $\smash{q(y \mid x_j, z_j) / \hat\Lambda_\alpha^{\scriptscriptstyle (2)}(x_j, z_j)}$ over the $k = 10$ nearest other neighbors $j$ of each test point in the evaluation pool. Neighbors are computed using Euclidean distance on the base features $Z$ for tabular data and cosine distance on the learned embedding of $Z$ for image data.

In the synthetic settings, we use the oracle conditional distributions, so $D^{\scriptscriptstyle (i)} = 0$ exactly. For the real-world datasets, estimating $D^{\scriptscriptstyle (i)}$ requires the true conditional distribution, so we set $D^{\scriptscriptstyle (i)} = 0$, which is consistent with the plug-in estimates of $I^{\scriptscriptstyle (i)}$. We emphasize that the resulting bounds for the real-world datasets are therefore approximate. These approximations do not undermine our validation, since the synthetic settings evaluate the bounds using the true conditional distributions, while the real-world datasets complement these results under realistic model imperfections.

The terms involving miscoverage rates are estimated by the test-set average of
\begin{equation}
    \frac{\mathbf{1}[Y \notin \mathcal{C}_\alpha(X^{\scriptscriptstyle (2)})] - \mathbf{1}[Y \notin \mathcal{C}_\alpha(X^{\scriptscriptstyle (1)})]}{\hat\lambda_\alpha^{\scriptscriptstyle (i)}}.
\end{equation}
For the APS bounds, we analogously replace $\smash{\hat\lambda_\alpha^{\scriptscriptstyle (i)}}$ with $\smash{\hat\Lambda_\alpha^{\scriptscriptstyle (i)}}$. Since the random threshold is a function of the input, this estimator is unbiased for the corresponding calibration terms in the bounds, conditional on the calibration set. 

\subsection{Evaluation protocols}

As described in Section~\ref{sec:experiments}, we reserve a dedicated pool of $n_\mathrm{eval}$ samples from each synthetic and real-world dataset for calibration and testing. We next describe how this evaluation pool is used in each experiment.

\subsubsection{Validation experiments}

We first describe the evaluation protocol for the validation experiments. For each experimental replicate, we sample a calibration dataset of size $n_\mathrm{cal}$ and a test dataset of size $n_\mathrm{test}$ without replacement from the evaluation pool. The calibration and test datasets are disjoint but do not necessarily partition the entire pool. For a given random seed, the sampled datasets are nested across sample sizes, while different seeds yield independent draws.

For each metric, we consider three regimes: (i) jointly varying $n_\mathrm{cal} = n_\mathrm{test}$, (ii) varying $n_\mathrm{cal}$ with $n_\mathrm{test}$ fixed, and (iii) varying $n_\mathrm{test}$ with $n_\mathrm{cal}$ fixed. For the normalized margin and violation rate, we use a logarithmically spaced grid of five to seven points ranging from $100$ to $\min(10{,}000,n_\mathrm{eval}/2)$, inclusive. Because the bound width becomes visually uninformative at larger sample sizes, we instead use a logarithmically spaced grid from $10$ to $\min(1{,}000, n_\mathrm{eval}/2)$. In regimes (ii) and (iii), the fixed sample size is set to the maximum of this grid.

Calibration follows the standard split conformal prediction procedure using either the probability score or APS score. For APS, we use the same score randomization draw $U \sim \mathrm{Unif}[0,1]$ between the procedures using $X^{\scriptscriptstyle (1)}$ and $X^{\scriptscriptstyle (2)}$. Both procedures are calibrated on the same calibration set and evaluated on the same test set. We estimate $\Delta_\alpha$ by the difference in average prediction set size on the test set under $X^{\scriptscriptstyle (1)}$ and $X^{\scriptscriptstyle (2)}$.

We additionally vary the nominal coverage level over $1-\alpha \in \{0.8, 0.85, 0.9, 0.95, 0.99\}$. For these experiments, $n_\mathrm{cal}$ and $n_\mathrm{test}$ are held fixed at the respective maximum sample sizes used for each metric, i.e., $\min(10{,}000, n_\mathrm{eval}/2)$ for the normalized margin and violation rate and $\min(1{,}000, n_\mathrm{eval}/2)$ for the bound width.

For the DPI validation experiments, we follow a similar calibration and evaluation procedure by sampling from the same evaluation pool. We fix the test set size at $n_\mathrm{test} = \min(10{,}000, n_\mathrm{eval}/2)$ and vary \mbox{$n_\mathrm{cal} \in \{10, 100, 1{,}000\}$}, omitting values of $n_\mathrm{cal}$ that exceed $n_\mathrm{eval} / 2$. Within each replicate, the procedures using the augmented and degraded features share the same calibration and test sets. We intentionally consider small calibration sample sizes to illustrate how the normalized DPI gap concentrates as $n_\mathrm{cal}$ increases and to demonstrate that finite-sample calibration variability can result in a negative gap at sufficiently small $n_\mathrm{cal}$.

\subsubsection{Feature selection experiments}

We next describe the evaluation protocol for the feature selection experiments. For each of 10 random seeds, we randomly partition the evaluation pool into three disjoint subsets, using $50\%$ of the samples as a selection set, $25\%$ as a calibration set, and the remaining $25\%$ as a test set. The selection set is used to compute the criteria and determine the feature ordering, while the calibration and test sets are used only to evaluate the resulting sequence. In the synthetic settings, the base features $Z$ are always observed, whereas no features are known initially for the real-world datasets.

Note that this setting contrasts with the active feature acquisition (AFA) literature, which typically selects features to acquire for each individual instance and trains dedicated acquisition policies to optimize predictive performance under a feature budget~\citep{shim2018joint}. We use a different evaluation protocol because our goal is not to propose a competitive AFA method. Instead, we use a simple protocol to empirically probe how the criteria from our theory rank features in practice.

Let $S$ denote the set of candidate features, initialized as $S = \emptyset$. At each step, we evaluate a selection criterion for every remaining candidate $X_j$ using the predicted probabilities given $X_{S \cup \{j\}}$, add the highest-scoring candidate to $S$, and repeat until all candidates are exhausted. This yields a single global feature ordering for each criterion and seed, which is shared across all test samples. For each prefix of the ordering, we then calibrate the conformal predictor on the calibration set at nominal coverage $1 - \alpha = 0.9$ and evaluate set size, coverage, and accuracy on the test set.

We compare the following selection criteria, each computed using the predicted probabilities of the oracle or mask-conditioned model as a plug-in estimate.
\begin{itemize}[leftmargin=*, topsep=2pt]
    \item $\Delta_\alpha$ (Prob.): the reduction in the average set size using probability score with $X_{S \cup \{j\}}$;

    \item $\Delta_\alpha$ (APS): the reduction in the average set size using APS score with $X_{S \cup \{j\}}$;
    
    \item $I_\lambda$ (Prob.): the conformal mutual information $I_\lambda(X_j; Y \mid X_S)$ evaluated at the geometric mean $\smash{\lambda = (\hat\lambda_\alpha^{\scriptscriptstyle (1)} \hat\lambda_\alpha^{\scriptscriptstyle (2)})^{1/2}}$ of the calibrated thresholds using $X_S$ and $X_{S \cup \{j\}}$;
    
    \item Shannon MI: the plug-in estimate $\hat I(X_j; Y \mid X_S) = \hat H(Y \mid X_S) - \hat H(Y \mid X_{S \cup \{j\}})$;
    
    \item Accuracy: the empirical top-1 accuracy gain with $X_{S \cup \{j\}}$;
    
    \item Random: a uniformly random ordering of the candidate features.
\end{itemize}
Since the conformal-based criteria require separate calibration data, we use 5-fold cross-fitting within the selection set. We calibrate the thresholds on four folds and evaluate the criterion on the held-out fold, then average the resulting values across folds. The Shannon mutual information and accuracy criteria do not require calibration and are computed directly on the full selection set.

We do not include an analogous $I_\Lambda$ criterion for the APS score since it requires the same $k$-NN approximation of the conditional expectation over $X_j \mid X_S$ as described above for $I^{\scriptscriptstyle (2)}$. Repeating this approximation for every candidate at every step is computationally expensive and introduces additional approximation error, so we omit this criterion.

\subsection{Metrics}

We next formalize the evaluation metrics in further detail.

\subsubsection{Validation experiments}

In the validation experiments, we consider four main metrics: normalized margin, violation rate, normalized width, and normalized DPI gap. Except for the violation rate, we normalize all metrics by the number of classes $K$ in order to compare across datasets. We next define each metric formally. 

Let $[L,U]$ denote the estimated bounds from Theorem~\ref{thm:sandwich_prob_score} for the probability score or Theorem~\ref{thm:sandwich_aps} for APS. When computing the normalized margin, violation rate, and normalized width, we discard replicates with invalid bounds ($U < L$), which occur in roughly $3\%$ of cases, and average over the remaining random splits of the data. This exclusion does not change our conclusions.

The \emph{normalized margin} quantifies the magnitude by which the empirical $\Delta_\alpha$ falls outside the theoretical bounds for a particular realization of the calibration and test data. Although the theoretical bounds hold exactly for the population quantities under the assumptions of our results, finite-sample Monte Carlo estimation and model approximation on the real-world datasets can lead to empirical violations. We define
\begin{equation}
    \mathrm{Normalized\ Margin} = \frac{d(\Delta_\alpha,[L,U])}{K} = \frac{\max\{L-\Delta_\alpha,\,\Delta_\alpha-U,\,0\}}{K},
\end{equation}
where $d(x,\mathcal{I})$ denotes the distance from $x$ to the interval $\mathcal{I}$, with $d(x,\mathcal{I})=0$ whenever $x \in \mathcal{I}$.

The \emph{violation rate} examines how frequently the empirical $\Delta_\alpha$ falls outside the theoretical bounds, providing a binary measure of empirical validity that does not account for the magnitude of a violation. We define
\begin{equation}
    \mathrm{Violation\ Rate} = \mathbb{P}\bigl(\Delta_\alpha \notin [L,U]\bigr),
\end{equation}
where the probability is estimated over random replicates of the calibration and test data. In practice, we expand the lower and upper bounds by a small slack of $10^{-4}$ to account for floating-point error when computing this metric.

The \emph{normalized width} measures the informativeness of the bounds through the size of the interval $[L,U]$. We define
\begin{equation}
    \mathrm{Normalized\ Width} = \frac{U-L}{K}.
\end{equation}
Since we exclude replicates with $U < L$, the normalized width is non-negative for every replicate.

Finally, the \emph{normalized DPI gap} compares the conformal set-size reduction obtained by augmenting the base features $Z$ with either $X$ or its degraded counterpart $\smash{\tilde{X}}$. For each replicate, we use identical test sets to estimate the expected set-size reductions $\Delta_\alpha$ and $\smash{\tilde{\Delta}_\alpha}$ and define
\begin{equation}
    \mathrm{Normalized\ DPI\ Gap} = \frac{\Delta_\alpha-\smash{\tilde{\Delta}_\alpha}}{K}.
\end{equation}
We examine the distribution of this gap across random draws of the calibration and test data.

\subsubsection{Feature selection experiments}

In the feature selection experiments, we evaluate each criterion by the quality of the conformal predictors obtained along its selection trajectory. Let $S_j$ denote the first $j$ candidates acquired under a given criterion. For each $j$, we compute the average prediction set size and empirical coverage on the test set using $X_{S_j}$ for both the probability score and the APS score. We additionally report the top-1 and top-5 classification accuracy of the model given $X_{S_j}$ on the test set.

Let $\mathcal{E}_j$ denote any one of these metrics evaluated using the first $j$ features. To summarize performance across the entire selection trajectory, we report the average
\begin{equation}
    \bar{\mathcal{E}} = \frac{1}{m - 1} \sum_{j=1}^{m-1} \mathcal{E}_j,
\end{equation}
where we exclude the cases $j = 0$ and $j = m$ since all criteria observe identical features at the endpoints. Intuitively, $\bar{\mathcal{E}}$ approximates the area under the curve of each metric normalized by $m$. This summarizes the performance across all steps rather than at a single fixed budget.

\subsection{Implementation details}

The validation experiments require evaluating a large number of calibration and test draws across sample sizes, random seeds, nominal coverage levels, feature sets, and conformity scores. To reduce the computational cost, we presort the predicted class probabilities for the evaluation pool and use binary search to efficiently compute the quantities required for each replicate. This optimization is an exact reformulation of the standard computation and does not affect any reported results.

For an evaluation pool of $N$ samples and $K$ classes, presorting requires $O(NK\log K)$ time. Each subsequent replicate requires $O(n_\mathrm{cal}\log n_\mathrm{cal} + n_\mathrm{test}\log K)$ time. This provides a substantial computational speedup over directly recomputing the conformal prediction sets, particularly for datasets with large $K$ such as ImageNet-1k. 

The presorting approach does not apply to the feature selection experiments, since the evaluated subsets depend on the selection criterion and the number of possible subsets grows exponentially. Instead, we naively compute the predicted probabilities for each candidate subset at each step.

\clearpage

\section{Additional experiments} \label{app:sec:experiments}

In this section, we further investigate the empirical behavior of our sandwich bounds for set-size reduction using the probability and APS scores. We then provide additional results for the DPI under different noise levels and report feature selection results across the remaining datasets. Details of the evaluation procedures and metrics are described in Appendix~\ref{app:sec:analysis}.

\subsection{Sandwich validation results}

We provide additional results demonstrating the empirical behavior of Theorems~\ref{thm:sandwich_prob_score}~and~\ref{thm:sandwich_aps}. We first investigate the violation rate of the bounds as we jointly vary $n_\mathrm{cal}$ and $n_\mathrm{test}$. We show these results in Figure~\ref{app:fig:sandwich_all}, together with the normalized margin and width results from Figure~\ref{fig:validation_experiments} for comparison.

Consistent with the normalized margin, the violation rate is nonzero at smaller calibration and test set sizes but decreases toward zero as the sample size increases. The main outlier is the Wine dataset using the probability score, for which the violation rate remains high over the evaluated range. This behavior largely reflects the binary nature of the violation rate. The corresponding normalized margin behaves similarly to those of the other datasets, indicating that the violations themselves are small in magnitude.

Comparing the probability and APS scores, we find that the violation rate is typically lower for APS across datasets and sample sizes. The lower violation rate for APS is consistent with its wider bounds, which are easier to satisfy and less sensitive to finite-sample estimation error.

We next examine the normalized margin, violation rate, and normalized width as we independently and jointly vary $n_\mathrm{cal}$ and $n_\mathrm{test}$ in Figures~\ref{app:fig:sandwich_margin}--\ref{app:fig:sandwich_width}. Additionally, we investigate the three metrics across different nominal coverage levels in Figure~\ref{app:fig:sandwich_alpha}.

\textbf{Normalized margin (Figure~\ref{app:fig:sandwich_margin}).~}~The magnitude of the empirical violations is driven primarily by the size of the test set. As $n_\mathrm{test}$ increases with $n_\mathrm{cal}$ fixed, the normalized margin decreases toward zero. In contrast, with $n_\mathrm{test}$ fixed, the margin is largely insensitive to $n_\mathrm{cal}$. This suggests that the observed violations primarily arise from finite test-set estimation error.

\vspace{12pt}
\begin{figure}[ht]
    \centering
    \includegraphics[width=\textwidth]{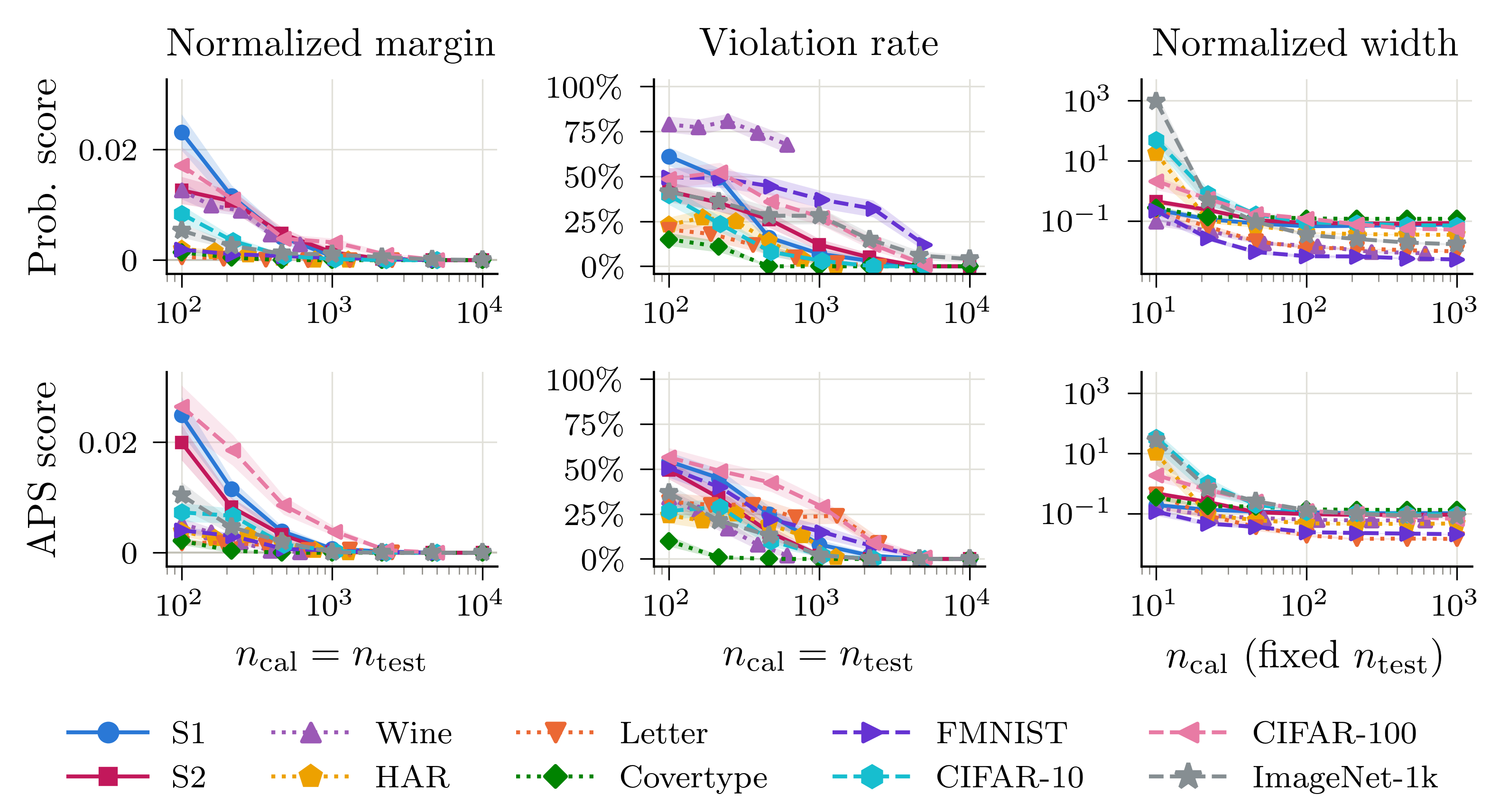}
    \caption{We show the normalized margin and violation rate as we jointly vary $n_\mathrm{cal}$ and $n_\mathrm{test}$ \mbox{(left and middle)} and the normalized width as we vary $n_\mathrm{cal}$ while keeping $n_\mathrm{test}$ fixed (right). All panels report the \mbox{mean $\pm$ standard error} based on $100$ random draws of the calibration and test sets from the fixed evaluation pool.}
    \label{app:fig:sandwich_all}
\end{figure}

\newpage

\begin{figure}[ht]
    \centering
    \includegraphics[width=\textwidth]{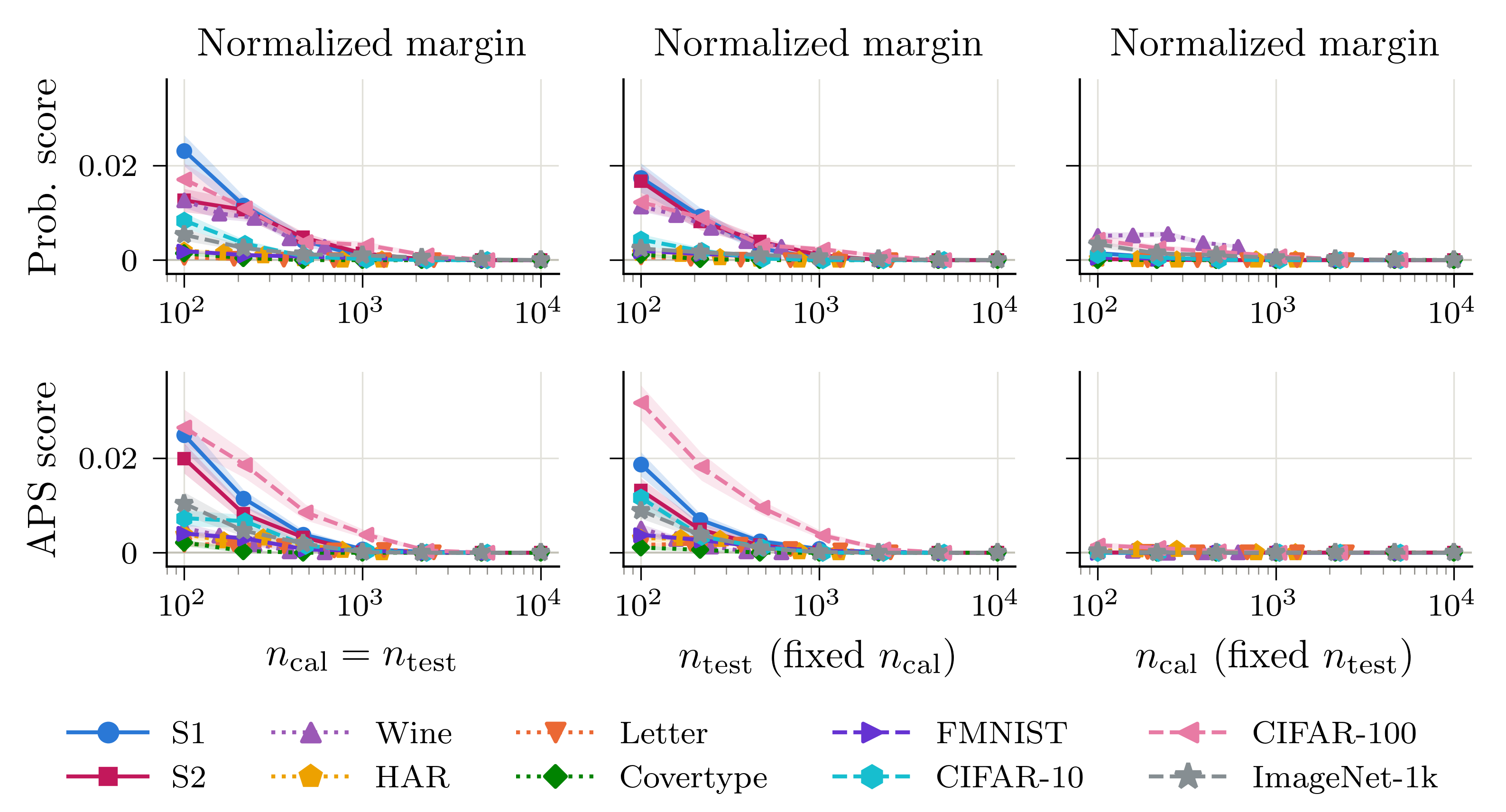}
    \caption{We show the normalized margin as we vary $n_\mathrm{cal}$ and $n_\mathrm{test}$ jointly (left) or independently (middle and right). All panels report the \mbox{mean $\pm$ standard error} based on $100$ random draws of the calibration and test sets from the fixed evaluation pool.}
    \label{app:fig:sandwich_margin}
\end{figure}

\begin{figure}[ht]
    \centering
    \includegraphics[width=\textwidth]{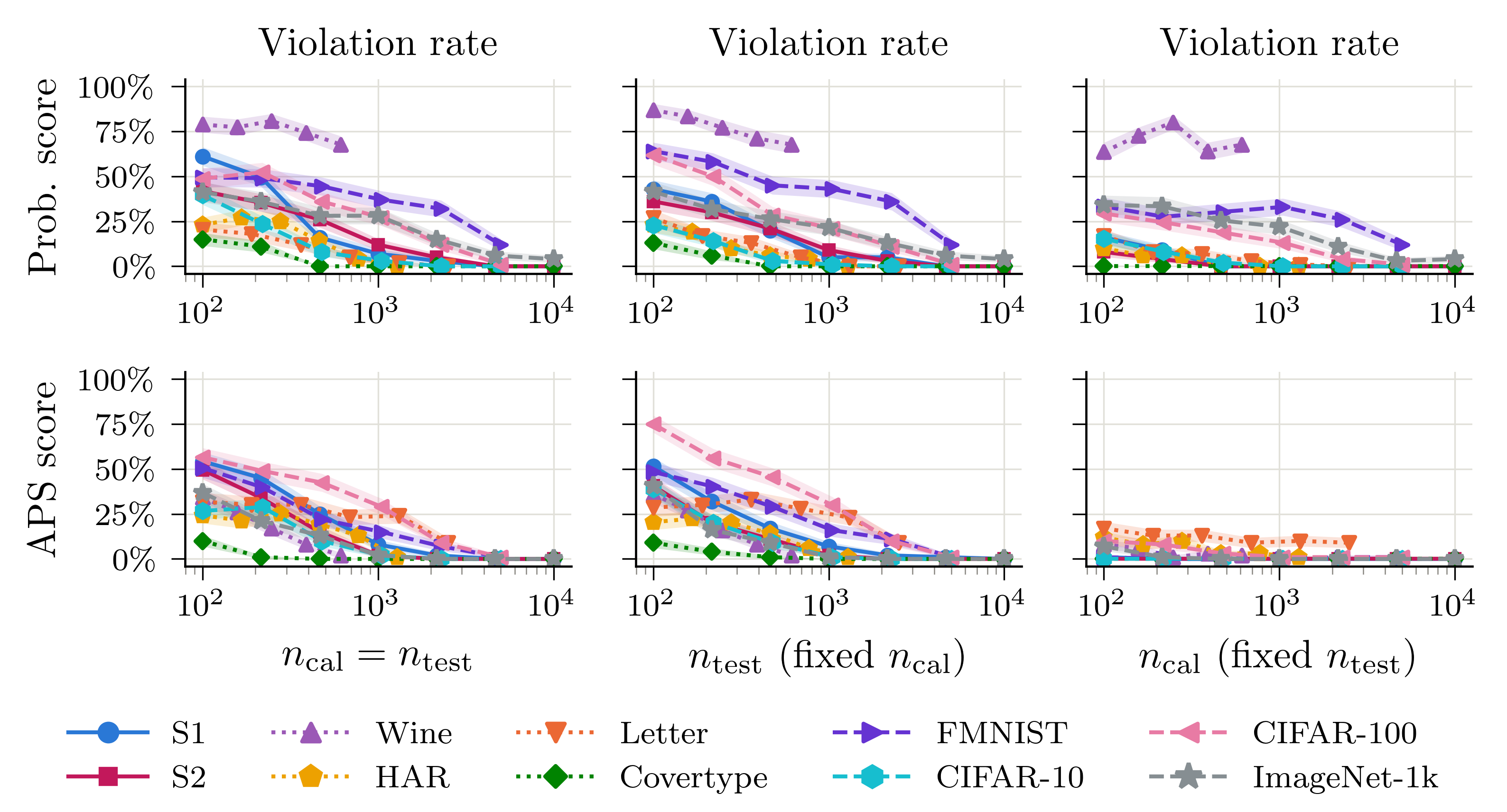}
    \caption{We show the violation rate as we vary $n_\mathrm{cal}$ and $n_\mathrm{test}$ jointly (left) or independently (middle and right). All panels report the \mbox{mean $\pm$ standard error} based on $100$ random draws of the calibration and test sets from the fixed evaluation pool.}
    \label{app:fig:sandwich_violation}
\end{figure}

\textbf{Violation rate (Figure~\ref{app:fig:sandwich_violation}).~}~The violation rate exhibits a similar dependence on $n_\mathrm{test}$. When $n_\mathrm{test}$ is fixed, the violation rate for APS is low and largely insensitive to $n_\mathrm{cal}$, whereas for the probability score it decreases toward zero as $n_\mathrm{cal}$ increases. The Wine dataset is the sole exception, where the violation rate remains between roughly $60\%$ and $80\%$.

\textbf{Normalized width (Figure~\ref{app:fig:sandwich_width}).~}~The normalized width mostly depends on the size of the calibration set and decreases quickly with $n_\mathrm{cal}$. This behavior agrees with our theoretical analysis of the finite-sample calibration terms. As the calibration size increases, the training-conditional miscoverage rates concentrate around the nominal coverage level, causing the finite-sample calibration terms to shrink and the bounds to tighten.

\begin{figure}[ht]
    \centering
    \includegraphics[width=\textwidth]{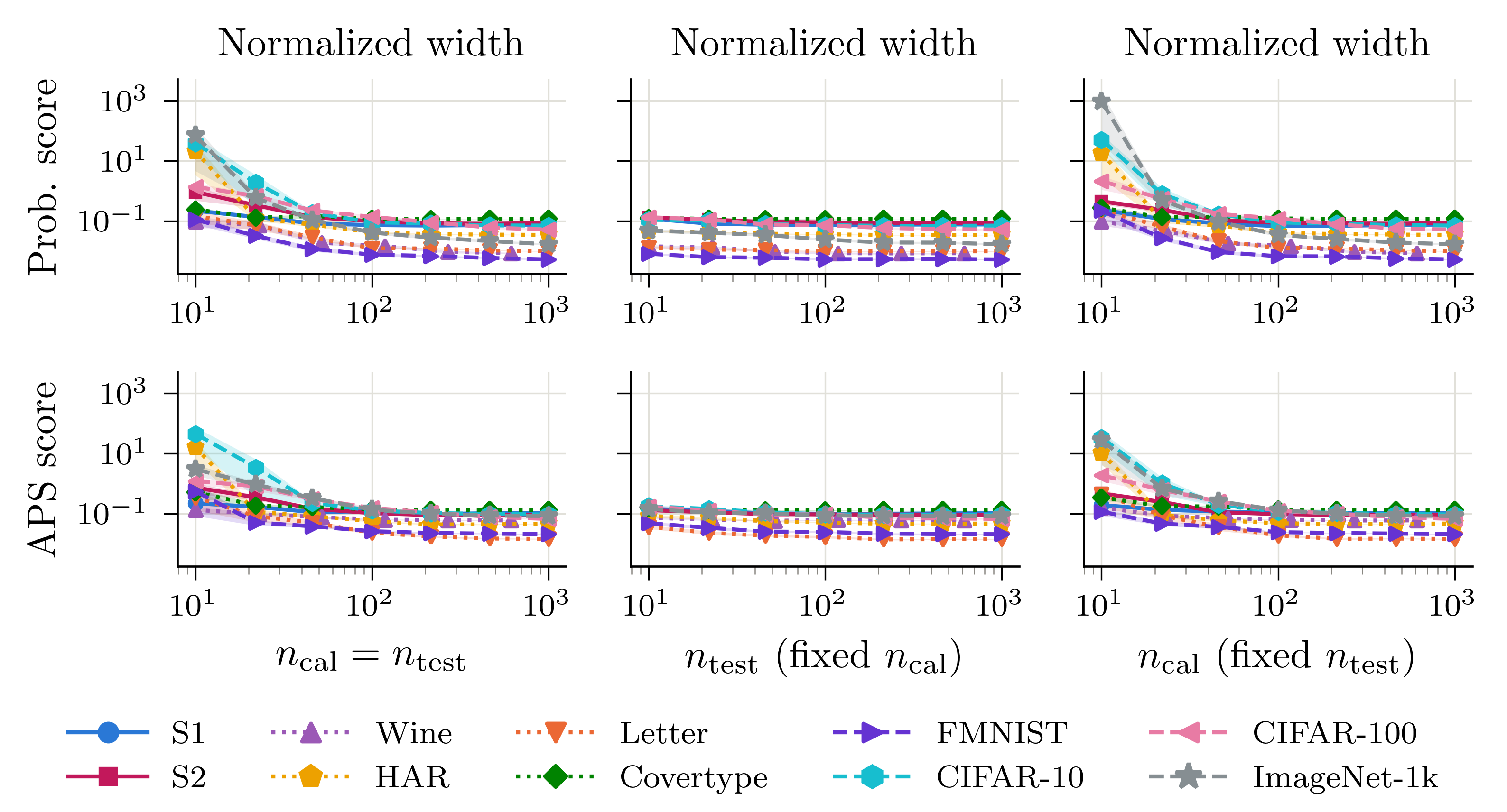}
    \caption{We show the normalized width as we vary $n_\mathrm{cal}$ and $n_\mathrm{test}$ jointly (left) or independently (middle and right). All panels report the \mbox{mean $\pm$ standard error} based on $100$ random draws of the calibration and test sets from the fixed evaluation pool.}
    \label{app:fig:sandwich_width}
\end{figure}

\begin{figure}[ht]
    \centering
    \includegraphics[width=\textwidth]{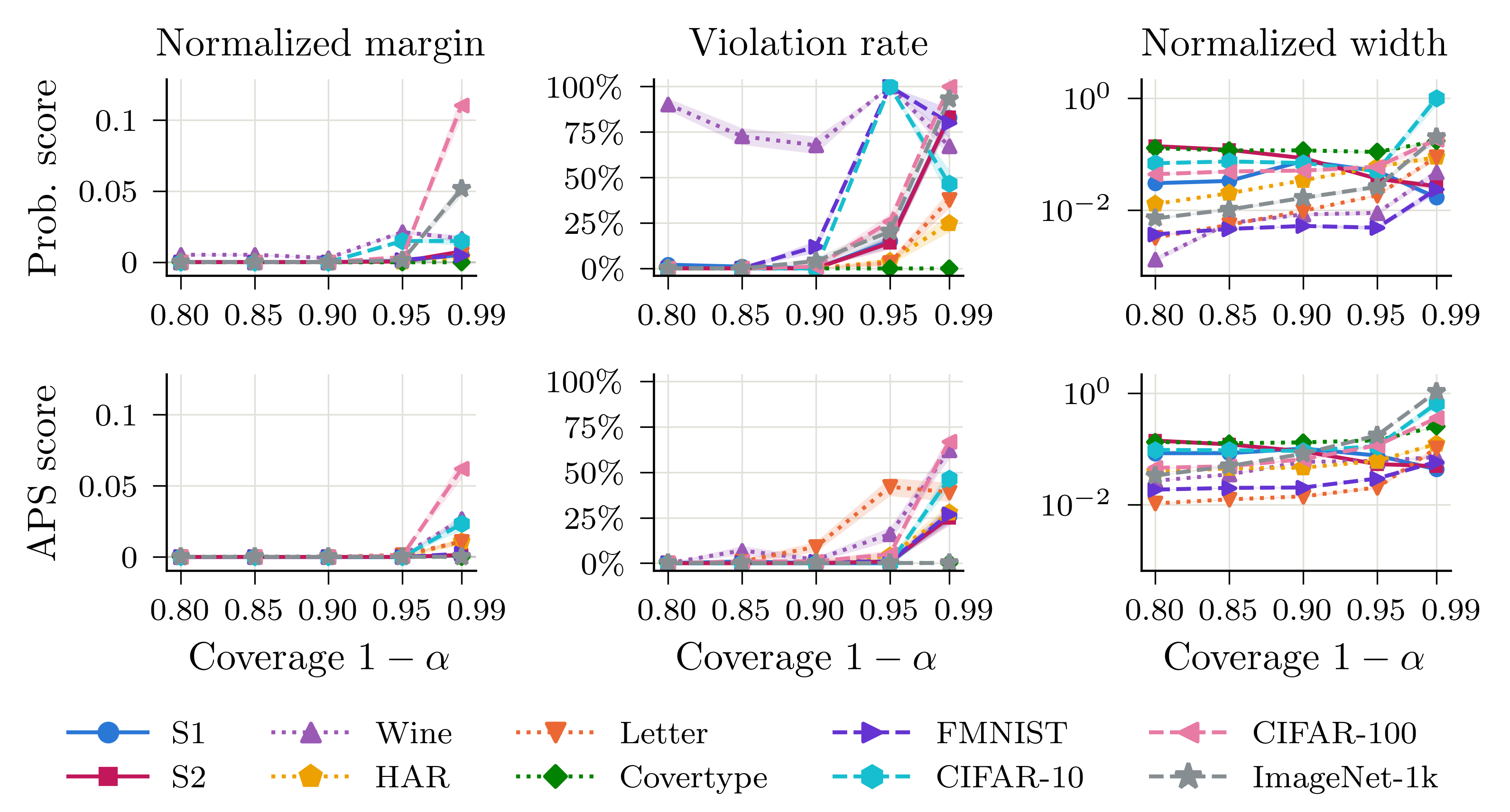}
    \caption{We show the normalized margin, violation rate, and normalized width as we vary the nominal coverage level $1 - \alpha \in \{0.8, 0.85, 0.9, 0.95, 0.99\}$, with $n_\mathrm{cal}$ and $n_\mathrm{test}$ fixed at their largest values for each dataset. All panels report the \mbox{mean $\pm$ standard error} based on $100$ random draws of the calibration and test sets from the fixed evaluation pool.}
    \label{app:fig:sandwich_alpha}
\end{figure}

\textbf{Nominal coverage levels (Figure~\ref{app:fig:sandwich_alpha}).~}~The bounds remain well-behaved across nominal coverage levels $1 - \alpha \leq 0.95$ for most datasets (with exceptions for Wine, FMNIST, and CIFAR-10). However, at $1 - \alpha = 0.99$, all three metrics increase sharply for most datasets. We attribute this behavior to sensitivity to the calibrated thresholds $\smash{\hat{\lambda}^{\scriptscriptstyle (i)}_\alpha}$ and $\smash{\hat{\Lambda}^{\scriptscriptstyle (i)}_\alpha}$. As the nominal coverage level increases, these thresholds become smaller, which amplifies the finite-sample calibration terms and consequently increases the width of the bounds. The smaller thresholds also increase the variance of the corresponding empirical estimates, leading to larger normalized margins and violation rates at fixed calibration and test sample sizes.

\subsection{DPI validation results}

We show the distribution of the normalized DPI gap across calibration sample sizes and degradation levels in Figure~\ref{app:fig:dpi}. The results for the medium degradation level ($\gamma = 0.5$) correspond to those reported in Section~\ref{sec:experiments}.

Across all degradation levels, the realized gap can be negative for particular draws of the calibration and test data. As the calibration sample size increases, however, the variability of the gap decreases and its distribution generally concentrates around a non-negative mean. Increasing the degradation level also generally increases the average gap as expected, since stronger noise removes more of the information present in $X$. These trends hold for both probability score and APS score.

These results further support the interpretation of conformal set-size reduction as a measure of information gain. Theorems~\ref{thm:dpi_prob_score} and~\ref{thm:dpi_aps} establish an approximate data processing inequality, with finite-sample calibration and model-error terms that allow the DPI gap to be negative. Empirically, the gap is generally non-negative on average, providing evidence of this data processing behavior in practice. At the same time, the negative gaps observed for individual replicates demonstrate that the deviations permitted by our theoretical results do indeed occur in practice.

\begin{figure}[t]
    \centering
    \includegraphics[width=\textwidth]{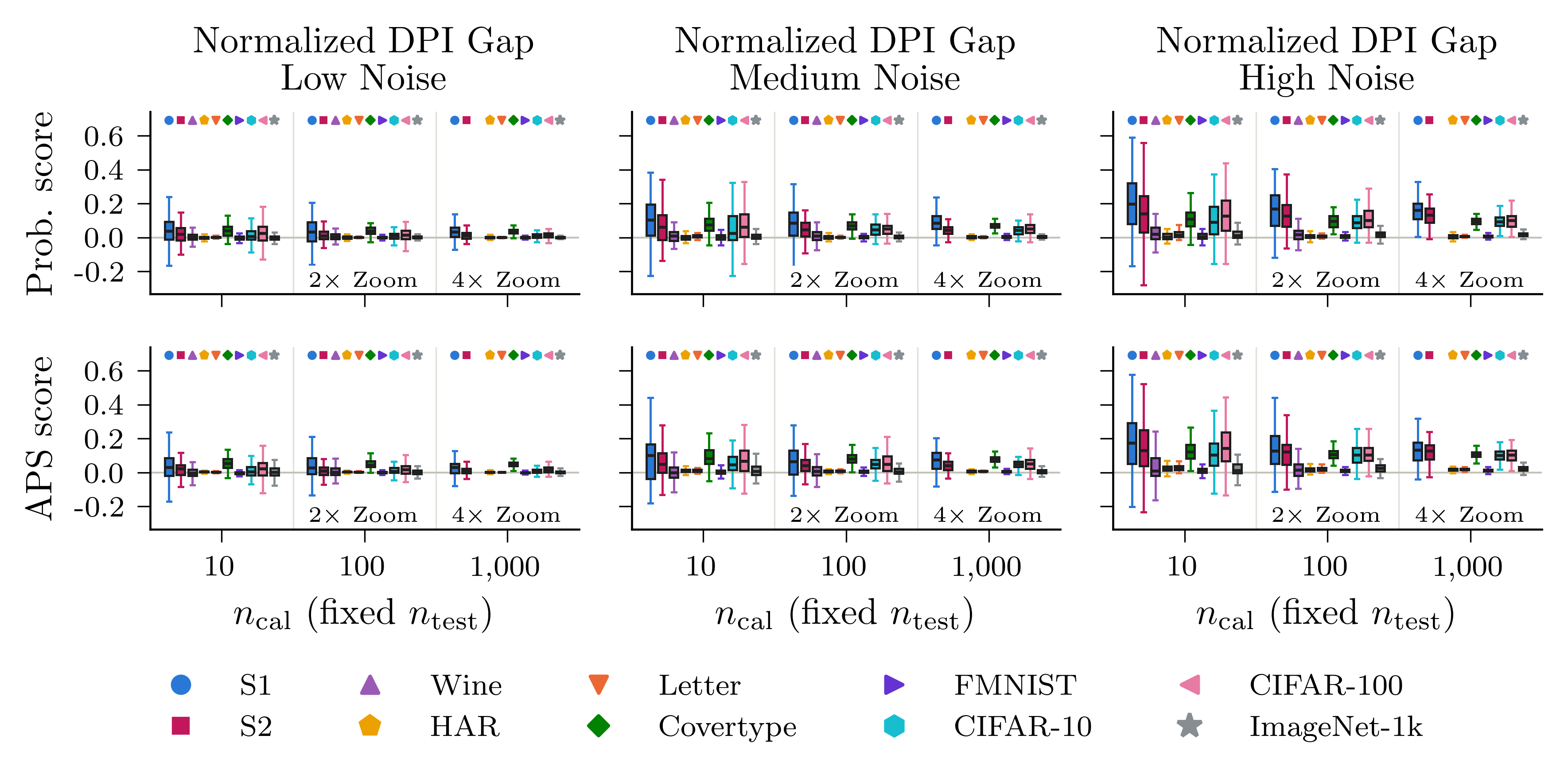}
    \caption{We show the normalized DPI gap across calibration sample sizes and degradation levels, using boxplots to summarize 100 random draws of the calibration and test sets. Each column corresponds to a degradation level, and the marker for each dataset is indicated above the corresponding boxplots. For visual clarity, we magnify the $y$-axis by \mbox{$2 \times $ and $4 \times$} for $n_\mathrm{cal}=100$ and $n_\mathrm{cal}=1{,}000$, respectively. Datasets whose evaluation pool is too small for a given calibration size are omitted.}
    \label{app:fig:dpi}
\end{figure}

\subsection{Feature selection results}

We report feature selection results for all datasets in Table~\ref{app:tab:feature_selection}, excluding synthetic setting S1 since it has too few candidate features to meaningfully distinguish the criteria.

Across most settings, selection based on set-size reduction yields the smallest prediction sets while typically achieving the best top-5 accuracy. The main exceptions occur for the APS score on Wine and FMNIST. The two $\Delta_\alpha$ criteria perform similarly to each other but do not always yield the smallest sets under their own score. Since these differences are often within one standard error, they likely reflect estimation noise in the selection set.

We also evaluate selection using $I_\lambda$, which performs similarly to $\Delta_\alpha$ under the probability score in the synthetic settings but often yields larger sets on the real-world datasets. We attribute this gap to model imperfection. Theorem~\ref{thm:sandwich_prob_score} brackets $\Delta_\alpha$ by conformal mutual information only up to the model error terms $D^{\scriptscriptstyle (i)}$, so the plug-in $I_\lambda$ at the geometric mean of the calibrated thresholds becomes a less precise proxy for $\Delta_\alpha$ when these terms are nonzero.

\begin{table}[t]
    \centering
    \caption{We evaluate feature selection performance using conformal set-size reduction. Each row represents one selection criterion and values are reported as \mbox{mean $\pm$ standard error} over 10 random splits of the evaluation pool. Arrows indicate the direction of improvement, and the best-performing method for each metric is bolded. Empirical coverage is reported for reference without bolding to indicate that differences in set size are not due to undercoverage.}
    \label{app:tab:feature_selection}
    \resizebox{\columnwidth}{!}{%
    \begin{tabular}{c c cccccc}
        \toprule
        \multirow{2}{*}{Dataset} & \multirow{2}{*}{Criterion}
        & \multicolumn{2}{c}{Prob.}
        & \multicolumn{2}{c}{APS}
        & \multicolumn{2}{c}{Acc.} \\
        \cmidrule(lr){3-4} \cmidrule(lr){5-6} \cmidrule(lr){7-8}
        &
        & $\mathbb{E}[|\mathcal{C}|]$ ($\downarrow$)
        & Cov.
        & $\mathbb{E}[|\mathcal{C}|]$ ($\downarrow$)
        & Cov.
        & Top-1 ($\uparrow$)
        & Top-5 ($\uparrow$) \\
        \midrule

        \multirow{6}{*}{\rotatebox[origin=c]{90}{S2}}
        & $\Delta_\alpha$ (Prob.)
        & $11.464 \pm 0.032$
        & $0.900 \pm 0.001$
        & $11.952 \pm 0.029$
        & $0.900 \pm 0.001$
        & $0.366 \pm 0.001$
        & $\mathbf{0.722 \pm 0.000}$ \\
        &
        $\Delta_\alpha$ (APS)
        & $11.459 \pm 0.033$
        & $0.900 \pm 0.001$
        & $11.950 \pm 0.029$
        & $0.900 \pm 0.001$
        & $0.366 \pm 0.001$
        & $\mathbf{0.722 \pm 0.001}$ \\
        &
        $I_\lambda$ (Prob.)
        & $\mathbf{11.456 \pm 0.032}$
        & $0.900 \pm 0.001$
        & $\mathbf{11.948 \pm 0.028}$
        & $0.900 \pm 0.001$
        & $0.366 \pm 0.001$
        & $\mathbf{0.722 \pm 0.001}$ \\
        &
        Shannon MI
        & $11.473 \pm 0.031$
        & $0.900 \pm 0.001$
        & $11.959 \pm 0.027$
        & $0.900 \pm 0.001$
        & $\mathbf{0.367 \pm 0.001}$
        & $\mathbf{0.722 \pm 0.001}$ \\
        &
        Accuracy
        & $11.503 \pm 0.029$
        & $0.900 \pm 0.001$
        & $11.986 \pm 0.026$
        & $0.900 \pm 0.001$
        & $\mathbf{0.367 \pm 0.001}$
        & $\mathbf{0.722 \pm 0.001}$ \\
        &
        Random
        & $12.207 \pm 0.085$
        & $0.900 \pm 0.001$
        & $12.670 \pm 0.085$
        & $0.900 \pm 0.001$
        & $0.354 \pm 0.002$
        & $0.707 \pm 0.002$ \\

        \midrule

        \multirow{6}{*}{\rotatebox[origin=c]{90}{S3}}
        & $\Delta_\alpha$ (Prob.)
        & $2.239 \pm 0.002$
        & $0.900 \pm 0.000$
        & $2.629 \pm 0.002$
        & $0.900 \pm 0.000$
        & $0.750 \pm 0.001$
        & $0.952 \pm 0.000$ \\
        &
        $\Delta_\alpha$ (APS)
        & $2.243 \pm 0.002$
        & $0.900 \pm 0.000$
        & $\mathbf{2.611 \pm 0.002}$
        & $0.900 \pm 0.000$
        & $0.752 \pm 0.000$
        & $0.952 \pm 0.000$ \\
        &
        $I_\lambda$ (Prob.)
        & $\mathbf{2.237 \pm 0.003}$
        & $0.900 \pm 0.000$
        & $2.626 \pm 0.002$
        & $0.900 \pm 0.000$
        & $0.753 \pm 0.000$
        & $\mathbf{0.953 \pm 0.000}$ \\
        &
        Shannon MI
        & $2.447 \pm 0.005$
        & $0.900 \pm 0.000$
        & $2.876 \pm 0.006$
        & $0.900 \pm 0.000$
        & $0.767 \pm 0.000$
        & $0.952 \pm 0.000$ \\
        &
        Accuracy
        & $3.613 \pm 0.020$
        & $0.900 \pm 0.000$
        & $4.487 \pm 0.029$
        & $0.900 \pm 0.000$
        & $\mathbf{0.768 \pm 0.001}$
        & $0.914 \pm 0.001$ \\
        &
        Random
        & $3.046 \pm 0.100$
        & $0.900 \pm 0.000$
        & $3.591 \pm 0.116$
        & $0.900 \pm 0.000$
        & $0.672 \pm 0.009$
        & $0.931 \pm 0.004$ \\

        \midrule

        \multirow{6}{*}{\rotatebox[origin=c]{90}{Wine}}
        & $\Delta_\alpha$ (Prob.)
        & $\mathbf{2.265 \pm 0.034}$
        & $0.904 \pm 0.008$
        & $2.387 \pm 0.029$
        & $0.905 \pm 0.006$
        & $\mathbf{0.553 \pm 0.006}$
        & $\mathbf{0.996 \pm 0.001}$ \\
        &
        $\Delta_\alpha$ (APS)
        & $2.279 \pm 0.042$
        & $0.902 \pm 0.008$
        & $2.378 \pm 0.036$
        & $0.904 \pm 0.007$
        & $0.548 \pm 0.005$
        & $0.995 \pm 0.001$ \\
        &
        $I_\lambda$ (Prob.)
        & $2.334 \pm 0.050$
        & $0.903 \pm 0.010$
        & $2.417 \pm 0.040$
        & $0.904 \pm 0.009$
        & $0.529 \pm 0.005$
        & $0.995 \pm 0.001$ \\
        &
        Shannon MI
        & $2.288 \pm 0.042$
        & $0.900 \pm 0.009$
        & $2.383 \pm 0.036$
        & $0.903 \pm 0.008$
        & $0.550 \pm 0.005$
        & $0.995 \pm 0.001$ \\
        &
        Accuracy
        & $2.284 \pm 0.042$
        & $0.904 \pm 0.009$
        & $\mathbf{2.371 \pm 0.032}$
        & $0.904 \pm 0.007$
        & $0.552 \pm 0.004$
        & $\mathbf{0.996 \pm 0.001}$ \\
        &
        Random
        & $2.420 \pm 0.046$
        & $0.906 \pm 0.007$
        & $2.506 \pm 0.046$
        & $0.902 \pm 0.007$
        & $0.507 \pm 0.005$
        & $0.995 \pm 0.001$ \\

        \midrule

        \multirow{6}{*}{\rotatebox[origin=c]{90}{HAR}}
        & $\Delta_\alpha$ (Prob.)
        & $0.912 \pm 0.005$
        & $0.899 \pm 0.005$
        & $0.963 \pm 0.002$
        & $0.900 \pm 0.002$
        & $\mathbf{0.971 \pm 0.001}$
        & $\mathbf{1.000 \pm 0.000}$ \\
        &
        $\Delta_\alpha$ (APS)
        & $\mathbf{0.909 \pm 0.006}$
        & $0.896 \pm 0.006$
        & $\mathbf{0.960 \pm 0.002}$
        & $0.899 \pm 0.001$
        & $0.970 \pm 0.001$
        & $\mathbf{1.000 \pm 0.000}$ \\
        &
        $I_\lambda$ (Prob.)
        & $0.918 \pm 0.005$
        & $0.897 \pm 0.005$
        & $0.961 \pm 0.001$
        & $0.898 \pm 0.002$
        & $0.968 \pm 0.002$
        & $\mathbf{1.000 \pm 0.000}$ \\
        &
        Shannon MI
        & $0.916 \pm 0.005$
        & $0.898 \pm 0.005$
        & $\mathbf{0.960 \pm 0.001}$
        & $0.899 \pm 0.002$
        & $0.968 \pm 0.002$
        & $\mathbf{1.000 \pm 0.000}$ \\
        &
        Accuracy
        & $0.910 \pm 0.005$
        & $0.898 \pm 0.004$
        & $0.961 \pm 0.003$
        & $0.900 \pm 0.001$
        & $\mathbf{0.971 \pm 0.001}$
        & $\mathbf{1.000 \pm 0.000}$ \\
        &
        Random
        & $0.999 \pm 0.015$
        & $0.900 \pm 0.004$
        & $1.064 \pm 0.017$
        & $0.899 \pm 0.002$
        & $0.934 \pm 0.005$
        & $\mathbf{1.000 \pm 0.000}$ \\

        \midrule

        \multirow{6}{*}{\rotatebox[origin=c]{90}{Letter}}
        & $\Delta_\alpha$ (Prob.)
        & $\mathbf{2.753 \pm 0.016}$
        & $0.898 \pm 0.003$
        & $\mathbf{3.358 \pm 0.019}$
        & $0.900 \pm 0.001$
        & $0.756 \pm 0.001$
        & $\mathbf{0.926 \pm 0.000}$ \\
        &
        $\Delta_\alpha$ (APS)
        & $2.835 \pm 0.038$
        & $0.898 \pm 0.002$
        & $3.398 \pm 0.032$
        & $0.900 \pm 0.001$
        & $0.758 \pm 0.002$
        & $0.924 \pm 0.001$ \\
        &
        $I_\lambda$ (Prob.)
        & $2.824 \pm 0.022$
        & $0.898 \pm 0.003$
        & $3.385 \pm 0.028$
        & $0.901 \pm 0.002$
        & $0.760 \pm 0.001$
        & $0.924 \pm 0.001$ \\
        &
        Shannon MI
        & $2.801 \pm 0.019$
        & $0.899 \pm 0.003$
        & $3.399 \pm 0.026$
        & $0.901 \pm 0.001$
        & $0.758 \pm 0.001$
        & $0.924 \pm 0.001$ \\
        &
        Accuracy
        & $2.879 \pm 0.020$
        & $0.899 \pm 0.002$
        & $3.415 \pm 0.022$
        & $0.901 \pm 0.001$
        & $\mathbf{0.763 \pm 0.001}$
        & $0.922 \pm 0.001$ \\
        &
        Random
        & $4.459 \pm 0.184$
        & $0.897 \pm 0.002$
        & $5.054 \pm 0.174$
        & $0.900 \pm 0.001$
        & $0.642 \pm 0.011$
        & $0.871 \pm 0.007$ \\

        \midrule

        \multirow{6}{*}{\rotatebox[origin=c]{90}{Covertype}}
        & $\Delta_\alpha$ (Prob.)
        & $\mathbf{1.010 \pm 0.001}$
        & $0.900 \pm 0.001$
        & $1.139 \pm 0.001$
        & $0.900 \pm 0.000$
        & $\mathbf{0.906 \pm 0.001}$
        & $\mathbf{1.000 \pm 0.000}$ \\
        &
        $\Delta_\alpha$ (APS)
        & $1.011 \pm 0.001$
        & $0.900 \pm 0.001$
        & $1.140 \pm 0.001$
        & $0.900 \pm 0.000$
        & $\mathbf{0.906 \pm 0.000}$
        & $\mathbf{1.000 \pm 0.000}$ \\
        &
        $I_\lambda$ (Prob.)
        & $\mathbf{1.010 \pm 0.001}$
        & $0.900 \pm 0.001$
        & $\mathbf{1.138 \pm 0.001}$
        & $0.900 \pm 0.000$
        & $\mathbf{0.906 \pm 0.001}$
        & $\mathbf{1.000 \pm 0.000}$ \\
        &
        Shannon MI
        & $\mathbf{1.010 \pm 0.001}$
        & $0.900 \pm 0.001$
        & $\mathbf{1.138 \pm 0.001}$
        & $0.900 \pm 0.000$
        & $\mathbf{0.906 \pm 0.001}$
        & $\mathbf{1.000 \pm 0.000}$ \\
        &
        Accuracy
        & $\mathbf{1.010 \pm 0.001}$
        & $0.900 \pm 0.001$
        & $1.139 \pm 0.001$
        & $0.900 \pm 0.000$
        & $\mathbf{0.906 \pm 0.001}$
        & $\mathbf{1.000 \pm 0.000}$ \\
        &
        Random
        & $1.593 \pm 0.035$
        & $0.901 \pm 0.001$
        & $1.731 \pm 0.036$
        & $0.900 \pm 0.000$
        & $0.717 \pm 0.011$
        & $0.996 \pm 0.000$ \\

        \midrule

        \multirow{6}{*}{\rotatebox[origin=c]{90}{FMNIST}}
        & $\Delta_\alpha$ (Prob.)
        & $\mathbf{1.020 \pm 0.002}$
        & $0.900 \pm 0.002$
        & $1.220 \pm 0.004$
        & $0.898 \pm 0.001$
        & $0.903 \pm 0.001$
        & $\mathbf{0.998 \pm 0.000}$ \\
        &
        $\Delta_\alpha$ (APS)
        & $1.028 \pm 0.002$
        & $0.901 \pm 0.002$
        & $1.167 \pm 0.002$
        & $0.899 \pm 0.001$
        & $0.901 \pm 0.001$
        & $\mathbf{0.998 \pm 0.000}$ \\
        &
        $I_\lambda$ (Prob.)
        & $1.040 \pm 0.002$
        & $0.901 \pm 0.002$
        & $1.166 \pm 0.002$
        & $0.899 \pm 0.001$
        & $0.899 \pm 0.001$
        & $\mathbf{0.998 \pm 0.000}$ \\
        &
        Shannon MI
        & $1.037 \pm 0.002$
        & $0.901 \pm 0.002$
        & $\mathbf{1.163 \pm 0.002}$
        & $0.899 \pm 0.001$
        & $0.901 \pm 0.001$
        & $\mathbf{0.998 \pm 0.000}$ \\
        &
        Accuracy
        & $1.023 \pm 0.002$
        & $0.902 \pm 0.002$
        & $1.220 \pm 0.003$
        & $0.898 \pm 0.001$
        & $\mathbf{0.904 \pm 0.001}$
        & $0.997 \pm 0.000$ \\
        &
        Random
        & $1.064 \pm 0.004$
        & $0.902 \pm 0.002$
        & $1.259 \pm 0.004$
        & $0.899 \pm 0.001$
        & $0.896 \pm 0.001$
        & $0.997 \pm 0.000$ \\

        \midrule

        \multirow{6}{*}{\rotatebox[origin=c]{90}{CIFAR-10}}
        & $\Delta_\alpha$ (Prob.)
        & $\mathbf{1.489 \pm 0.007}$
        & $0.899 \pm 0.003$
        & $1.774 \pm 0.007$
        & $0.899 \pm 0.001$
        & $0.795 \pm 0.002$
        & $\mathbf{0.986 \pm 0.000}$ \\
        &
        $\Delta_\alpha$ (APS)
        & $1.507 \pm 0.006$
        & $0.900 \pm 0.002$
        & $\mathbf{1.742 \pm 0.004}$
        & $0.900 \pm 0.001$
        & $0.793 \pm 0.002$
        & $\mathbf{0.986 \pm 0.000}$ \\
        &
        $I_\lambda$ (Prob.)
        & $1.592 \pm 0.010$
        & $0.900 \pm 0.002$
        & $1.774 \pm 0.008$
        & $0.899 \pm 0.001$
        & $0.781 \pm 0.001$
        & $0.984 \pm 0.000$ \\
        &
        Shannon MI
        & $1.561 \pm 0.007$
        & $0.899 \pm 0.002$
        & $1.756 \pm 0.005$
        & $0.899 \pm 0.001$
        & $0.785 \pm 0.001$
        & $0.984 \pm 0.000$ \\
        &
        Accuracy
        & $1.502 \pm 0.008$
        & $0.901 \pm 0.002$
        & $1.780 \pm 0.004$
        & $0.900 \pm 0.001$
        & $\mathbf{0.800 \pm 0.002}$
        & $0.985 \pm 0.000$ \\
        &
        Random
        & $1.711 \pm 0.012$
        & $0.900 \pm 0.002$
        & $1.996 \pm 0.015$
        & $0.900 \pm 0.002$
        & $0.772 \pm 0.003$
        & $0.978 \pm 0.001$ \\

        \midrule

        \multirow{6}{*}{\rotatebox[origin=c]{90}{CIFAR-100}}
        & $\Delta_\alpha$ (Prob.)
        & $12.683 \pm 0.100$
        & $0.904 \pm 0.002$
        & $15.661 \pm 0.127$
        & $0.902 \pm 0.002$
        & $0.466 \pm 0.002$
        & $\mathbf{0.753 \pm 0.002}$ \\
        &
        $\Delta_\alpha$ (APS)
        & $\mathbf{12.594 \pm 0.088}$
        & $0.902 \pm 0.002$
        & $\mathbf{15.472 \pm 0.129}$
        & $0.902 \pm 0.002$
        & $\mathbf{0.467 \pm 0.002}$
        & $\mathbf{0.753 \pm 0.002}$ \\
        &
        $I_\lambda$ (Prob.)
        & $13.700 \pm 0.115$
        & $0.901 \pm 0.002$
        & $16.438 \pm 0.110$
        & $0.901 \pm 0.002$
        & $0.456 \pm 0.003$
        & $0.740 \pm 0.002$ \\
        &
        Shannon MI
        & $13.763 \pm 0.115$
        & $0.902 \pm 0.002$
        & $16.597 \pm 0.145$
        & $0.902 \pm 0.002$
        & $0.458 \pm 0.003$
        & $0.741 \pm 0.002$ \\
        &
        Accuracy
        & $12.872 \pm 0.130$
        & $0.902 \pm 0.002$
        & $15.877 \pm 0.158$
        & $0.901 \pm 0.002$
        & $\mathbf{0.467 \pm 0.003}$
        & $0.752 \pm 0.002$ \\
        &
        Random
        & $14.243 \pm 0.253$
        & $0.902 \pm 0.003$
        & $17.189 \pm 0.224$
        & $0.902 \pm 0.003$
        & $0.454 \pm 0.002$
        & $0.737 \pm 0.002$ \\

        \midrule

        \multirow{6}{*}{\rotatebox[origin=c]{90}{ImageNet-1k}}
        & $\Delta_\alpha$ (Prob.)
        & $\mathbf{47.159 \pm 0.297}$
        & $0.901 \pm 0.001$
        & $58.993 \pm 0.352$
        & $0.900 \pm 0.001$
        & $0.476 \pm 0.001$
        & $0.711 \pm 0.001$ \\
        &
        $\Delta_\alpha$ (APS)
        & $47.262 \pm 0.322$
        & $0.901 \pm 0.001$
        & $\mathbf{58.155 \pm 0.337}$
        & $0.901 \pm 0.001$
        & $0.475 \pm 0.001$
        & $0.711 \pm 0.001$ \\
        &
        $I_\lambda$ (Prob.)
        & $51.958 \pm 0.443$
        & $0.901 \pm 0.001$
        & $63.887 \pm 0.359$
        & $0.900 \pm 0.001$
        & $0.472 \pm 0.001$
        & $0.707 \pm 0.001$ \\
        &
        Shannon MI
        & $55.959 \pm 0.358$
        & $0.900 \pm 0.001$
        & $68.675 \pm 0.320$
        & $0.900 \pm 0.001$
        & $0.471 \pm 0.001$
        & $0.705 \pm 0.001$ \\
        &
        Accuracy
        & $51.504 \pm 0.494$
        & $0.900 \pm 0.001$
        & $64.796 \pm 0.581$
        & $0.900 \pm 0.001$
        & $\mathbf{0.480 \pm 0.001}$
        & $\mathbf{0.713 \pm 0.001}$ \\
        &
        Random
        & $58.187 \pm 0.522$
        & $0.901 \pm 0.001$
        & $72.623 \pm 0.616$
        & $0.900 \pm 0.001$
        & $0.472 \pm 0.001$
        & $0.704 \pm 0.001$ \\

        \bottomrule
    \end{tabular}%
}

\end{table}

\clearpage

\begin{figure}[t]
    \centering
    \includegraphics[width=\textwidth]{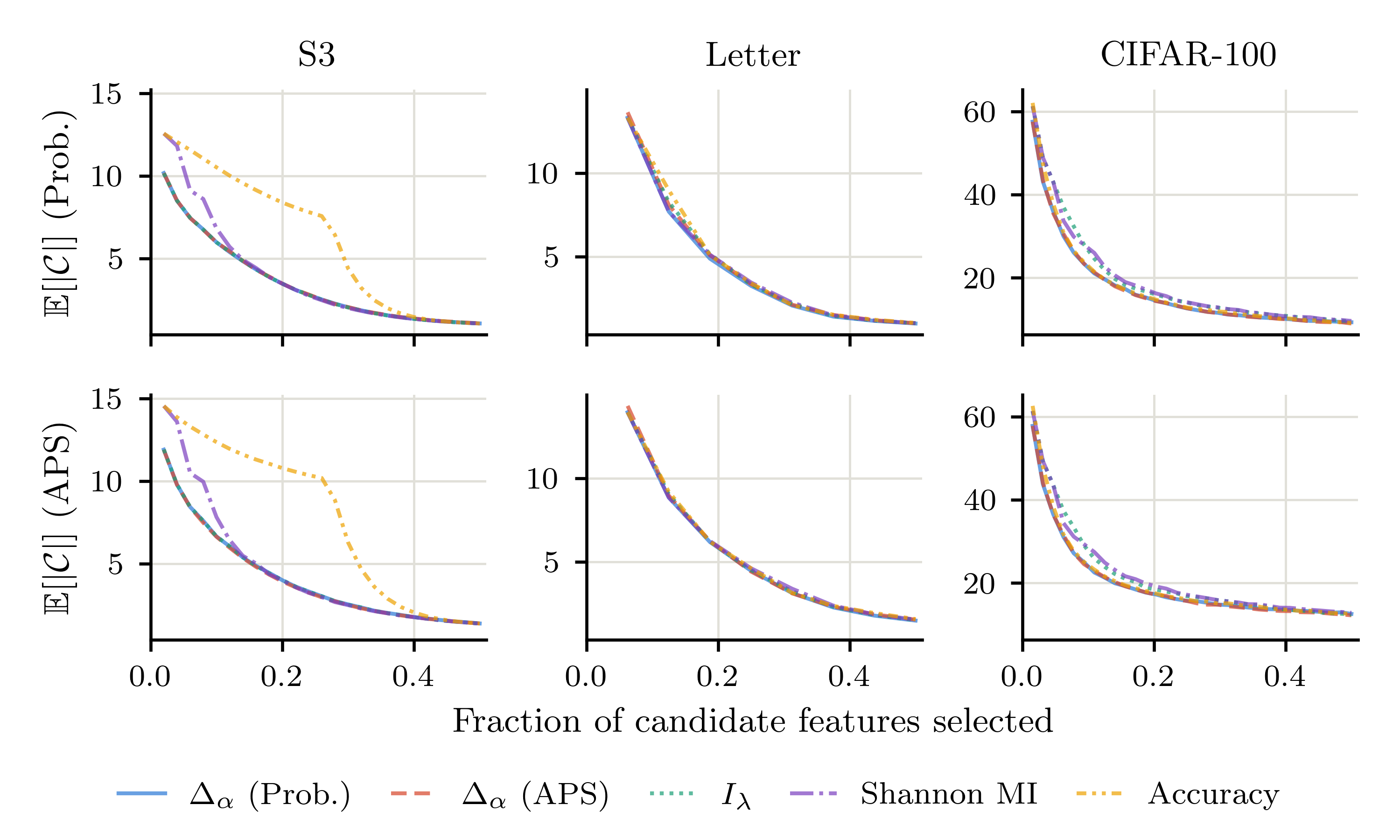}
    \caption{We show the expected set size using the probability score and APS score against the fraction of candidate features selected. Each curve is the mean over 10 random splits of the evaluation pool. We show only the first half of the selection trajectory, after which the curves largely coincide. On S3 and CIFAR-100, the gap between $\Delta_\alpha$ and Shannon MI is concentrated in the early steps, which becomes diluted when averaging over the full trajectory.}
    \label{app:fig:feature_selection}
\end{figure}

The metrics in Table~\ref{app:tab:feature_selection} are averaged over all intermediate steps of the feature ordering. Since every criterion eventually selects all candidate features, the set sizes converge as more features are acquired, and averaging over the full trajectory can mask differences that occur early in selection. To examine this, we show the expected set size at each step for the three datasets from Table~\ref{tab:feature_selection} in Figure~\ref{app:fig:feature_selection}. The $\Delta_\alpha$ criteria yield noticeably smaller sets than Shannon MI when only a small fraction of the candidates has been selected for S3 and CIFAR-100. On Letter, $\Delta_\alpha$ using the probability score provides only a small improvement throughout the trajectory.

\subsection{Subset-specific model ablation} \label{app:subsec:subset_specific_ablation}

In the feature selection experiments, the selection criteria and the evaluation both rely on a single mask-conditioned model. A difference between two criteria could therefore reflect how well this model fits particular subsets of features, rather than the information those features carry about $Y$. To test for this effect, we retrain a subset-specific model for each selected subset on CIFAR-100.

For each criterion and random split, we take the feature ordering from the main experiment and consider its first $k \in \{1, 2, 4, 8\}$ candidates, since the difference between the criteria is concentrated in the early steps of selection (Figure~\ref{app:fig:feature_selection}). We train a separate classification head on each resulting subset of the frozen backbone representation, using the same architecture and hyperparameters as the mask-conditioned head but without masking. Each subset-specific head is then calibrated and evaluated on the same data as the mask-conditioned model.

Table~\ref{app:tab:subset_specific_ablation} reports the difference in expected set size between Shannon MI and $\Delta_\alpha$ under both models. For simplicity, each $\Delta_\alpha$ criterion is evaluated under its own score. The subset-specific models yield smaller sets than the mask-conditioned model for every criterion, and the advantage of $\Delta_\alpha$ decreases when averaged over $k$ from $4.07$ to $1.89$ under the probability score and from $3.84$ to $2.11$ under the APS score. However, $\Delta_\alpha$ continues to yield smaller sets at every $k$, with the largest differences in the first two steps.

\newpage

\begin{table}[t]
    \centering
    \caption{Difference in expected set size between Shannon MI and $\Delta_\alpha$ on CIFAR-100 when each feature ordering is evaluated with the mask-conditioned model or with subset-specific models. Positive values indicate that $\Delta_\alpha$ yields smaller sets. Values are reported as \mbox{mean $\pm$ standard error} over 10 random splits of the evaluation pool.}
    \label{app:tab:subset_specific_ablation}
    \begin{tabular}{c cc cc}
        \toprule
        & \multicolumn{2}{c}{Prob.} & \multicolumn{2}{c}{APS} \\
        \cmidrule(lr){2-3} \cmidrule(lr){4-5}
        $k$ & Mask-cond. & Subset-specific & Mask-cond. & Subset-specific \\
        \midrule
        1    & $3.75 \pm 0.36$ & $2.49 \pm 0.32$ & $3.59 \pm 0.40$ & $2.33 \pm 0.35$ \\
        2    & $5.81 \pm 0.46$ & $4.06 \pm 0.35$ & $5.24 \pm 0.54$ & $3.87 \pm 0.46$ \\
        4    & $3.71 \pm 0.33$ & $0.42 \pm 0.24$ & $2.98 \pm 0.67$ & $1.32 \pm 0.31$ \\
        8    & $3.00 \pm 0.46$ & $0.58 \pm 0.32$ & $3.56 \pm 0.64$ & $0.91 \pm 0.28$ \\
        \midrule
        Avg. & $4.07 \pm 0.27$ & $1.89 \pm 0.14$ & $3.84 \pm 0.41$ & $2.11 \pm 0.14$ \\
        \bottomrule
    \end{tabular}
\end{table}

\phantom{text}

\clearpage

\section{Compute resources}

All training and analysis were performed on a single machine running Ubuntu 24.04 with 32 GB of RAM and an NVIDIA GeForce RTX 5070 Ti GPU. The hyperparameter grid search required approximately 10 hours and subsequent model training required approximately 12 hours. The validation experiments were run on a single CPU thread without GPU acceleration, whereas the feature selection experiments evaluated the mask-conditioned models on the GPU. All downstream analyses required approximately 4 hours in total.

\section{Experimental assets} \label{app:sec:assets}

We obtain the Human Activity Recognition dataset~\citep{reyesortiz2013human} directly from the UCI Machine Learning Repository, available under the CC BY 4.0 license. For \mbox{ImageNet-1k}, we use the $64 \times 64$ downsampled version of~\citet{chrabaszcz2017downsampled}, which is made available through the \texttt{ChocolateDave/imagenet-64} mirror on Hugging Face. The pretrained \mbox{ResNet-18} weights~\citep{he2016deep} are the \texttt{IMAGENET1K\_V1} weights distributed with torchvision. All remaining datasets are obtained through OpenML or torchvision, with their original sources cited in Section~\ref{sec:experiments} and Appendix~\ref{app:sec:datasets}. Code for reproducing our analyses will be released upon publication.

\end{document}